\documentclass{article}

\usepackage[main, final]{neurips_2026}

\usepackage{cite}
\usepackage[textsize=tiny]{todonotes}

\usepackage{graphicx}

\usepackage{subfigure}
\usepackage{subcaption}

\usepackage{amsmath}
\usepackage{amssymb}
\usepackage{mathtools}
\usepackage{amsthm}

\usepackage{multirow}
\usepackage{import}
\usepackage{pifont}
\usepackage{textcomp}
\usepackage{color, colortbl}
\definecolor{greyC}{RGB}{180,180,180}
\definecolor{greyL}{RGB}{235,235,235}
\usepackage{footmisc}
\usepackage{algorithm}
\usepackage{bm}
\usepackage[flushleft]{threeparttable}
\usepackage[]{hyperref} 
\usepackage{makecell}

\makeatletter
\newcommand*{\addFileDependency}[1]{
  \typeout{(#1)}
  \@addtofilelist{#1}
  \IfFileExists{#1}{}{\typeout{No file #1.}}
}
\makeatother

\usepackage[utf8]{inputenc} 
\usepackage[T1]{fontenc}    
\usepackage{hyperref}       
\usepackage{url}            
\usepackage{booktabs}       
\usepackage{amsfonts}       
\usepackage{nicefrac}       
\usepackage{microtype}      
\usepackage{xcolor}         
\usepackage{adjustbox}
\usepackage{subcaption}
\usepackage{enumitem}

\usepackage{algcompatible}
\usepackage{hyperref}
\usepackage{wrapfig}

\usepackage{etoc}
\etocdepthtag.toc{mtchapter}
\etocsettagdepth{mtchapter}{subsection}
\etocsettagdepth{mtappendix}{none}

\theoremstyle{plain}
\newtheorem{theorem}{Theorem}[section]
\newtheorem{proposition}[theorem]{Proposition}
\newtheorem{lemma}[theorem]{Lemma}
\newtheorem{corollary}[theorem]{Corollary}
\theoremstyle{definition}
\newtheorem{definition}{Definition}[section]
\newtheorem{assumption}{Assumption}[section]
\newtheorem{condition}[theorem]{Condition}
\theoremstyle{remark}
\newtheorem{remark}{Remark}[section]

\newcommand{\cmark}{\ding{51}}%
\newcommand{\xmark}{\ding{55}}%

\newcommand{\lambdamax}{\lambda_{\text{max}}}

\newcommand{\lambdamin}{\lambda_{\text{min}}}
\newcommand{\fisherinfj}{I_{j,\tau^{\perp}}}
\newcommand{\fisherinfk}{I_{k,\tau^{\perp}}}
\newcommand{\conditionnum}{\gamma_{\tau^{\perp}}}

\newcommand{\hiddensegstart}{h^{\text{start}}}

\newcommand{\promptseq}{S_n}

\newcommand{\Lzeroone}{L_{\text{0-1}}}

\newcommand{\indicator}{\mathbf{1}}
\newcommand{\badset}{\sB}

\newcommand{\errstart}{\epsilon^\theta_{\text{start}}}
\newcommand{\errdelim}{\epsilon^\theta_{\text{delim}}}

\newcommand{\obs}{o}
\newcommand{\obsset}{\sO}

\newcommand{\obsseg}{O}
\newcommand{\obsdelim}{o^{\text{delim}}}

\newcommand{\y}{y}
\newcommand{\Xtest}{x_{\text{test}}}
\newcommand{\ytest}{y_{\text{test}}}

\newcommand{\pprompt}{p_{\text{prompt}}}
\newcommand{\ppromptstart}{p_{\text{prompt}}}

\newcommand{\minv}{{-1}}

\DeclareMathOperator*{\argmin}{arg\,min}
\DeclareMathOperator*{\argmax}{arg\,max}

\newlength{\widebarargwidth}
\newlength{\widebarargheight}
\newlength{\widebarargdepth}

\newcommand\sB{\ensuremath{\mathcal{B}}}

\newcommand\sO{\ensuremath{\mathcal{O}}}

\ifthenelse{\isundefined{\definition}}{\newtheorem{definition}{Definition}}{}
\ifthenelse{\isundefined{\assumption}}{\newtheorem{assumption}{Assumption}}{}
\ifthenelse{\isundefined{\hypothesis}}{\newtheorem{hypothesis}{Hypothesis}}{}
\ifthenelse{\isundefined{\proposition}}{\newtheorem{proposition}{Proposition}}{}
\ifthenelse{\isundefined{\theorem}}{\newtheorem{theorem}{Theorem}}{}
\ifthenelse{\isundefined{\lemma}}{\newtheorem{lemma}{Lemma}}{}
\ifthenelse{\isundefined{\corollary}}{\newtheorem{corollary}{Corollary}}{}
\ifthenelse{\isundefined{\alg}}{\newtheorem{alg}{Algorithm}}{}
\ifthenelse{\isundefined{\example}}{\newtheorem{example}{Example}}{}
\newcommand{\E}{\ensuremath{\mathbb{E}}} 

\newcommand{\Lcal}{\mathcal{L}}
\newcommand{\Mcal}{\mathcal{M}}
\newcommand{\Ncal}{\mathcal{N}}

\newcommand{\Tcal}{\mathcal{T}}

\DeclareMathAlphabet{\mathsfit}{\encodingdefault}{\sfdefault}{m}{sl}
\SetMathAlphabet{\mathsfit}{bold}{\encodingdefault}{\sfdefault}{bx}{n}

\def\sB{{\mathbb{B}}}

\def\sO{{\mathbb{O}}}

\ifx\BlackBox\undefined
\newcommand{\BlackBox}{\rule{1.5ex}{1.5ex}}  
\fi

\ifx\QED\undefined
\def\QED{~\rule[-1pt]{5pt}{5pt}\par\medskip}
\fi

\ifx\proof\undefined
\newenvironment{proof}{\par\noindent{\em Proof:\ }}{\hfill\BlackBox\\[.0mm]}
\fi

\ifx\theorem\undefined
\newtheorem{theorem}{Theorem}[section]
\fi

\ifx\example\undefined
\newtheorem{example}{Example}[section]
\fi

\ifx\property\undefined
\newtheorem{property}{Property}
\fi

\ifx\lemma\undefined
\newtheorem{lemma}{Lemma}[section]
\fi

\ifx\proposition\undefined
\newtheorem{proposition}{Proposition}[section]
\fi

\ifx\remark\undefined
\newtheorem{remark}{Remark}
\fi

\ifx\corollary\undefined
\newtheorem{corollary}{Corollary}[section]
\fi

\ifx\definition\undefined
\newtheorem{definition}{Definition}[section]
\fi

\ifx\conjecture\undefined
\newtheorem{conjecture}{Conjecture}[section]
\fi

\ifx\axiom\undefined
\newtheorem{axiom}[theorem]{Axiom}
\fi

\ifx\claim\undefined
\newtheorem{claim}[theorem]{Claim}
\fi

\ifx\assumption\undefined
\newtheorem{assumption}{Assumption}
\fi

\ifx\problem\undefined
\newtheorem{problem}{Problem}[section]
\fi

\ifx\fact\undefined
\newtheorem{fact}{Fact}
\fi

\newcommand{\benr}{\begin{eqnarray}}
\newcommand{\eenr}{\end{eqnarray}}
\newcommand{\benrr}{\begin{eqnarray*}}
\newcommand{\eenrr}{\end{eqnarray*}}
\newcommand{\ben}{\begin{equation}}
\newcommand{\een}{\end{equation}}
\newcommand{\benn}{\begin{equation*}}
\newcommand{\eenn}{\end{equation*}}

\definecolor{mygray}{gray}{.92}

\definecolor{baselinecolor}{rgb}{1, 1, 1}

\definecolor{ourmethodcolor}{rgb}{0.94, 0.97, 1}

\newcommand{\noisy}[1]{\textcolor[rgb]{0.70, 0.29, 0.27}{#1}}

\title{Mediator: Memory-efficient LLM Merging with Less Parameter Conflicts and Uncertainty Based Routing}

\author{\hspace{-1mm}
Kunfeng Lai\textsuperscript{1,*} \quad
Zhenheng Tang\textsuperscript{2,*} \quad
Xinglin Pan\textsuperscript{1} \quad
Peijie Dong\textsuperscript{1} \\
\bf Xiang Liu\textsuperscript{1} \quad
Haolan Chen\textsuperscript{3} \quad
Huacan Wang\textsuperscript{4} \quad
Li Shen\textsuperscript{5} \\
\bf Bo Li\textsuperscript{2} \quad
Xiaowen Chu\textsuperscript{1,2} \\
\textsuperscript{1} The Hong Kong University of Science and Technology (Guangzhou) \\
\textsuperscript{2} The Hong Kong University of Science and Technology \\
\textsuperscript{3} Platform and Content Group, Tencent \\
\textsuperscript{4} UCAS \quad
\textsuperscript{5} Sun Yat-sen University
}

\begin{document}

\maketitle

\begin{abstract}
Model merging aggregates Large Language Models (LLMs) finetuned on different tasks into a stronger one. However, parameter conflicts between models leads to performance degradation in averaging. While model routing addresses this issue by selecting individual models during inference, it imposes excessive storage and compute costs, and fails to leverage the common knowledge from different models. 
In this work, we observe that different layers exhibit varying levels of parameter conflicts. Building on this insight, we  average layers with minimal parameter conflicts and use a novel task-level expert routing for layers with significant conflicts.
To further reduce storage costs, inspired by task arithmetic sparsity, we decouple multiple fine-tuned experts into a dense expert and several sparse experts. Considering the out-of-distribution samples, we select and merge appropriate experts based on the task uncertainty of the input data. 
We conduct extensive experiments on both LLaMA and Qwen with varying parameter scales, and evaluate on real-world reasoning tasks. Results demonstrate that our method consistently achieves significant performance improvements while requiring less system cost.\footnote{$^{*}$Equal contribution. Correspondence to: xwchu@hkust-gz.edu.cn}
\end{abstract}

\section{Introduction}
\label{sec:Introduction}

\begin{wrapfigure}{r}{0.5\textwidth}
\setlength{\abovedisplayskip}{-1pt}
    \subfigbottomskip=0pt
    \subfigcapskip=1pt
    \setlength{\abovecaptionskip}{-1pt}
    \centering
    \includegraphics[width=1.0\linewidth]{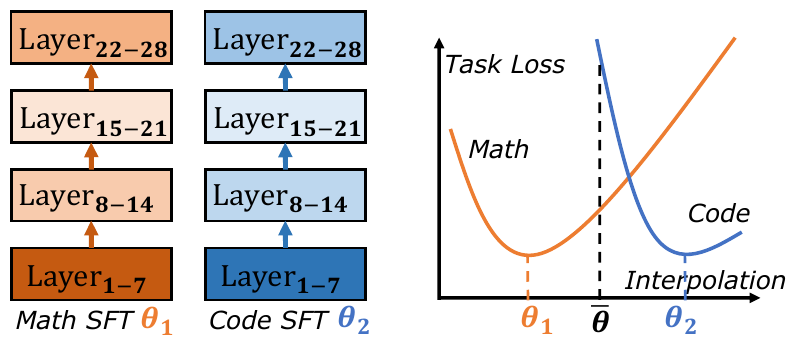}
    \caption{\small Knowledge conflict across finetuned LLMs and math and code dataset. Deeper color means larger parameter conflicts. And it is difficult for the linear averaged model to achieve low loss of both tasks.}
    \label{fig:knowledge-conflict}
\end{wrapfigure}

Finetuning Large Language Models (LLMs) enables them to adapt to downstream applications including sentiment analysis~\citep{sun2023sentiment}, text summarization~\citep{fang2024multillmtextsummarization}, mathematical reasoning~\citep{ruis2024proceduralknowledgepretrainingdrives}, code writing~\citep{jiang2024surveylargelanguagemodels},  roleplay chatting \citep{chen2025rolplay} so on.  Open-source platforms such as Huggingface~\citep{wolf2019huggingface} facilitate access to a diverse array of highly trained expert models with varying capabilities. Considering the computational resources are scarce and implementing green computing~\citep{samsi2023wordswattsbenchmarkingenergy,stojkovic2024greenerllmsbringingenergyefficiency,bai2024beyond}, the community is increasingly interested in how to merge these models to create a superior LLM that retains the strengths of finetuned ones without retraining~\citep{yang2024model,lu2024twin,Du:2024aa,Yadav:2023aa}.







One predominant merging strategy is model averaging~\citep{yang2024model,FisherMerging_NeurIPS2022,thennal2024fisher,DARE_Arxiv2023}, which computes weighted averages of parameters to synthesize collective knowledge~\citep{FisherMerging_NeurIPS2022,Yadav:2023aa}. However, model averaging faces challenges from parameter conflicts arising from diverse finetuning tasks, leading to performance degradation as shown in Figure~\ref{fig:knowledge-conflict}. Another direction is model routing~\citep{lu2024twin,muqeeth2024soft,Yang:2023aa,Du:2024aa,Lu:2024aa,He:2024aa,Wei:2024aa,Chen:2024aa}, which aggregates models and performs model selection during inference. This method avoids parameter conflicts but incurs significant computing and storage (system) costs due to maintaining all finetuned models. This motivates us to rethink the following questions:
\vspace{-0.0cm}
\begin{center}
\vspace{-0.0cm}
\emph{How to merge common and unique knowledge from various finetuned models while simultaneously avoiding parameter conflicts and minimizing system costs?}
\vspace{-0.0cm}
\end{center}
\vspace{-0.0cm}


To study the noise influences, we first quantify the parameter conflicts between fine-tuned LLMs using the sign consistency of their task arithmetics. Our analysis reveals a distinct layer-wise distribution of these conflicts: the front and last layers exhibit the highest levels of conflict, suggesting that they are particularly sensitive to simple averaging. In contrast, the central layers demonstrate comparatively lower conflict levels, indicating that they successfully retain shared common knowledge across different tasks. 

Motivated by these observations, we introduce Mediator, an adaptive model merging framework designed to enhance LLM merging with minimal storage and computational overhead. First, Mediator employs adaptive merging by averaging the low-conflict layers to capture common knowledge~\citep{Yadav:2023aa,he2024localize}, while treating high-conflict layers as dynamically routed experts to preserve task-specific knowledge without dilution. Second, we leverage the high sparsity of task arithmetics~\citep{Yadav:2023aa} to decompose models into base and task-specific components~\citep{TaskArithmetic_ICLR2023,he2024localize,Yang:2023aa,Tang:2024aa}, radically reducing storage requirements from 50\% to 7\%. Third, to align with how LLMs are fine-tuned on complete downstream sentences, we propose task-level expert routing rather than token-level routing~\citep{Lepikhin:2020aa,Sukhbaatar:2024aa,Zhou:2022aa,Jiang:2024aa}, and incorporate an uncertainty-based expert selection mechanism to handle out-of-distribution (OOD) samples. Ultimately, Mediator successfully preserves complex reasoning abilities~\citep{zelikman2022star,kojima2022large,guo2025deepseek} and enables running a model comparable to a 7B $\times$ 4 LLM ensemble on a single RTX 4090 GPU with negligible performance degradation (0.06\% $\sim$ 0.3\%) (Appendix~\ref{appx:system-performance-consumer}).

\begin{itemize}[leftmargin=*]
\item \noindent We investigate and demonstrate that different layers of fine-tuned models exhibit varying levels of parameter conflicts (Section~\ref{sec:Understanding-Conflict}). Then we propose Mediator, an adaptive layer-wise model merging approach to average layers with minimal conflicts and use task-level expert routing for layers with significant conflicts (Section~\ref{sec:AdaptiveMerging}).

\item \noindent We propose a method in Mediator to decouple fine-tuned experts into one dense expert and several sparse experts (Section~\ref{sec:expert-decomposition}), achieving high compression ratio while maintaining accuracy. Our approach enables dynamic expert selection based on task uncertainty (Section~\ref{sec:expert-routing}), effectively handling OOD data.

\item \noindent We conduct experiments based on modern LLMs including LLaMA and Qwen with CoT enhanced finetuning the real-world cutting-edge LLM evaluation tasks. Results show that Mediator achieves performance improvements and less system cost compared to existing methods. (Section~\ref{sec:Experiments}).
\end{itemize}

\vspace{-0.0cm}
\section{Preliminary and Related Works}
\vspace{-0.0cm}
\label{sec:Preliminary}

\subsection{Language Modeling and LLM Finetuning}\label{sec:LanguageModel}

\textbf{Task Data Distribution.}
Given a set of different downstream tasks $\Tcal $, based on the sampling task $\tau \in \Tcal$, the pretraining data sample is a sequence $x_{1:T}$ of tokens with the maximum length $T$ generated from a distribution $p_{\tau} = p(x_{1:T} \vert \tau)  = p(\obs_1, \dots, \obs_T \vert \tau)$~\citep{xie2022an,wies2023learnability,hahn2023theory,li2024language}. We define the pretraining data is sampled from $p(x|\Tcal^\star) = \int_{\tau^\star \in \Tcal^\star} p(\obs_1, \dots, \obs_T \vert \tau)p(\tau^\star)d \tau^\star$. Each token $\obs$ is sampled from a vocabulary $\obsset$. And both ($\Tcal$ and $\Tcal^\star$ belong to a large task family $\Omega$, i.e. $\Tcal,\Tcal^\star \subset \Omega$.




\textbf{Language Modeling.}
Current LLMs~\citep{gpt3_2020,touvron2023llama2,xie2022an} usually utilize the next word prediction as the language modelling, which predicts the next token $x_t$ given the previous tokens $x_{1:t-1}$ for all $t=1,\dots, T$. Formally, a LLM parameterized by $\theta$ is a distribution $f_{\theta}(x_t \vert x_{1:t-1})$. And it is pretrained on a huge corpus sampled from the pretraining distribution $p(x|\Tcal^\star)$~\citep{xie2022an}.



\textbf{Finetuning LLM.} Normally, for each downstream task $\tau \in \Tcal$, finetuning LLM is to minimize the cross-entropy loss function as $L_\text{CE}(\theta, \tau)= -\sum_{t=1}^T \mathbb{E} [p_{\tau}(x_t|x_{1:t-1})\cdot\log f_{\theta}(x_t|x_{1:t-1})]$. After finetuning, the model parameters $\theta$ are updated to $\theta_{\tau}$. 




\vspace{-0.0cm}
\subsection{Model Merging}\label{sec:Model Merging}


Given finetuned task-specific LLMs $\{\theta_1, \theta_2, \dots, \theta_{n_{\tau}}\}$ finetuned on task set $\Tcal$, where $n_{\tau}=|\Tcal|$, model merge aims to find a unified model parameterized by $\phi$ that can achieve the low loss on all tasks $\Tcal$ as following

\begin{align}\label{eq:model-merging-goal}
    \min_{\phi}  L_\text{CE}(\phi, \Tcal) = \frac{1}{n_{\tau}} \sum_{\tau \in \Tcal} L_\text{CE}(\phi, \tau).
\end{align}
Different from training $\phi$ that can be optimized towards any direction, model merging aims to exploit combining $\{\theta, \theta_1, \theta_2, \dots, \theta_{n_{\tau}}\}$ to obtain the $\phi$. The current model merging methods include following two categories.

\textbf{Model Averaging.} Averaging parameters to fuse the knowledge from different finetuned models is straightforward. Mathematically, averaged model is $\phi = \sum_{\tau \in \Tcal} w_\tau \theta_{\tau}$, in which $w_\tau$ is the averaging weight and $\sum_{\tau \in \Tcal} w_\tau = 1$. Considering different model parameters have different importance on downstream tasks~\citep{kirkpatrick2017overcoming,sun2024a,dong2024pruner}, assigning larger weights to more important parameters. Current methods usually utilize Taylor expansion~\citep{lee2018snip,FisherMerging_NeurIPS2022} to measure the importance of the parameters. However, the knowledge conflicts still exist for parameters that have high importance simultaneously.


\textbf{Model Routing.} To completely avoid conflicts, another way is to select the most relevant model for each task. This approach typically employs a selection mechanism to activate the most relevant model based on the input task $\tau$~\citep{Yang:2023aa} or sequence $x_{1:t}$ at $t$-th token like the Mixture of Experts (MoE)~\citep{tang2024moemerging}. Current methods propose different routing and re-training mechanisms to improve the performance~\citep{He:2024aa, Wei:2024aa,sukhbaatar2024branchtrainmix}. However, these methods fail to consider merging parameters to find the common knowledge that can be shared across different tasks, and cause large memory and computational costs. We leave detailed discussions about related works in Appendix~\ref{appx:more-related-works}).

\textbf{Out-of-distribution Data.} In real-world deployment, the test data $x$ may come from other distributions instead of the $p(x|\tau)_{\tau \in \Tcal}$. To this end, we need to consider how to handle OOD data $x$ within merging LLMs. In this work, we mainly consider two OOD cases and tackle them in Section~\ref{sec:Method}.



\vspace{-0.0cm}
\section{Understanding Conflict between LLMs}\label{sec:Understanding-Conflict}
\vspace{-0.0cm}

\begin{wraptable}{r}{0.5\textwidth}
    \centering
    \setlength{\abovedisplayskip}{-2pt}
    \subfigbottomskip=-1pt
    \subfigcapskip=1pt
    \setlength{\abovecaptionskip}{-2pt}
    \caption{Accuracy of finetuning Llama 3.2 3B.}
    \resizebox{0.9\linewidth}{!}{
    \begin{tabular}{c|cccc}
    \toprule
    Model & GSM8K Math  & TriviaQA & H.Eval Code & All tasks \\
    \midrule
    $\theta$ (Pretrained) & 27.52 & 57.71 & 22.56 & 35.93 \\
    $\theta_1$ (Math SFT) & \textbf{46.47} & 54.59 & 25.00 & 42.02 \\
    $\theta_2$ (QA SFT) & 32.75 & \textbf{61.45} & 28.05 & 40.75 \\
    $\theta_3$  (Coding SFT) & 33.13 & 57.71 & \textbf{40.85} & 43.90 \\
    $\phi_{\text{AVG}}$ & 42.61 & 60.99 & 31.30 & 44.97 \\
    $\phi_{\text{SEL}}$ & \textbf{46.47} & \textbf{61.45} & \textbf{40.85} & \textbf{49.59} \\
    \bottomrule
\end{tabular}}
\label{tab:preliminary-experiments}
\end{wraptable}
\vspace{-0.0cm}

\textbf{Preliminary Experiments.} We finetune Llama-3.2-3B on three datasets and evaluation with according tasks (details in Section~\ref{sec:Experiments}). Table~\ref{tab:preliminary-experiments} shows the performance of the merged model and  individual finetuned models. We use $P(\theta, \tau)$ to represent the performance of the $\theta$ on $\tau$.


\textbf{Comparing Performance on All Tasks.} We write $P_{\text{ORI}} = P(\theta, \Tcal)$ as the performance of the original model $\theta$ on all tasks $\Tcal$, the $P_{\text{AVG}} = P(\phi_{\text{AVG}}, \Tcal)$ as the performance of the averaged model $\phi_{\text{AVG}}$ on all tasks $\Tcal$. We write the $P_{\text{SEL}} = \frac{1}{n_{\tau}}  \sum_{\tau \in \Tcal} P(\phi_{\text{SEL}}, \tau)$ as the performance of the selection based model $\phi_{\text{SEL}}$ on all tasks $\Tcal$. The results show that the
\begin{align}\label{eq:SEL-AVG-ORI}
    P_{\text{ORI}} < P_{\text{AVG}} < P_{\text{SEL}},
\end{align}
which means that finetuning and averaging can successfully increase the model performance. However, simply averaged model cannot recover the task performance of the corresponding finetuned models.


\textbf{Comparing Performance on Individual Models.} Table~\ref{tab:preliminary-experiments} shows that while the model $\theta_{\tau}$ has the best performance on its according task $\tau$, its performance on other tasks $\Tcal \setminus \tau$ is lower than $\phi_{\text{AVG}}$.

\begin{figure*}[!th]
    \centering
    \includegraphics[width=0.8\linewidth]{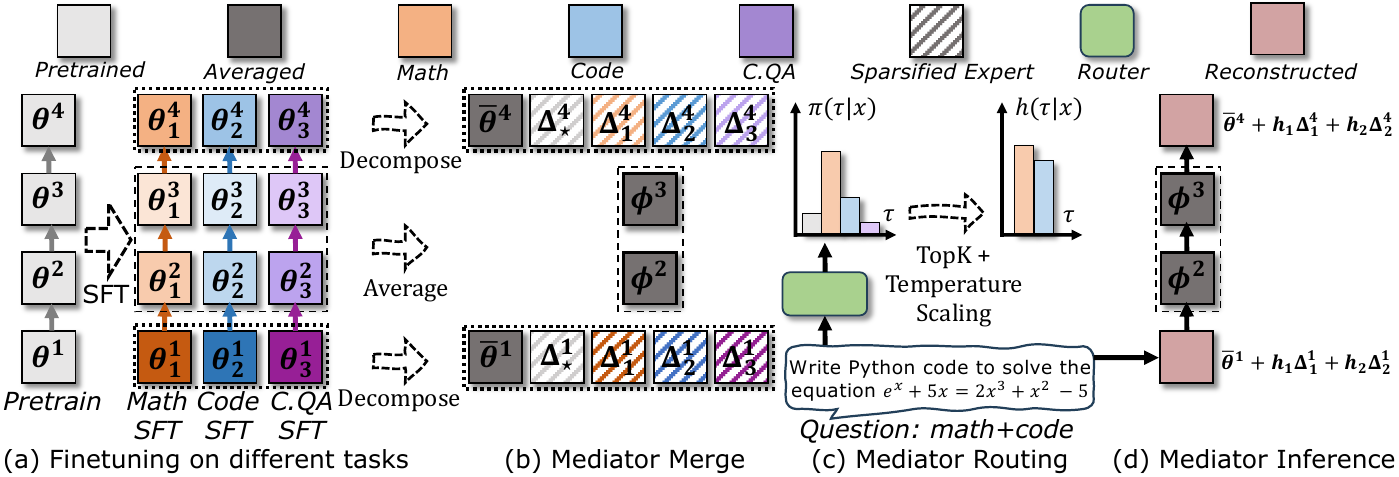}
    \vspace{-0.0cm}
    \caption{The framework of Mediator.}
\vspace{-0.0cm}
\label{fig:Mediator}
\end{figure*}

Table~\ref{tab:preliminary-experiments} indicates that the averaged model can still benefit from merging knowledge from different finetuned models. Thus, a better merging strategy should be able to average the parameters that have less parameter conflict to find the common knowledge that can be shared across different tasks and avoid the parameter conflict that degrades the performance of the finetuned model.

\begin{wrapfigure}{r}{0.5\textwidth}
    \centering
    \includegraphics[width=1.0\linewidth]{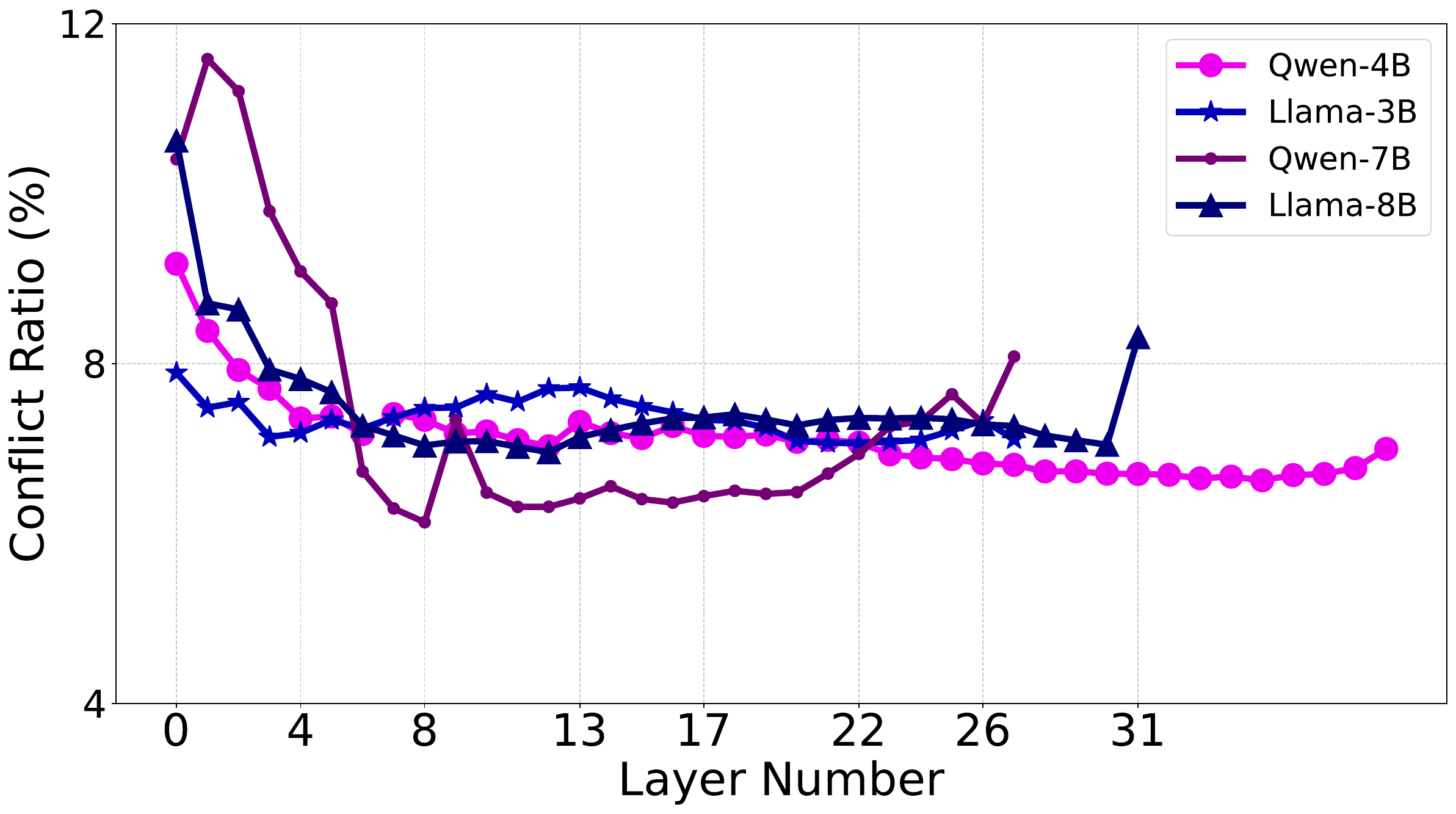}
    \vspace{-0.0cm}
    \caption{Parameter conflict distribution across different layers of finetuned models (Qwen 2.5 7B).}
\label{fig:parameter-conflict}
\vspace{-0.0cm}
\end{wrapfigure}
\vspace{-0.0cm}

\begin{definition}[Task Arithmetic]\label{def:task-arithmetic}
    A task arithmetic on task $\tau$ is the parameter difference between the finetuned LLM $\theta_{\tau}$ and the pre-trained LLM $\theta$, i.e., $\Delta_{\tau} = \theta_{\tau} - \theta$. 
\end{definition}


\textbf{Denoising Parameters.}
Due to the stochastic optimization process in finetuning, some elements in $\Delta_\tau$ are noisy and do not influence the performance~\citep{Yadav:2023aa,he2024localize}. Thus, before measuring the parameter conflicts~\citep{Yadav:2023aa,he2024localize}, we firstly denoise the parameters by removing the elements in $\Delta_\tau$ that have have small magnitude. To this end, we model the update directions of different elements as the Gaussian distribution $\Ncal_{\text{UPD}}(\mu_{\text{UPD}}, \sigma_{\text{UPD}}^2)$, where $\mu_{\text{UPD}}$ is the mean of the update direction and $\sigma_{\text{UPD}}^2$ is the variance. Then, we denoise the parameters by removing the elements within range ($\mu_{\text{UPD}} - \sigma_{\text{UPD}}$, $\mu_{\text{UPD}} + \sigma_{\text{UPD}}$) and obtain the new parameter arithmetic $ \hat{\theta}_{\tau} = \theta + \hat{\Delta}_{\tau}$ (details in Appendix~\ref{appx:measuring-parameter-conflicts}).



\textbf{Measuring Parameter Conflict.} Following~\citep{Yadav:2023aa}, we measure the parameter conflict between different finetuned models using a sign-based approach. For layer $l \in \Lcal$, given parameters $w_i^l \in \theta_{\tau_i}^l$ and $w_j^l \in \theta_{\tau_j}^l$ from two different models, we consider them conflicting if they have opposite signs, i.e., $sgn(w_i^l w_j^l) = -1$. 
We define the conflict ratio of layer $l$ as the proportion of conflicting parameters in that layer:
\begin{align}
    d_l = (\sum_{i,j} \mathbb{I}(sgn(w_i^l w_j^l) = -1))/|\theta^l|,
\end{align}
where $|\theta^l|$ is the total number of parameters in layer $l$, and $\mathbb{I}(\cdot)$ is the indicator function.

Figure~\ref{fig:parameter-conflict} shows that the parameter conflict is higher in the front and last layers, and lower in the central layers. This phenomenon indicates that central layers share more common knowledge, while the front and last layers capture more task-specific knowledge.



\vspace{-0.0cm}
\section{The Design of Mediator}\label{sec:Method}
The high-level idea of Mediator is to hybridly combine averaging and routing to preserve the downstream knowledge and avoid the parameter conflict (Sectino~\ref{sec:AdaptiveMerging}) according to parameter conflict distribution across different layers. We decompose the downstream models as the base model and experts of task arithmetics, which brings into the opportunity to sparsify the task arithmetics to prune out the noisy parameters to reduce the memory costs (Section~\ref{sec:expert-decomposition}). Lastly, considering the OOD inputs, we propose a Bayesian expert routing (Section~\ref{sec:expert-routing}) to better combine knowledge from different experts. 


\subsection{Adaptive Layer-wise Model Averaging and Routing}\label{sec:AdaptiveMerging}

\begin{wrapfigure}{R}{0.5\textwidth}
\begin{minipage}{0.5\textwidth}
\begin{algorithm}[H]
	\caption{Adaptive Merging and Routing with Sparsified Expert Decomposition in Mediator}
	\label{algo:adaptive-merging}
	\textbf{Input: } Different finetuned models $\theta_1, \theta_2, \dots, \theta_{n_{\tau}}$.\\
    \textbf{Output:} The merged layers $\Phi$.\\
    \begin{algorithmic}[1]
        \STATE Calculate the conflict distribution $\left\{d_l\right\}_{l=1, \cdots, |\Lcal|}$;
        \STATE Estimate $\mu, \sigma$ based on $\left\{d_l\right\}_{l=1, \cdots, |\Lcal|}$;
        \FOR{layer $ l=1, \cdots, |\Lcal|$}
            \IF{$d_l < \mu + \sigma$}
                \STATE $\phi_{\text{AVG}}^l = \Mcal(\theta_1^l, \theta_2^l, \dots, \theta_{n_{\tau}}^l)$; 
                \ELSE
                \STATE $\bar{\theta}^{l} = 1/n_{\tau} \sum_{\tau\in\Tcal} \theta_{\tau}^l $;
                \STATE $\Delta_{\tau}^l = \theta_{\tau}^l - \bar{\theta}^{l}$, $\Delta_{\star}^l = \theta^l - \bar{\theta}^{l}$;
                \STATE $\hat{\Delta}_{\tau}^l = \text{Denoise}(\Delta_{\tau}^l)$;
                \STATE $\phi_{\text{UP}}^l = \left\{\bar{\theta}^{l}, \hat{\Delta}_{1}^l,\hat{\Delta}_{2}^l,\dots,\hat{\Delta}_{n_{\tau}}^l,\hat{\Delta}_{\star}^l  \right\}$;
                \ENDIF
            \STATE Insert $\phi_{\text{AVG}}^l$ or $\phi_{\text{UP}}^l$ into $\Phi_{\text{AVG}}$ and $\Phi_{\text{UP}}$;
        \ENDFOR
        \STATE \textbf{Return} $\Phi = \{\Phi_{\text{AVG}}, \Phi_{\text{UP}}\}$.
\end{algorithmic}
\vspace{0.1cm}
\end{algorithm}
\end{minipage}
\end{wrapfigure}
Inspired by the empirical observation in Figure~\ref{fig:parameter-conflict}, Mediator averages layers with less parameter conflicts, and route layers with more parameter conflicts. As shown in Figure~\ref{fig:Mediator}, Mediator calculates the conflicts $d_l$ across different layers. Then, Mediator models the layer-wise conflicts as a Gaussian distribution $d_l \sim \Ncal(\mu, \sigma)$.

Then, for each layer index $l$, Mediator average layer parameters if the conflict $d_l$ is less than the $\mu + \sigma$, otherwise routing this layer. We denote the averaged layer parameters as $\phi_{\text{AVG}}^l$ and the routing layer parameters as $\phi_{\text{UP}}^l$. Algorithm~\ref{algo:adaptive-merging} shows this detailed process. The average operation $\Mcal$ can be any averaging operation, such as the unified average, importance based~\citep{FisherMerging_NeurIPS2022}, or subspace based~\citep{Yadav:2023aa}. In our experiments, we mainly use the denoised parameters $\hat{\theta}_{\tau}$ defined in Section~\ref{sec:Understanding-Conflict} to conduct averaging like TIES~\citep{Yadav:2023aa}  (details of averaging operations in Appendix~\ref{appx:averaging-operations}). Note that all attention layers are averaged, because they are found to save non-specific domain knowledge~\citep{Sukhbaatar:2024aa}.

\subsection{Expert Decomposition}\label{sec:expert-decomposition}

The routing layer occupies $n_{\tau} \times M_l$ memory, where $n_{\tau}=|\Tcal|$ and $M_l$ is the memory of each layer in original model. Large $n_{\tau}$ significanly increases the memory cost of routing layers, thus leading to weak scalability. Thus, we consider compressing the routing layers to reduce the memory cost.


\begin{wrapfigure}{r}{0.5\textwidth}
    \subfigbottomskip=-1pt
    \subfigcapskip=1pt
  \centering
     \subfigure[Qwen 1.5 4B]{\includegraphics[width=0.45\linewidth]{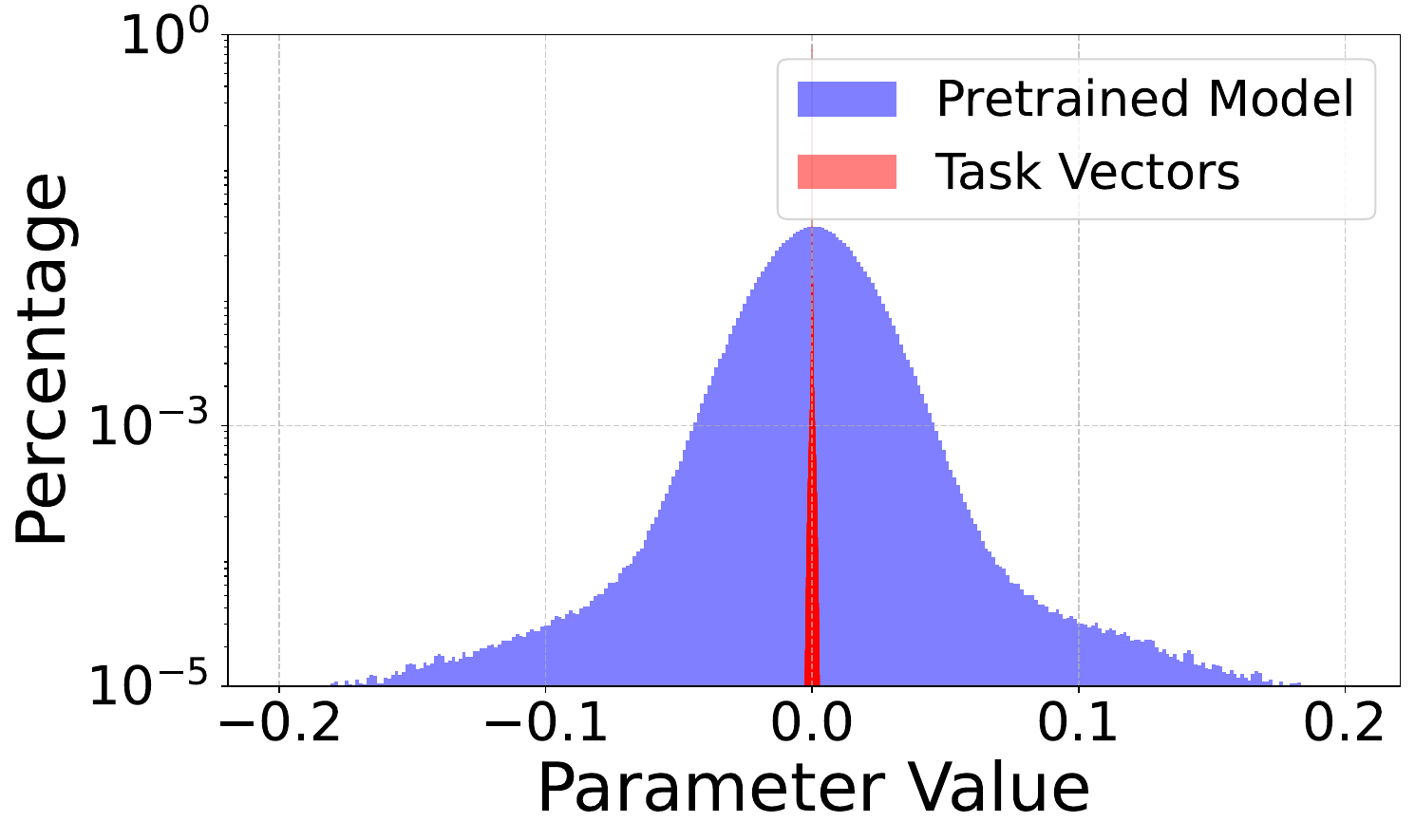}}
     \subfigure[Qwen 2.5 7B]{\includegraphics[width=0.45\linewidth]{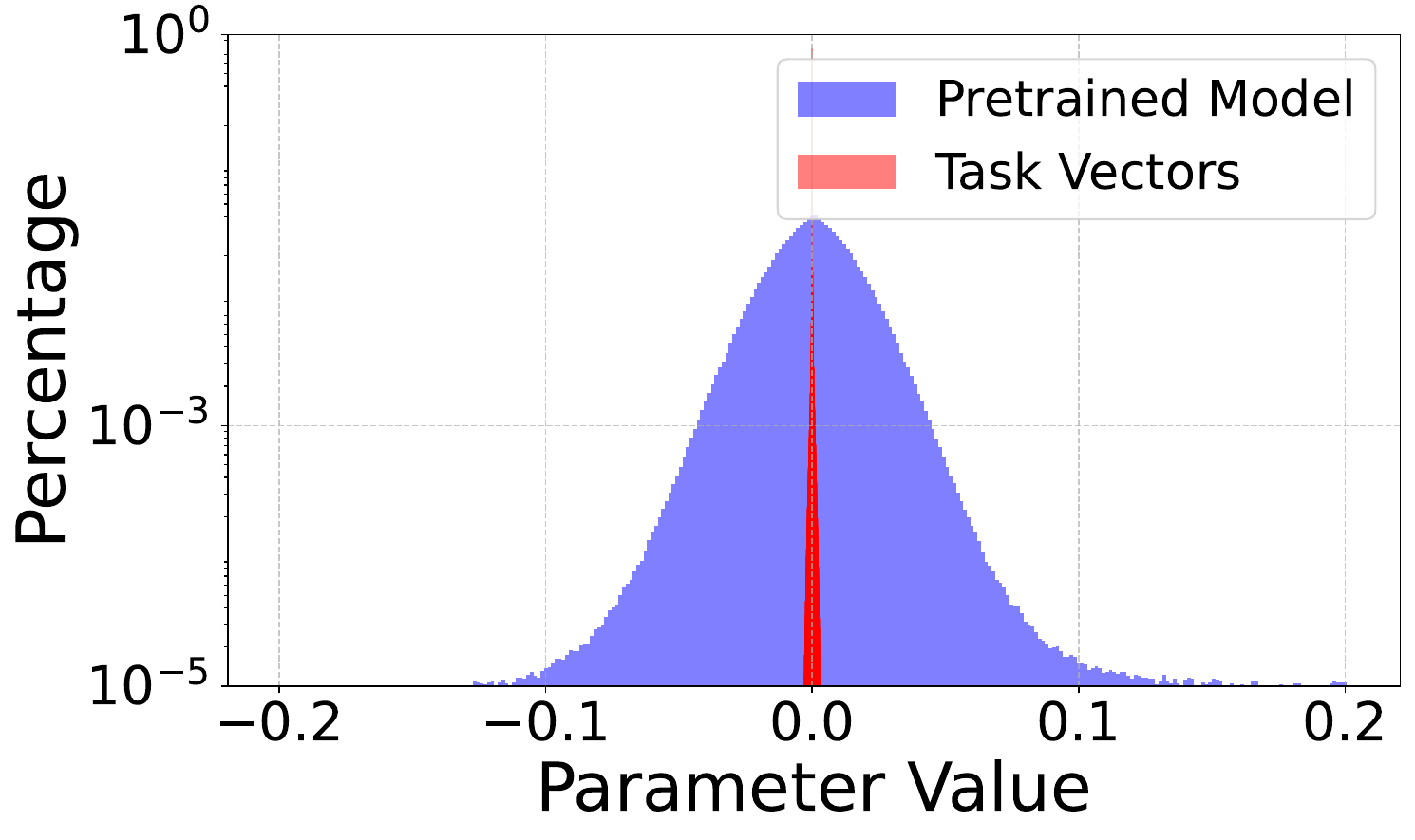}}
     \caption{Comparing magnitudes of task arithmetic and pretrained model parameters.}
    \label{fig:ta_distribution}
    \vspace{-0.3cm}
\end{wrapfigure}


However, previous LLM pruning or quantization~\citep{dong2024pruner,sun2024a} on $\left\{\theta^{l}\right\}$ cannot achieve high compression ratio on $\theta_\tau$. Different from directly compressing the finetuned model, we find that the task arithmetic $\Delta_{\tau}$ shows significantly higher sparsity due to its noisy updates~\citep{Yadav:2023aa}. As shown in Figure~\ref{fig:ta_distribution}, the empirical magnitudes of the task arithmetic are significantly smaller than the model parameters. This indicates that the $\Delta_{\tau}$ can be sparsified with a higher degree to reduce the memory cost.



To this end, we can decompose a finetuned LLM $\theta_{\tau}$ into its base model $\theta$ and the task-specific expert $\Delta_{\tau}$ (task arithmetic in Definition~\ref{def:task-arithmetic}). Each $\Delta_{\tau}$ will be sparsified as $\hat{\Delta}_{\tau}$ by removing the elements that can be seen as noise (Section~\ref{sec:Understanding-Conflict}). Thus, the memory cost is reduced from $n_{\tau} \times M_l$ to $n_{\tau} \times M_l \times c$, where $c$ is the compression ratio. When inferencing, we can use the $\theta + \hat{\Delta}_{\tau}$ to approximately recover the finetuned model $\theta_{\tau}$.

\textbf{Out-of-distribution Data \uppercase\expandafter{\romannumeral1}} (OOD to $\Tcal$ but in-distribution to pretraining data $p(x|\Tcal^\star)$): The finetuned model $\theta_\tau$ may not be able for processing some test-time questions $x^{\tau^{\star}} \sim p(x|\Tcal^\star)$ that sampled from other distributions instead of $p(x|\tau) $ for any $\tau \in \Tcal$. Considering the the modern LLM $\theta$ is pretrained on the huge corpus~\citep{gpt3_2020}, the $\theta_\tau$ might keep the original knowledge in  $x^{\tau^{\star}}$. However, the finetuning may cause forgetting problem~\citep{zhu2024model}. Thus, to preserve pretraining knowledge, we make a calibration on task arithmetic as $\Delta_{\tau} = \theta_{\tau} - \bar{\theta}$ and saves the pretraining arithmetic as $\Delta_{\star} = \theta - \bar{\theta}$, where $\bar{\theta} = 1/n_{\tau} \sum_{\tau \in \Tcal} \theta_{\tau}$.




\subsection{Expert Routing}\label{sec:expert-routing}

In the token-level expert routing like pretrained MoE~\citep{Jiang:2024aa,Sukhbaatar:2024aa}, there are routers for all layers trained to select the task-specific experts for each token. The token-level routing implies that the router is designed for selecting experts at different token index. In other words, for a same sequence $x_{1:T}$ sampled from $p(x_{1:T} | \tau)$, the token-level router may select different experts for different tokens $t \in \left\{1, \cdots, T\right\}$.


\begin{wrapfigure}{R}{0.5\textwidth}
\begin{minipage}{0.5\textwidth}
\begin{algorithm}[H]
\caption{Uncertainty based Expert Routing in Mediator}
\label{algo:routing}
\textbf{Input: } The question prompt $x_{1:t}$, router $\pi_{\kappa}$, $\Phi_{\text{UP}}$, $k$, $\beta$.\\
\textbf{Output:} Reconstructed routing layers $\Phi_{\text{UP},x}$. \\
\begin{algorithmic}[1]
\STATE $ \Tcal_k^x = \text{Indexes of TopK}(\left\{\pi_{\kappa}(\tau|x) \right\},k)$;
\STATE Calculate and cache $\left\{ h(\tau|x)  \right \} $ (Equation~\ref{eq:temperature-scale});
\FOR{$\phi_{\text{UP}}^l \in \Phi_{\text{UP}}$}
    \STATE $ \phi_{\text{UP},x}^l = \bar{\theta}^{l} + \sum_{\tau\in\Tcal_k^x} h(\tau|x) \hat{\Delta}_{n_{\tau}}^l$;
    \STATE Insert $\phi_{\text{UP},x}^l$ into $\Phi_{\text{UP},x}$;
\ENDFOR
\STATE \textbf{Return} $\Phi_{\text{UP},x}$.
\end{algorithmic}
\vspace{0.1cm}
\end{algorithm}
\end{minipage}
\end{wrapfigure}

\textbf{Task-level Routing.} However, different from pretrained MoE where one sequence $x_{1:T}$ might be splited into different subsequences and feed-forwarded to different experts, the finetuned LLM $\theta_{\tau}$ is trained on the complete sequence $x_{1:T}\sim p_{\tau}$. Therefore, a subsequence $x_{1:t} \subset x_{1:T}$ might be OOD to the another model $\theta_{\hat{\tau}}$. Therefore, the most suitable experts for each data sample might be the experts that are trained on the task $\tau$, i.e. $\left\{\theta_{\tau}^l\right\}_{l \in \Lcal}$ for each subsequence $\left\{ x_{1:t} \right\}_{t=1, \cdots, T}$ in $x_{1:T}\sim p_{\tau}$. 


We also provide a theoretical understanding to understand the differences between task-level and token-level routing following the Bayesian inference interpretation for the in-context learning~\citep{xie2022an} in Appendix~\ref{appx:theoretical-analysis}).
Moreover, we also empirically show in Section~\ref{sec:Experiments} that the trained token-level routers~\citep{Sukhbaatar:2024aa} also frequently select the task-specific experts for each data sample. 


\textbf{Out-of-distribution Data \uppercase\expandafter{\romannumeral2}} (OOD but close to $\Tcal$ and $\Omega$): Considering that $\tau$ is continuous~\citep{xie2022an}, while the test data $x^{\text{OOD}} \sim p_{\hat{\tau}}(x|\hat{\tau})$ ($\hat{\tau} \notin \Tcal, \Tcal^\star$) has different distribution from $p_{\tau}$ for any $\tau \in \Tcal, \Tcal^\star$, the $\hat{\tau}$ might be close to or a linear combination of multiple $\tau_1, \tau_2, \dots, \tau_{n_\tau} \in \Tcal$ and $\tau \in \Tcal^\star$. Any finetuned model $\theta_\tau$ may not be able to process some test-time questions $x^{\text{OOD}}$. A better way to conduct inference on $x^{\text{OOD}}$ is to combine multiple experts $\left\{\Delta_{\tau} \right\}_{\tau \in \Tcal}$ and $\Delta_\star$ together. Thus, we propose the Algorithm~\ref{algo:routing} and the following design.


\textbf{Uncertainty-based Expert Selection.} We propose the uncertainty-based expert selection to combine the experts from different tasks into a new expert $\bar{\Delta}$ for the test data $x$. Thus, we need to estimate the posterior distribution $p(\bar{\Delta} | x)$ for the test data $x$. We introduce a small classifier $\kappa$ use cross-entropy loss to learn the likelihood $\pi_{\kappa} (\tau| x )$ for the training dataset including all data pairs $\left\{(x, \tau)| x \sim p_{\tau}\right\}_{\tau \in \Tcal\cup \left\{\tau^\star\right\} }$ (details of constructing this dataset and learning $\pi_\kappa (\tau|x)$ is in Appendix~\ref{appx:expert-routing}). For the training data $x\sim p_{\tau}$, we have known that its best expert is $\Delta_{\tau}$, thus having $\bar{\Delta} = \Delta_{\tau}$. For a test data $x$, we introduce the temperature scaling to soft the logits $h_\kappa$ to obtain the prediction uncertainty $\pi_{\kappa} (\tau| x )$, like OOD detection~\citep{Guo2017OnOCO} and model distillation~\citep{Hinton2015DistillingDTK} as follows: 
\begin{align}
h(\tau|x) = \frac{e^{\pi_\kappa(\tau| x)/\beta}}{\sum_{\tau \in \Tcal} e^{\pi_\kappa(\tau| x)/\beta}}.
\label{eq:temperature-scale}
\end{align}

Then, the final obtained expert is $\bar{\theta}^{l} + \sum_{\tau\in\Tcal_k^x} h(\tau|x) \hat{\Delta}_{n_{\tau}}^l$ as shown in Algorithm~\ref{algo:routing}. Then the whole feed-forward process is as same as the original model. We further provide system optimization in the real-world serving as follows.






\subsection{System-level Optimization}\label{sec:SystemOptimization}

We compare the inference times of Mediator with those of existing LLMs. The incremental time overhead primarily arises from expert routing and the loading of expert parameters. Fortunately, by using task-level routing, we only run a small classifier \( \kappa \) and load the expert parameters once per sample. To accelerate inference, we explore two methods for optimizing the loading of expert parameters. For sparse expert parameters, we store them in the CPU and prefetch the parameters of the upcoming layer while simultaneously performing computations in the current layer. This approach masks the long loading times associated with expert uploads. Additionally, we accelerate the integration of sparse experts (i.e., Algorithm~\ref{algo:routing}) into the dense backbone by using CUDA's atomicAdd for parallel processing. For non-sparse expert parameters, we store them on disk and utilize ZipNN~\citep{hershcovitch2024zipnn} to accelerate the loading of parameters from disk to CPU. Our results show that the optimized inference time of Mediator remains between 0.2s and 0.4s longer than that of existing LLMs per sample (details in Appendix~\ref{appx:system-optimization}).

\vspace{-0.0cm}
\section{Experiments}\label{sec:Experiments}

\textbf{Models and Evaluation Tasks.} We conduct comprehensive experiments on cutting-edge LLMs including Qwen-1.5-4B, Qwen-2.5-7B~\citep{yang2024qwen2}, LLaMA-3.1-3B, and LLaMA-3.2-8B~\citep{dubey2024llama}. We select different evaluation tasks to effectively demonstrate model capability in resolving parameter conflicts during model merging, including GSM8K of mathematical question-answering~\citep{cobbe2021gsm8k}, TriviaQA~\citep{joshi2017triviaqa} of a large-scale Wikipedia-based question answering dataset, HumanEval~\citep{chen2021humaneval} of Python programming tasks, WinoGrande~\citep{sakaguchi2019winogrande} of logical reasoning, MMLU~\citep{hendrycks2021mmlu} of vertical domain knowledge (as OOD to the finetuned models).


\textbf{Finetuning Settings.} The finetuning datasets are constructed by augmenting some publicly  datasets (task related but without overlap) with GPT-4o~\citep{gilardi2023chatgpt} and Chain-of-Thoughts~\citep{CoT}. For each finetuning process, we use at least 180K training samples to ensure sufficient performance improvement on the corresponding task, which helps validate the effectiveness of our experiments (Details of constructing finetuning datasets in Appendix~\ref{appx:finetuning-data-generation} and hyperparameters in Appendix~\ref{appx:detailed-hyper-parameters}). \textit{To the best of our knowledge, this is the first LLM merging study with CoT enhanced finetuning and evaluated with generative tasks.}

\begin{table}[htbp]
\centering
\setlength{\abovedisplayskip}{-2pt}
\subfigbottomskip=-1pt
\subfigcapskip=1pt
\setlength{\abovecaptionskip}{-2pt}
\caption{Comparing performance of model merging methods on Llama 3.1 3B and Llama-3.2 8B.}
\resizebox{\textwidth}{!}{
\begin{tabular}{c|cccccc|cccccc}
\toprule
Model & \multicolumn{6}{c|}{Llama 3.1 3B} & \multicolumn{6}{c}{Llama-3.2 8B} \\
\midrule
Alg./Tasks & GSM. & TrA. & Wino. & H.Eval & MMLU & AVG. & GSM. & TrA. & Wino. & H.Eval & MMLU & AVG. \\
\midrule
Base & 27.52 & 57.71 & \textbf{69.69} & 22.56 & \textbf{54.08} & 46.31 & 56.33 & 72.39 & 73.64 & 27.44 & \textbf{67.99} & 59.56 \\
Math & \textbf{46.47} & 54.59 & 69.06 & 25.00 & 52.73 & 49.57 & \textbf{77.18} & 73.99 & 74.98 & 20.12 & 62.10 & 61.67 \\
QA & 32.75 & \textbf{61.45} & \textbf{69.69} & 28.05 & 54.17 & 49.22 & 69.60 & \textbf{74.14} & \textbf{75.45} & 31.71 & 62.21 & 62.43 \\
Code & 33.13 & 57.71 & 68.59 & \textbf{40.85} & 53.09 & \textbf{50.67} & 61.41 & 73.94 & 74.59 & \textbf{62.80} & 62.73 & \textbf{67.09} \\
All data & 44.12 & 47.74 & 69.21 & 34.76 & 53.75 & 49.92 & 70.89 & 69.77 & 75.06 & 48.17 & 62.94 & 65.37 \\
\midrule
TIES & 42.61 & 60.99 & 71.11 & 31.30 & 54.32 & 51.27 & 76.04 & 76.78 & 74.19 & 53.05 & 62.36 & 68.48 \\
PCB & 46.02 & 60.39 & 71.27 & 29.88 & 54.21 & 52.35 & 76.04 & 76.89 & 74.35 & 53.66 & 62.42 & 68.67 \\
Twin & 39.04 & 52.45 & 69.27 & 29.94 & 53.91 & 48.11 & 76.80 & 72.71 & 74.49 & 59.14 & 64.43 & 69.51 \\
BTX & 45.19 & \textbf{62.05} & 71.87 & 28.05 & 54.44 & 52.33 & 76.72 & 73.99 & 75.22 & 60.98 & 65.68 & 70.52 \\
\textbf{Mediator}  & \textbf{46.47} & 61.02 & \textbf{72.03} & \textbf{40.42} & \textbf{54.91} & \textbf{54.97} & \textbf{76.95} & \textbf{76.70} & \textbf{75.69} & \textbf{62.80} & \textbf{67.87} & \textbf{71.80} \\
\bottomrule
\end{tabular}
}
\label{tab:model_results_llama}
\end{table}

\begin{table}[htbp]
\centering
\caption{Comparing performance of model merging methods on Qwen 1.5 4B and Qwen 2.5 7B.}
\resizebox{\textwidth}{!}{
\begin{tabular}{c|cccccc|cccccc}
\toprule
Models & \multicolumn{6}{c|}{Qwen 1.5 4B} & \multicolumn{6}{c}{Qwen 2.5 7B} \\
\midrule
Alg./Tasks & GSM. & TrA. & Wino. & H.Eval & MMLU & AVG. & GSM. & TrA. & Wino. & H.Eval & MMLU & AVG. \\
\midrule
base & 47.16 & 44.54 & 56.75 & \textbf{41.46} & 54.45 & 48.87 & 83.41 & 51.67 & 67.68 & 67.68 & 67.70 & \textbf{67.63} \\
Math & \textbf{51.00} & 46.95 & 54.62 & 26.83 & 53.54 & 46.79 & \textbf{85.14} & 51.67 & 65.75 & 61.59 & 67.27 & 66.27 \\
Code & 43.29 & 46.39 & 54.14 & 43.29 & \textbf{54.82} & 48.39 & 52.31 & 49.47 & 64.64 & \textbf{71.95} & \textbf{72.30} & 62.13 \\
QA & 45.56 & \textbf{48.02} & \textbf{57.93} & 39.02 & 52.32 & 48.57 & 84.62 & \textbf{55.58} & 62.83 & 43.29 & 71.51 & 63.57 \\
all-sft & 48.52 & 47.73 & 55.88 & 39.14 & 53.93 & \textbf{49.04} & 64.90 & 52.98 & \textbf{69.30} & 65.85 & 69.66 & 64.59 \\
\midrule
TIES & 47.76 & 46.59 & 54.14 & 44.51 & 54.58 & 49.5 & 84.76 & 54.46 & 66.46 & 65.85 & 71.55 & 68.62 \\
PCB-merging & 47.83 & 47.60 & 56.75 & 43.90 & 54.58 & 49.93 & 73.46 & 53.90 & 69.53 & 60.98 & 71.41 & 65.86 \\
Twin-merging & 47.99 & 44.63 & 57.54 & 40.85 & 52.98 & 48.80 & 83.46 & 54.64 & 66.37 & 69.51 & 70.56 & 68.91 \\
BTX & 48.44 & 46.94 & 57.77 & 42.68 & 53.88 & 49.94 & 84.46 & 55.89 & 67.72 & 67.68 & 72.30 & 69.61 \\
Mediator & \textbf{50.94} & \textbf{48.20} & \textbf{57.85} & \textbf{45.12} & \textbf{54.87} & \textbf{51.40} & \textbf{85.14} & \textbf{56.06} & \textbf{69.30} & \textbf{71.95} & \textbf{72.56} & \textbf{71.00} \\
\bottomrule
\end{tabular}
}
\label{tab:model_results_qwen}
\end{table}

\textbf{Baselines.} 
We compare pretrained, finetuned models, and the state-of-the-art static and dynamic merging methods with Mediator. The \textit{static merging methods} include TIES~\citep{TiesMerging_NeurIPS2023} and PCB-merging~\citep{Du:2024aa} achieve the best performance in weighted average method and do not require calibration data, and also partly consider OOD evaluation tasks. The dynamic merging methods include BTX~\citep{sukhbaatar2024branchtrainmix} with token-level routing and the twin-merge~\citep{Lu:2024aa} with task-level routing and SVD decomposition (Details of hyperparameters and optimization of these baselines in Appendix~\ref{appx:detailed-hyper-parameters}).

\vspace{-0.0cm}
\subsection{Main Results}\label{sec:main-results}
\textbf{Fine-grained Comparison on All Tasks.} Table~\ref{tab:model_results_llama} and ~\ref{tab:model_results_qwen} show the fine-grained performance on each tasks and their overall averaged one of different methods and algorithms. In most of time, the finetuned LLM can achieve the best performance across all single and merged models on its specialized domain, like Math finetuned models on GSM8K and Code finetuned models on HumanEval. While merged LLMs can generally outperform single models on the averaged performance, their specialized domain performance is weaker. However, Mediator can catch up the domain performance of specialized models, and almost always outperform other merged models. Also, the overall performance on all tasks of Mediator is consistently better than other methods. Expert routing methods includes BTX and Mediator generally improve performance. This aligns with findings in TIES~\citep{TiesMerging_NeurIPS2023} and Twin-merging~\citep{lu2024twin}. As model scale increases, the improvement of all merging algorithms decrease, which may be attributed to enhanced comprehensive capabilities of individual finetuned models. 


\begin{wraptable}{r}{0.5\textwidth}
\centering
\setlength{\abovedisplayskip}{-2pt}
\subfigbottomskip=-1pt
\subfigcapskip=1pt
\setlength{\abovecaptionskip}{-2pt}
\caption{Overall Performance on all tasks.} 
\resizebox{\linewidth}{!}{
\begin{tabular}{l|cccccccccc}
\toprule
 Model/Algo. &  Qwen-4B & Llama-3B   &  Qwen-7B &  Llama-8B  \\
\midrule
 base model       &  48.87  &  46.31  &   67.63    &  59.56 \\
 all data sft     &  49.04  &  49.92   &  64.59   &  65.37 \\
 TIES             &  49.50  &  51.27   &  68.62  &  68.48    \\
 Twin           &   48.80    &   48.11    &    68.91       &  69.51 \\
 PCB             &  49.93  &  52.35   &  65.86    &  68.67 \\
 BTX             &   49.94  &  52.33    &  69.61   &  70.52\\
 \textbf{Mediator}             &  \textbf{51.40} ($\uparrow$2.9\%) &  \textbf{54.97}($\uparrow$5.0\%)  &  \textbf{71.00} ($\uparrow$2.0\%)  &  \textbf{71.80} ($\uparrow$1.8\%)\\
\bottomrule
\end{tabular}}
\label{tab:major-results}
\end{wraptable}

\vspace{-0.0cm}


\textbf{Overall Comparison.} As shown in Table~\ref{tab:major-results}, the advantages of PCB over TIES become less pronounced at larger model scales, and even shows performance degradation on Qwen-7B, which demonstrates PCB's instability. Dynamic routing approaches include BTX and Mediator show stable performance improvements. Our method demonstrates consistent improvements across different models.





\begin{wraptable}{r}{0.4\textwidth}
\centering
\setlength{\abovedisplayskip}{-2pt}
\subfigbottomskip=-1pt
\subfigcapskip=1pt
\setlength{\abovecaptionskip}{-2pt}
\caption{Post Training Time (Hours).} 
\resizebox{\linewidth}{!}{
\begin{tabular}{l|cccc}
\toprule
Model/Algo & LlaMA-3B & LlaMA-8B & Qwen-4B & Qwen-7B\\
\midrule
Twin       & 1.33     & 1.87    & 1.60    & 1.80    \\
PCB        & 3.42     & 5.75    & 5.73    & 5.80    \\
BTX        & 7.55     & 12.52   & 8.83    & 12.18   \\
Mediator   & 1.35     & 2.03    & 1.57    & 1.78    \\
\bottomrule
\end{tabular}}
\label{tab:system-performance-post-training-time}
\end{wraptable}
\textbf{Post-Training Time After Merging.} As many model merging methods like Twin, PCB and BTX require post-training, it is critical to compare the extra training time. Table~\ref{tab:system-performance-post-training-time} shows the post-training time of different methods. PCB merging require weight exploration thus leads to higher time. The BTX with token-level routing needs to completely train the layer-wise routers for each token, thus, the post-training time of them is significantly high. In contrast, for task-level routing approaches like Twin-merging and Mediator, taking the lowest time.



\textbf{Inference Time.} Table~\ref{tab:system-performance-inference-time} shows the inference time of different methods. As the token-level routing methods need to load and compute the layer-wise routers for each token, the inference time of them is significantly higher (more than 2x) than our method. We have detailed our inference acceleration process in Section~\ref{sec:SystemOptimization} and Appendix~\ref{appx:system-optimization}.

\begin{table}[htbp]
    \centering
    \begin{minipage}{0.48\textwidth}
        \centering
        \setlength{\abovedisplayskip}{-2pt}
        \subfigbottomskip=-1pt
        \subfigcapskip=1pt
        \setlength{\abovecaptionskip}{-2pt}
            \caption{Inference time per sample (seconds).} 
            \resizebox{\linewidth}{!}{
            \begin{tabular}{c|cccc}
                \toprule
                 Model/Algo &  LlaMA-3B &   LlaMA-8B &  Qwen-4B &  Qwen-7B\\
                            &  32 layers & 32 layers & 40 layers & 28 layers \\
                \midrule
                Base model   & 1.452 & 3.600 & 3.112 & 3.057 \\
                Twin & 1.725 & 4.151 & 3.792 & 3.648 \\
                BTX  & 3.237 & 8.68 & 7.082 & 7.153 \\
                Mediator  & 1.609 &  4.053 & 3.674 & 3.489 \\
                \bottomrule
        \end{tabular}
        }\label{tab:system-performance-inference-time}
    \end{minipage}
    \hfill
    \begin{minipage}{0.48\textwidth}
        \centering
        \setlength{\abovedisplayskip}{-2pt}
        \subfigbottomskip=-1pt
        \subfigcapskip=1pt
        \setlength{\abovecaptionskip}{-2pt}
        \caption{Comparing memory costs.} 
        \resizebox{\linewidth}{!}{
        \begin{tabular}{l|cccc}
            \toprule
            Model/Algo & LLaMA-3B & LLaMA-8B & Qwen-4B & Qwen-7B\\
            \midrule
            Base model   & 9G & 33G & 11G & 31G \\
            \midrule
            Twin-merging & 10G & 35G & 13G & 32G \\
            BTX (MOE) & 37G & 80G & 40G & 78G \\
            Mediator  & 10G &  35G & 13G & 33G \\
            \bottomrule
        \end{tabular}}
        \label{tab:system-performance-memory-costs}
    \end{minipage}
\end{table}

\textbf{Memory Cost.} Table~\ref{tab:system-performance-memory-costs} shows the memory costs of different methods. Our method significantly reduces the memory costs compared to saving all finetuned models because there is only one router for all experts, and the sparsified experts saving, and the layer-wise merging strategy (Details of formally comparing memory costs of merging methods in Appendix~\ref{appx:more-related-works}).

\begin{wraptable}{r}{0.6\textwidth}
\centering
\setlength{\abovedisplayskip}{-2pt}
\subfigbottomskip=-1pt
\subfigcapskip=1pt
\setlength{\abovecaptionskip}{-2pt}
\caption{Performance of scaling up finetuned models.} 
\resizebox{1.0\linewidth}{!}{
\begin{tabular}{l|llllllllll}
\toprule
Alg./Tasks & GSM. & TrA. & Wino. & H.Eval & MMLU & I.Eval & C.Eco & C.Med & C.law & AVG.\\
\midrule
Base & 47.16 & 44.54 & 56.75 & 41.46 & 54.45 & 30.70 & 49.09 & 55.10 & 41.67 & 46.77 \\
\midrule
Math & \textbf{50.95} & 46.95 & 54.62 & 26.83 & 53.54 & 34.05 & 45.45 & 55.10 & 45.83 & 45.92 \\
QA & 45.56 & \textbf{48.02} & \textbf{57.93} & 39.02 & 52.32 & 31.65 & 43.64 & 59.18 & 45.83 & 47.02 \\
code & 43.29 & 46.39 & 54.14 & \textbf{43.29} & 54.82 & 31.65 & 43.64 & 59.18 & 45.83 & 46.91 \\
Instruct. & 47.54 & 40.96 & 55.09 & 37.80 & \textbf{54.88} & \textbf{38.37} & 52.73 & 59.18 & 50.00 & 48.51 \\
Economy & 45.56 & 46.24 & 57.93 & 28.86 & 54.21 & 32.13 & \textbf{56.36} & 55.10 & 45.83 & 46.78 \\
medicine & 39.12 & 44.50 & 56.67 & 1.83 & 54.63 & 28.30 & 50.09 & \textbf{61.22} & 41.66 & 42.00 \\
Law & 40.64 & 46.64 & 56.59 & 0.61 & 54.61 & 27.94 & 41.82 & 57.14 & \textbf{58.33} & \textbf{47.15} \\
All Data  & 43.75 & 46.25 & 56.43 & 40.85 & 54.60 & 35.37 & 49.09 & 55.10 & 41.67 & 46.90 \\
\midrule
TIES & 47.38 & 47.19 & 55.80 & 36.59 & 55.38 & 34.17 & 40.00 & 57.14 & 50.00 & 47.07 \\
PCB & 47.38 & 47.19 & 55.80 & 36.59 & 55.45 & 34.29 & 41.82 & 57.14 & 50.80 & 47.38 \\
Twin & 47.91 & 44.78 & 57.54 & 40.85 & 53.01 & 37.53 & 53.32 & 59.56 & 50.00 & 49.38 \\
BTX & 48.44 & 46.94 & 57.85 & 42.68 & 54.93 & 36.93 & 54.40 & 60.36 & 58.33 & 51.18 \\
\textbf{Mediator} & \textbf{50.64} & \textbf{48.04} & \textbf{57.93} & \textbf{44.51} & \textbf{55.12} & \textbf{38.50} & \textbf{56.01} & \textbf{61.17} & \textbf{58.33} & \textbf{52.25} \\
\bottomrule
\end{tabular}
}
\label{tab:scaling-up-finetuned-models}
\vspace{-0.0cm}
\end{wraptable}

\subsection{Ablation Studies}\label{sec:AblationStudies}
\textbf{Scalability of Finetuned Models.} 
To verify the scalability of Mediator, we finetune another 4 LLMs according to the following 4 extra evaluation tasks including: (1) \textit{Instruction Following} with IFEval.~\citep{zhou2023ifeval} which assess models ability to accurately interpret and execute natural language instructions; (2-4) \textit{Medicine, College Economics and Law} from CEval.~\citep{huang2023ceval} which assess knowledge and capabilities across various academic and professional domains. We utilize four accordingly domain datasets for finetuning including Magpie \citep{xu2024magpie}, IndustryInstruction \citep{IndustryInstruction_Finance-Economics}, DISC-Med \citep{bao2023discmedllm}, DISC-Law \citep{yue2023disclawllm} without overlap with IFEval and CEval (Details in Appendix~\ref{appx:finetuning-data-generation}). 

Table~\ref{tab:scaling-up-finetuned-models} demonstrate several key findings: 
1) Static merging methods like PCB and TIES show diminished performance improvements after task expansion, performing even worse than instruction-following finetuned models in overall scores. The similar performance between PCB merging and TIES aligns with findings from the TIES paper, which noted degraded model capabilities when merging more than three tasks; 2) Dynamic merging approaches like BTX, Twin and Mediator maintain relatively stable performance after task expansion; 3) Mediator consistently outperforms BTX by a margin of 2.09\% and achieves the best scores across all individual tasks, showing its good scalability.



\begin{wraptable}{r}{0.5\textwidth}
\centering
\setlength{\abovedisplayskip}{-0pt}
\subfigbottomskip=-0pt
\subfigcapskip=0pt
\setlength{\abovecaptionskip}{-0pt}
\caption{Model Performance w/o layer-wise merging.}
\resizebox{1.0\linewidth}{!}{
\begin{tabular}{l|ccc}
    \toprule
    Model scale & with averaging & w/o averaging & perf. gap \\
    \midrule
    Qwen-1.5 & 51.40 & 51.43 & -0.06\% \\
    Qwen-2.5 & 71.00 & 71.29 & -0.27\% \\
    \bottomrule
\end{tabular}
}
\label{tab:layer-wise-compression}
\end{wraptable}
\vspace{-0.0cm}
\textbf{Layer-wise Merging.} Layers with low conflicts are averaged thus reducing $n_\tau \times$ memory occupation. Table~\ref{tab:layer-wise-compression} averaging or not on Qwen-1.5 with 4 and 8 experts show almost no performance loss of the parameter averaging. Given the significant compression ratios achieved (3.5x for 4 experts, 7x for 8 experts), this minor performance trade-off is acceptable.

\textbf{Comparison between Task-level routing and Token-level routing.}
Fig.~\ref{fig:train_token_heat_map} in Appendix demonstrates the expert selection probabilities of the BTX model across different tasks (MMLU, math, coding, and QA). The results indicate that 1) bottom and upper layers show obvious task preference; 2) middle layers suggesting some shared logical processing across tasks in these layers, which explains why averaging middle layers that have less conflicts in Mediator results in small performance loss.

\begin{wraptable}{r}{0.5\textwidth}
\centering
\setlength{\abovedisplayskip}{-2pt}
\subfigbottomskip=-1pt
\subfigcapskip=1pt
\setlength{\abovecaptionskip}{-2pt}
\caption{Performance under Different Temperature $\beta$.}
\resizebox{1.0\linewidth}{!}{
\begin{tabular}{c|ccccccc}
\toprule
Temperature & 0.1 & 0.5 & 1.0 & 1.25 & 1.50 & 1.75 & 2.0 \\
\midrule
Mediator & 50.31 & 50.35 & 50.64 &  51.02 & 51.40 & 50.84 & 50.92 \\
\bottomrule
\end{tabular}
}
\label{tab:temperature}
\end{wraptable}

\textbf{Different Temperatures.} 
In the routing process, the hyperparameter temperature $\beta$ is a key factor. Table~\ref{tab:temperature} shows performance change of Mediator with different temperatures. Results show that only around 2.2\% score variation, the $\beta=1.5$ achieves the highest performance, $\beta=0.1$ almost equals to the Top-1 routing, results in the lowest performance.


\vspace{-0.2cm}
\section{Conclusion}\label{sec:Conclusion }
\vspace{-0.2cm}
We propose Mediator for merging LLMs that addresses parameter conflicts through adaptive layer-wise strategies, which average low-conflict layers and routing high-conflict ones via task-specific experts. This preserves individual model strengths while integrating shared knowledge, improving performance and reducing system costs. By decomposing experts into a dense core and sparse components, Mediator minimizes storage without sacrificing efficacy. Dynamic expert selection via task uncertainty enhances adaptability across diverse inputs. Mediator advances efficient LLM merging methods, balancing resource constraints with practical versatility. Future work can consider exploring theoretical foundations of parameter conflicts, large-scale deployment optimizations, and faster expert loading mechanisms.





\bibliography{cite}



\newpage
\clearpage
\onecolumn
\appendix 
\etocdepthtag.toc{mtappendix}
\etocsettagdepth{mtchapter}{none}
\etocsettagdepth{mtappendix}{subsection}
\renewcommand{\contentsname}{Appendix}
\tableofcontents 
\clearpage

\vspace{-0.0cm}
\section{Limitations}\label{sec:Limitations}
\vspace{-0.0cm}
\textbf{Parameter Conflict Theory.} Our empirical analysis in Section~\ref{sec:Understanding-Conflict} provides initial insights, but a theoretical foundation for parameter conflicts remains unexplored.


\textbf{Scaling Challenges.} While scaling to 8 models is feasible, expanding to hundreds or thousands (e.g., for personalization~\citep{chan2024scaling}) poses deployment challenges.

\textbf{Loading Time.} Loading experts per sample remains a system bottleneck. Although optimized (Section~\ref{sec:SystemOptimization}), faster loading requires further research.

\section{Broader Impact}\label{apdx:impact}

\textbf{Societal Impacts.} Our approach demonstrates significant effectiveness by enabling the deployment of merging 7B x 4 LLMs with only 24GB VRAM. Compared to ensemble learning with these models, our method not only maintains better accuracy but also requires significantly less computational resources and demonstrates superior performance. This breakthrough in resource efficiency makes advanced language models more accessible and cost-effective.

\textbf{Ethical Concerns.} We declare no conflicts of interest that could inappropriately influence our work. All experiments were conducted using publicly available resources. Our study does not involve human subjects, data collection from individuals, or experiments on protected groups. The models and basic datasets used in this work are publicly available and widely used in the research community. We have made efforts to ensure our experimental design and reporting of results are fair, unbiased, and do not misrepresent the capabilities or limitations of the methods presented.

\textbf{Reproducibility.} For openness of LLM research, we declare our code and the CoT enhanced crafted finetuning datsets will be made available to ensure reproducibility. We will provide detailed documents of code implemnetation. And we have provided the details of all hyper-parameters of implementing Mediator and optimizing baselines.

\textbf{Potential Applications.} The technology may have significant potential across specialized vertical domains. Considering that many vertical domains,  personalized LLM agents~\citep{li2024personalllmagentsinsights}, LLM applications like roleplay chatting~\citep{chan2024scaling,} and professional domain-specific writing~\citep{gómezrodríguez2023confederacymodelscomprehensiveevaluation}, an LLM service provider may need to simultaneously deploy different finetuned LLMs. Our technology enables efficient and effective serving multiple popular LLM applications, and merging knowledge from different LLMs together.



\section{More Related Works}\label{appx:more-related-works}

We introduce more related works about model merging and routing in this section. Current common methods include: (1) \textbf{Averaging based merging.} This direction combines multiple models into a single model while preserving their capabilities with minimal or no additional training; (2) \textbf{Routing based merging.} It considers to route inputs to specialized expert modules like mixture-of-experts (MoE) approaches but completely different.


Table~\ref{tab:demystify-detail} provides an overview of comparing different model merging methods and our framework. Given $n_\tau$ different finetuned models, averaging based methods do not completely address the parameter conflicts, thus having higher parameter conficts than routing based merging. The weighted averaging requires calibration data to compute the importance metrics. The token-level routing requires routing for each layer, thus having totally $n_{\Lcal}$ routers. Our framework Mediator exploits layer-wise characteristics to both reduce parameter conflict and improve common knowledge fusion. And Mediator utilizes compression to further reduce the memory costs. To the best of our knowledge, the most of previous model merging works focus on experiments on traditional CV and NLP tasks~\citep{FisherMerging_NeurIPS2022}, while Mediator conducts experiments on modern LLMs and real-world experiments.

Table~\ref{tab:demystify-system} provides the system performance comparison. Because token-level routing like MoE requires to route each token towards different (possibly) experts, its inference cost is significantly large as more than $T\times n_{\Lcal}$ times than task-level routing, where $T$ is the sequence length. Besides, the token-level routing requires more than $n_{\Lcal}$ times routers in memory costs. With the layer-wise adaptive averaging and Routing, Mediator significantly reduce the memory costs of from $M_{\theta}\times n_{\tau}$ to $M_{\theta} \times (c_{\text{avg}} + c_{\text{route}} \times n_{\tau} \times c)$~\footnote{Normally, each transformer layer occupies the same memory.}. The experimental memory reduction and the system performance comparisons are shown in the Section~\ref{sec:main-results}. We also provide system optimization to accelerate the inference during the deployment of Mediator in Section~\ref{sec:SystemOptimization} and Appendix~\ref{appx:finetuning-data-generation}. And the hyper-parameters $c_{\text{avg}}$ and $c_{\text{route}}$ are adaptively decided by the parameter conflict estimation.


\begin{table*}[htb!]
    \centering
    \caption{Demystifying different merging methods. The $n_\tau$ represents the number of finetuning tasks, $n_{\Lcal}$ the number of layers in the model.}
    \vspace{-0.1cm}
    \centering
    \begin{adjustbox}{max width=\linewidth}
    \begin{tabular}{c|cccccccc}
    \toprule[1.5pt]
    Method Type &  \makecell{Parameter \\ Conlict Level } & \makecell{Merging \\ Common Knowledge}   &  \makecell{Require \\ Calibration Data} & \makecell{Routing \\ Type} & \makecell{Considering \\ Layer-wise Characteristics } &  \makecell{Considering\\ Compression} &  \makecell{Considering \\ OOD Samples}  &   \makecell{Experimental \\ Scanario} \\
    \midrule[1.5pt]
    Basic Averaging &  High & \cmark  & \xmark    & NA  & \xmark & \xmark  &  \xmark & Traditional CV, NLP \\
    Weighted Averaging & Middle & \cmark   & \cmark & NA  & \xmark & \xmark & \xmark & Traditional CV, NLP \\
    Subspace Averaging & Middle & \cmark  & \xmark & NA  & \xmark & \xmark & \xmark & Traditional CV, NLP \\
    Token-Level Routing & \textbf{Low} & \xmark    & \xmark & Token-level  & \xmark & \xmark & \xmark & Traditional CV, NLP \\
    Task-Level Routing & \textbf{Low} &  \xmark  & \xmark & Task-level  & \xmark & \xmark & \xmark & Traditional CV, NLP \\
    Mediator & \textbf{Low} & \cmark    & \xmark  & Task-level & \cmark & \cmark & \cmark & Generative LLMs   \\
    \bottomrule[1.5pt] 
    \end{tabular}
    \end{adjustbox}
    \vspace{-0.1cm}
    \label{tab:demystify-detail}
\end{table*}

\begin{table*}[htb!]
    \centering
    \caption{Demystifying different merging methods in system performance costs. Considering the memory costs of the base model and one router are $M_{\theta}$ and $M_{h}$, each layer occupies the same memory $M_l$, compression ratio $c$, the ratio of selected layers for averaging is $c_{\text{avg}}$, for routing is $c_{\text{route}}$, $FP_{\theta}$ and $BP_{\theta}$ are  the forward time and backward time of the model. $FP_{h}$ and $BP_{h}$ are the forward time and backward time of the router.}
    \vspace{-0.1cm}
    \centering
    \begin{adjustbox}{max width=0.8\linewidth}
    \begin{tabular}{c|cccc}
    \toprule[1.5pt]
    Method Type &  \makecell{Requiring \\ \# of routers}  & \makecell{Costs of \\ Trainig Routers}   & \makecell{Memory Costs \\ After Merging} & \makecell{Inference \\ Cost} \\
    \midrule[1.5pt]
    Basic Averaging & NA  & NA & $M_{\theta}$ & $FP_{\theta}$  \\
    Weighted Averaging &  NA  & NA & $M_{\theta}$ & $FP_{\theta}$ \\
    Subspace Averaging & NA  & NA & $M_{\theta}$ & $FP_{\theta} $ \\
    Token-Level Routing & $n_{\Lcal} \times n_{\tau} $  & High & $ M_{\theta} \times n_{\tau} + M_{h} \times n_{\Lcal} $ & $FP_{\theta}  +  FP_{h} \times T \times n_{\Lcal} \times n_{\tau} $   \\ 
    Task-Level Routing & $ n_{\tau} $  & Middle & $M_{\theta}  \times n_{\tau} + M_{h} $ & $FP_{\theta}  +  FP_{h} \times n_{\tau} $    \\
    Mediator & $ n_{\tau} $  & Low & $M_{\theta} \times (c_{\text{avg}} + c_{\text{route}} \times n_{\tau} \times c) + M_{h} $ & $FP_{\theta}  +  FP_{h} \times n_{\tau} $  \\
    \bottomrule[1.5pt] 
    \end{tabular}
    \end{adjustbox}
    \vspace{-0.1cm}
    \label{tab:demystify-system}
\end{table*}

Besideds, we also review some highly related works include following directions that are closely related to our framework. Insights from these directions have provided valuable guidance for our framework. 
\begin{enumerate}[leftmargin=*]
\item \textbf{Layer-wise training dynamics and optimization.} This direction discusses the layer-wise training dynamics to help shed some light on the paramter conflicts and the layer-wise adaptivity.
\item \textbf{Bayesian deep learning.} This direction reviews some works of the Bayesian deep learning, discussing the uncertainty and Bayesian model averaging.
\item \textbf{OOD Detection \& Generalization.} This direction reviews some works of the OOD Detection and Generalization, shedding light on deployment of the model merging on the out-of-distribution data.
\item \textbf{Model compression.} This direction shortly review some works about the model compression, in which many methods can be directly applied into our framework to further reduce the memory costs. Note that in our paper we propose a general framework instead of a new model compression method. Different model compression methods can be combined into our framework.
\item \textbf{Data Synthesis.} This direction reviews some works about how to generate new synthetic data to improve the model merging performance. In our framework, we exploit the CoT to generate new synthetic data to improve the finetuning performance on downstream tasks, which is a real-world downstream task instead of traditional model fine-tuning using the in-domain training and testing data. 
\end{enumerate}

\subsection{Averaging-based Model Merging}
Model merging, also known as model fusion, combines the parameters of multiple separate models with different capabilities to create a universal model. In this paper, we temporarily focus on models that have the same architecture but different parameters that are finetuned on different downstream tasks.

\textbf{Basic Averaging.} The traditional approach to merge different trained or finetuned models is to evenly average the parameters of different models~\citep{utans1996weight,shoemake1985animating}. This process does not require access to the original training data and allows for enhanced performance without the need for expensive computation. However, the performance of these simply weight averaging is generally unsatisfactory. 

Some related directions of model averaging also include Federated learning (FL)~\citep{FedAvg_AISTAT2017,tang2020survey}. In FL, the model averaging is performed on the server side to reduce the communication costs after the local training. Many methods have been proposed to stable and smooth the model averaging process~\citep{Fednova_2020,jhunjhunwala2024fedfisher,yurochkin2019PFNM,otfusion_neurips2020,FederatedMA_ICLR2020,tang2024fedimpro,VHL} to enhance the averaging performance. Different from the multi-rounds FL, the model merging is performed in a single round, which is more similar to the one-shot FL~\citep{guha2019one,tang2024fusefl}.

\textbf{Weighted Averaging.} Rethinking the cause of the poor performance of the basic averaging method, many works propose to use the weighted averaging method to improve the merging performance. Intuitively, different model parameters have different importance on downstream tasks. Such a heterogeneity of the parameter importance motivates other research directions including model sparsification~\citep{sun2024a,dong2024pruner,dong2024stbllmbreaking1bitbarrier,tang2020survey}, continual learning~\citep{catastrophicforgetting_1995,kirkpatrick2017overcoming,zhu2024model,marczak2024magmax} and FL~\citep{jhunjhunwala2024fedfisher,yurochkin2019PFNM,otfusion_neurips2020}. Thus, to avoid the important parameters being overwhelmed by the unimportant parameters, during averaging, we can assign large weights to those important parameters. To this end, the importance measurement is crucial. Many works propose to use \textit{fisrt or second orders of Taylor expansion} to measure the importance of the parameters~\citep{lee2018snip,jhunjhunwala2023towards,qu2022generalized}. Some works employ \textit{local linearization and task vectors} to measure the importance of the parameters~\citep{zhou2024metagpt}. The \textit{fisher information} also a kind of importance measurement~\citep{FisherMerging_NeurIPS2022,jhunjhunwala2024erasure,thennal2024fisher,jhunjhunwala2024fedfisher,thennal2024fisher,daheim2024model}.

While these importance measurement methods can improve the merging performance than the basic averaging method, they still face some typical challenges. 
\begin{enumerate}[leftmargin=*]
\item \textit{Require Calibration Dataset.} The importance measurement is based on the calibration dataset. In the LLM era, the pretrained dataset is significantly large, it is difficult to collect the complete pretrained dataset and measure the importance of the parameters on it.
\item \textit{Computation Costs.} Because that the importance measurement is based on the calibration dataset, the computation costs is almost similar to conduct the complete forward process of the different models. In traditional small models, such a computaton cost is acceptable. However, in the LLM era, the model size is significantly large, such a computation cost is unbearable.
\item \textit{Unaddressed Parameter Conflicts.} While methods in these importance based weighted averaging methods can improve the merging performance, they still face the parameter conflicts between different models. Because of the highly non-convex structure of the LLMs, it is difficult to find a optimal merging method based on averaging the parameters of different models.
\end{enumerate}

\textbf{Subspace Averaging.}  
Considering that the neural networks are over-parameterized, removing most of the parameters from the model barely affects its accuracy~\citep{he2023structuredpruning,choudhary2020comprehensive}. Besides, during the training or finetuning, some parameters might be optimized towards a random direction which has small impact on the model performance~\citep{TiesMerging_NeurIPS2023}. Thus, works propose to firstly process different models in a subspace manner. Then, the parameter conflicts can be mitigated by the subspace averaging methods~\citep{deep2024della,he2024localize}.

DARE (Drop and Rescale)~\citep{Yu:2023aa, DARE_Arxiv2023} introduces a parameter pruning and rescaling strategy that significantly reduces the number of parameters in SFT models while preserving their performance, thereby serving as an effective preprocessing step for model merging. Similarly, Model Breadcrumbs~\citep{davari2023modelbreadcrumbs} enhances sparsification by eliminating both low-magnitude parameters and outlier parameters with exceptionally high weights, thereby reducing noise and improving the generalization of hyperparameters during model merging.

TALL-masks~\citep{wang2024localizingtall} creates task-specific mask matrices based on predefined thresholds tailored to individual models, while Model Tailor~\citep{zhu2024model} further refines this approach by masking parameters according to their sensitivity to loss changes and deviations from pre-trained values. APL~\citep{kong2024activated} advances parameter importance estimation through causal interventions, providing a robust metric for selective parameter retention.

EMR-Merging~\citep{emrmerging_arxiv2024} departs from traditional model merging by maintaining a shared model across multiple tasks alongside sparse task-specific models, where each shared parameter is determined by the maximum value among corresponding parameters from all models. Concrete~\citep{tang2023concrete} further innovates by framing mask construction and model merging as a learnable bi-level optimization problem, with the outer level optimizing the mask matrix and the inner level performing model merging and optimization utilizing unlabeled test samples.

Task Arithmetic~\citep{Ilharco:2022aa} exploits parameter-space arithmetic operations, treating model parameters as vectors and employing addition and subtraction to synthesize new model capabilities. However, many of these approaches, including DARE and Task Arithmetic, heavily rely on hyperparameters for parameter fusion, which can negatively impact the performance of model merging. Additionally, as highlighted in studies such as TIES~\citep{Yadav:2023aa} and Crisostomi~\citep{Crisostomi:2024aa}, model merging often encounters parameter conflicts that degrade performance when integrating multiple models.

Addressing these challenges, TIES (Trim, Elect, and Disjoint Merge)~\citep{Yadav:2023aa} implements a comprehensive approach by trimming parameters based on magnitude, selecting relevant weights, and disjointly merging weights using outcomes from task arithmetic operations. This methodology mitigates parameter conflicts and enhances the overall performance of the merged model, positioning TIES as a robust solution in the domain of model merging.

\subsection{Routing-based Model Merging.}
Average-based methods primarily aim to enhance the averaging process of client models. However, the inherently non-linear architecture of deep neural networks complicates the derivation of a globally comparable model through simple averaging. 

The basic, weighted-based, and subspace-based merging methods are \textit{static} merging techniques. This implies that the merged model remains consistent across all samples or tasks. Given the variability among input samples and tasks, the model's performance can fluctuate when processing diverse inputs. To this end, certain studies advocate for the \textit{dynamic} merging of models (or subsets of layers) tailored to specific samples or tasks~\citep{li2023merge,muqeeth2024soft,tang2024moemerging,lu2024twin,kang2024self,tang2024fusefl,shen2024hot} during the inference phase.

For each input instance, SMEAR~\citep{muqeeth2024soft} initially computes a weighted average of the parameters from each expert by leveraging the distribution of router inputs to the expert modules. This approach maintains a computational cost comparable to that of a single expert. Similarly, Twin-Merging~\citep{lu2024twin} adaptively integrates task-shared and task-specific knowledge based on routing mechanisms during inference. In the same vein, Weight-Ensembling MoE~\citep{tang2024moemerging} introduces a dynamic merging Transformer architecture. This method identifies that the parameters of the linear layer in the fine-tuned model undergo more significant changes compared to the nonlinear layers, which adversely affects merging performance. Consequently, Weight-Ensembling MoE employs a standard weighted average for all modules except the linear layer, which is dynamically weighted and merged based on the routing network (utilizing sample features as input and merging coefficients as output) during inference. PWE MoE~\citep{tang2024towards} extends Weight-Ensembling MoE to a multi-objective optimization framework, incorporating the preference vector as an input for routing.

\textit{AdaMerging}~\citep{Yang:2023aa} adaptively learns merging coefficients in a task-aware or layer-wise manner, offering an automated and unsupervised approach to task arithmetic. While this method significantly enhances performance, it incurs high computational costs. \textit{PCB Merge}~\citep{Du:2024aa} introduces a parameter importance detection mechanism that accounts for parameter conflicts and employs heuristic algorithms to explore model fusion parameters, thereby achieving superior results. \textit{TwinMerge}~\citep{Lu:2024aa} utilizes LoRA or SVD techniques in conjunction with supervised training for parameter fusion, resulting in improved performance.

Nevertheless, these methods encounter inherent limitations. Both AdaMerging and PCB Merge utilize static fusion approaches, which can lead to performance degradation when the actual sample distribution varies during runtime. Meanwhile, TwinMerge performs parameter fusion at the task level; however, the application of LoRA and SVD matrix decomposition markedly reduces model accuracy and introduces substantial online computational overhead. Besides, the code implementation of the TwinMerge actually exploits the LoRA finetuning to replace SVD decomposition. Using SVD decomposition in compressing model parameters leads to disturbed LLMs and significantly degraded model performance.

\textit{Mixture-of-Experts (MoE)}~\citep{Jacobs:1991aa, Jordan:1994aa} is a foundational model concatenation and routing strategy comprising multiple expert networks and a router that dynamically selects relevant experts based on the input. This methodology has been extensively adopted in large language models, offering significant reductions in computational costs while preserving model performance. Recent studies, particularly sparse gated MoE~\citep{Shazeer:2017aa} in transformer-based large language models~\citep{Lepikhin:2020aa}, have concentrated on maintaining load balancing among experts during training~\citep{Zhou:2022aa, Jiang:2024aa}, reducing training costs~\citep{Dai:2024aa}, and mitigating performance degradation due to uncoordinated expert training~\citep{Chi:2022aa}.

\textit{Upcycling Methods} have been developed to alleviate the high computational demands of training MoE models from scratch by initializing experts from existing dense models. These methods encompass copying existing dense models as experts~\citep{He:2024aa, Wei:2024aa}, introducing noise to the MLP layers of dense models to create experts (Noise upcycling)~\citep{Chen:2024aa}, and drop upcycling~\citep{anonymous2024dropupcycling}, which combines parameter dropout with expert copying during training to enhance model robustness, reduce overfitting, and improve performance.

\textit{Branch-Train-Merge (BTM)}~\citep{li2022branch} and \textit{Branch-Train-Mix (BTX)}~\citep{sukhbaatar2024branchtrainmix} are methodologies aimed at further optimizing model training efficiency. These approaches employ different SFT-trained dense models derived from the same base LLM as MoE experts. The experts are interconnected via a router without necessitating additional training, while non-expert components are amalgamated through model merging techniques such as parameter averaging. Only the router undergoes training, thereby substantially reducing overall training costs. Although these methods achieve lower training expenses and marginally outperform traditional model merging approaches, our research indicates that token-level routing can partially degrade model performance. Additionally, maintaining all experts in GPU memory leads to significant parameter redundancy and escalates inference costs, which motivates our ongoing research endeavors.

However, the token-level routing methods are not suitable for model merging. We have provided detailed discussions in the main text Section~\ref{sec:AdaptiveMerging} and Appendix~\ref{appx:theoretical-analysis}. The token-level routing methods after merging normally require re-training based on all training datasets to obtian a better token-level router, which significantly increases the computational costs, which is discussed in the main text Section~\ref{sec:SystemOptimization} and Appendix~\ref{appx:system-optimization}.

\textbf{LoRA based Routing.} Routing samples to different LoRA experts is a promising direction to dynamically route the input to different LoRA experts. This direction includes the \textit{LoraHub}~\citep{Huang:2023aa} and \textit{sLora}~\citep{Babakniya:2023aa}, which explore serving multiple LoRA adapters through techniques like unified paging and tensor parallelism. However, these methods do not consider the better dynamic expert merging method to further improve the model merging performance. In real-world applications, the input distribution is dynamic and the input samples are diverse, which motivates our ongoing research endeavors. Besides, their reliance on LoRA matrix decomposition significantly degrades model serving performance.  Additionally, they do not consider model compression opportunities or the potential to average similar layers between models, which could further optimize storage and computation costs while maintaining model capabilities.

\subsection{Layer-wise Training Dynamics and Optimization}

Layer-wise training was initially explored to achieve effective initialization~\citep{hinton2006fast, bengio2006greedy}. From the perspective of the information propagation~\citep{tishby2000information,mahabadi2021variational,tishby2015deep}, the fundamental issue with layer-wise training is that each layer is unable to access information from the layers that precede it. Some works~\citep{xiong2020loco} proposed a method that permits backpropagation within a local block, allowing information from subsequent layers to progressively influence earlier layers by training them sequentially. Furthermore,~\citep{gomez2022interlocking} builds upon the concept of ``overlapping local updates'', introducing a learning strategy that harmonizes the high parallelism characteristic of layer-wise training with the superior predictive accuracy associated with end-to-end (E2E) learning. Besides, classification-based loss functions are employed at each layer~\citep{mostafa2018deep, belilovsky2019greedy, belilovsky2020decoupled}, whereas similarity-based loss functions are utilized in other scenarios~\citep{kulkarni2017layer, nokland2019training, siddiqui2023blockwise}. Additionally,~\citep{wang2020revisiting} incorporates a reconstruction error term into the local objective function, drawing from an information-theoretic perspective.

Some works find that different layers have different convergence rates during the whole training process~\citep{raghu2017svccasingularvectorcanonical}. This property can be used to freeze front layers and only train the later layers, thus reducing the training costs. The PipeTransformer~\citep{he2021pipetransformer} utlizes this property to reduce the training costs of transformer models.

LISA~\citep{Pan:2024aa} discovered that the weight norm distributions across layers in LoRA and full parameter fine-tuning are skewed, indicating varying layer importance in large-scale LLM training. Based on this observation, LISA applies importance sampling to different layers in LLMs, randomly freezing most intermediate layers during optimization. It periodically samples Transformer layers from the model, randomly selecting r layers for fine-tuning while keeping others frozen. The initial word/position embeddings (wte/wpe) and final language modeling head ($\text{lm\_head}$) are consistently fine-tuned. This aligns with our observations regarding layer merging.

Layer-wise model training and merging approaches have also provided inspiration for our research direction.~\citep{Li:2024aa} discovered that in the field of large language models, the effectiveness of deeper layers gradually diminishes, with many studies showing that deeper layers can be pruned without significantly affecting model performance - a phenomenon often viewed as an opportunity for model compression. To address this, they proposed a novel normalization technique called Mix-LN, which combines pre-LN and post-LN within the same model. Specifically, Mix-LN applies post-LN to earlier layers and pre-LN to deeper layers, ensuring more uniform gradients across all layers.

Different from these methods that focus on improving the layer-wise training and optimization, we focus on improving merging LLMs inspired from the layer-wise training dynamics.

\subsection{Bayesian Deep Learning}

\textbf{Bayesian Neural Networks.} Considering the uncertainty of the model parameters, sampling bias in the training datasets, predictive uncertainty to domain shift (also referred to as out-of-distribution examples)~\citep{Lakshminarayanan2016SimpleSAS,Blundell2015WeightWUI,Hendrycks2016AABF}, Bayesian Neural Networks (BNNs) view the model parameters as a random variable. Then, optimizing the model parameters is equivalent to optimizing the posterior distribution of the model parameters conditioned on the training datasets)~\citep{Blundell2015WeightWUI}. However, the training costs of BNNs are significantly higher than the non-Bayesian neural networks~\citep{Lakshminarayanan2016SimpleSAS}. A proper scoring creterion for training non-Bayesian NN~\citep{Lakshminarayanan2016SimpleSAS}, model ensemble~\citep{Guo2017OnOCO} and adversarial training~\citep{goodfellow2014explaining} are found to be a good way to improve the robustness of neural networks as an alternative to BNNs.

\textbf{Bayesian Model Averaging (BMA).} Except for the static importance measurement mentioned in previous section, Bayesian model averaging is another promising direction to improve the model merging performance based on the Bayesian inference. The deep model ensemble and Stochastic Weight Averaging~\citep{izmailov2018averaging, NEURIPS2019_118921ef} are actually a compelling approach to BMA~\citep{Wilson2020BayesianBDL}. The Bayesian marginalization can particularly improve the accuracy and calibration of modern deep neural networks~\citep{Wilson2020BayesianBDL}. 

However, the previous works in BNN and BMA consider the model parameters trained with the same datasets. How to merge models trained with different datasets is a new open problem which also emerges in FL~\citep{tang2024fusefl,Liu2023FedLPAFOF,Liu2021AABF,Al-Shedivat2020FederatedFLV,pmlr-v97-yurochkin19a,wangfederated} and merging LLM models in pretraining~\citep{liu2024checkpoint}.

\subsection{OOD Detection and Generalization}
The input test samples in the real-world deployment are usually diverse and the distribution of the input test samples is dynamic. Normally, these samples are not shown in the training datasets, and their distribution might be different from the training distribution, which is call out-of-distribution (OOD) data. It is important to detect the OOD data (\textbf{OOD Detection})~\citep{Liu2020EnergybasedEOD,Hendrycks2016AABF} and improve the model generalization on the OOD data (\textbf{OOD Generalization})~\citep{Ovadia2019CanCYT,NIPS2017_2650d608,Lakshminarayanan2016SimpleSAS}. When confronted with distributional shifts, models optimized purely based on average training errors lead to poor performance~\citep{duchi2018learning, arjovsky2019invariant, creager2020environment}.

\textbf{OOD Generalization.} Some methods seek to find better invariant representations in neural networks~\citep{bengio2013representation,locatello2019challenging}, which means the representations are invariant to the distribution shift. From the causal perspective, the invariant representations are the representations that are invariant to the causal factors~\citep{yang2021causalvae}. Causal learning methods aim to learn the underlying causal structure of the data and to predict the outcome variable based on the identified causal variables. By correctly identifying the cause-effect relationships, these methods are expected to perform well even when the data distribution changes, as the underlying causal structure is often assumed to remain invariant across different environments or domains~\citep{buhlmann2018invariance}. The invariant learning is to learn an invariant representation or model across environments leveraging contextual information such as domain labels~\citep{muandet2013domain, arjovsky2019invariant,albuquerque2020adversarial}, where methods can be mainly divided into invariant risk minimization~\citep{arjovsky2019invariant} and domain-irrelevant representation learning~\citep{li2018domain,gong2019dlow,sicilia2021domain}.

\textbf{OOD Detection.} 
Some methods assume access to extensive OOD data alongside in-distribution (ID) data during training, formulating OOD detection as a discriminative classification task by allocating a special label for OOD samples~\citep{fei2016breaking,larson2019evaluation,kamath2020selective,kim2018joint}. Another approach optimizes outlier exposure regularization terms on OOD samples to refine the representations and OOD scores, such as the generalized outlier exposure (OE) loss introduced by~\citep{hendrycks2018deep}, which pushes the predicted distribution of OOD samples toward uniformity~\citep{hendrycks2018deep,lee2018training}, and entropy regularization objectives employed by~\citep{zeng2021adversarial} to enforce high entropy predictions for OOD samples. Additionally, leveraging contrastive learning techniques~\citep{zeng2021modeling,zhou2021contrastive,cho2022enhancing,mou2022uninl} to increase inter-class discrepancies and enhance discriminative features for ID and OOD samples has been demonstrated to improve OOD detection performance. 

Previous works have found that the softmax outputs from models can be used as a measurement of the uncertainty of model predictions~\citep{Guo2017OnOCO,Hinton2015DistillingDTK}. And the early work in model distillation utilizes the softmax outputs as a kind of soft labels to guide the model training~\citep{Hinton2015DistillingDTK}. Some works propose to scale the logits with the temperature scaling~\citep{Liang2017EnhancingETR}, thus the ID and OOD samples are more distinguishable based on the scaled softmax scores.

Our work proposes dynamically merging task arithmetics from the Bayesian perspective to improve the OOD generalization. Inspired by the temperature scaling and the uncertainty measurement, we propose to scale the logits with the temperature scaling and to use the softmax outputs as an adjustment factor to estimating the likelihood of the task arithmetics conditioned on the input.




\subsection{Model Compression}

\textbf{Pruning.}
\textit{Unstructured pruning}~\citep{FrantarA23,sun2024a,10445737,zhang2024dynamic,dong2024pruner,tang2020survey} effectively maintains LLM performance without requiring retraining, but leads to irregular structures that necessitate specialized optimizations for inference. SparseGPT~\citep{FrantarA23} offers a novel one-shot pruning strategy by framing it as a sparse regression problem, achieving over 50\% sparsity with minimal perplexity increase. Wanda~\citep{sun2024a} reduces weight update costs by pruning low-magnitude weights scaled by input activations, while SAMSP~\citep{10445737} adjusts sparsity based on weight sensitivity using the Hessian matrix. DSnoT~\citep{zhang2024dynamic} iteratively prunes and grows weights to minimize reconstruction error in sparse models.

\textit{Structured pruning} is hardware-agnostic, facilitating accelerated inference but may degrade performance due to the removal of critical components, often necessitating fine-tuning. Loss-based Pruning~\citep{MolchanovMTFK19} measures the impact of unit removal on loss. LLM-Pruner~\citep{ma2023llmpruner} uses gradient information to identify dependent structures for optimal pruning. In contrast, Shortened LLaMA~\citep{kim2024mefomo} focuses on depth pruning of Transformer blocks based on loss derivatives, employing LoRA to quickly recover performance post-pruning. Magnitude-based Pruning~\citep{NIPS2015_ae0eb3ee} assesses pruning unit importance based on their magnitudes, pruning those below a set threshold. Regularization-based Pruning~\citep{NIPS2016_41bfd20a} incorporates regularization terms to induce sparsity. 

Different from these pruning methods which focus on the weight pruning, our method is inspired from the sparse property of the task arithmetics to reduce the expert memory occupation~\citep{he2025localizeandstitch,tang2020survey}. We sparsity the task arithmetics based on denoising and the magnitudes in our work. Note that our framework is a general framework, any other sparsity method can be combined with our framework.

\textbf{Quantization.}
Weight-only quantization is the most conventional and widespread method. For example, LUT-GEMM~\citep{park2024lutgemm} uses binary-coding quantization (BCQ)~\citep{RastegariORF16} format, which factorizes the parameters of LLMs into binary parameters and a set of scaling factors, to accelerate quantized matrix multiplications in weight-only quantization. GPTQ~\citep{frantar2023optq} proposes a layer-wise quantization method based on Optimal Brain Quantization (OBQ)~\citep{frantar2022optimal}, which updates weights with inverse Hessian information, and quantizes LLMs into 3/4-bit. QuIP~\citep{chee2023quip}  optimally adjusts weights by utilizing the LDL decomposition of the Hessian matrix derived from vectors drawn uniformly at random from a calibration set, and multiplies weight and Hessian matrices with a Kronecker product of random orthogonal matrices to ensure incoherence between weight and Hessian matrices. Combining these two steps, QuIP successfully quantizes LLMs into 2-bits with minimal performance loss.

To further minimize quantization errors in the weight-only quantization of LLMs, lots of works identify sensitive weights, which have an important effect on LLMs' quantization performance, and store these sensitive weights in high precision. For example, AWQ~\citep{abs-2306-00978} stores the top 1\% of weights that have the most significant impact on LLM performance in high-precision, and integrates a per-channel scaling method to identify optimal scaling factors. Here, "channel" denotes individual dimensions or feature maps within the model. Similar with AWQ, OWQ~\citep{LeeJKKP24} store weights sensitive to activation outliers in high-precision, and quantizes other non-sensitive weights.  Different from OWQ, SpQR~\citep{dettmers2024spqr} employs the L2 error between the original and quantized predictions as a weight sensitivity metric.  Furthermore, SqueezeLLM~\citep{abs-2306-07629} introduces a  weights clusters algorithm  based on sensitivity, using k-means centroids as quantized weight values, to identify sensitive weights.

\subsection{Data Synthesis}
\textbf{Data Labeling.}  
The data labeling process utilizes the advanced language comprehension capabilities of large language models (LLMs) to annotate extensive unlabeled datasets, proving particularly beneficial in areas like cross-lingual processing and multimodal learning~\citep{zhu2023can,gilardi2023chatgpt,alizadeh2023open}. Automating this process enhances data preparation efficiency. Recent studies have investigated the zero-shot potential of models like GPT-4 for annotating political discourse on platforms like Twitter~\citep{tornberg2023chatgpt}. Some works consider constructing a preference tree~\citep{zeng2024automatic,yuan2024advancing} from LLM responses to refine incorrect responses based on feedback from models like GPT-4, creating more diverse and robust preference data.

\textbf{Data Reformation.}  
Data reformation aims to transform existing datasets into diverse variations to improve data augmentation~\citep{dixit2022core,dunlap2023diversify}. This enriches the training set with varied examples, enhancing model robustness and generalization. Novel approaches leveraging LLMs have emerged, such as Disco by Chen et al.~\citep{chen2022disco}, which generates large-scale, high-quality counterfactual datasets. A prominent method in this area is in-context learning~\citep{dong2022survey}, where examples embedded in prompts guide LLMs to generate responses that reflect the provided patterns. Early works, such as Self-Instruct~\citep{wang2023self} and Unnatural Instructions~\citep{honovich2022unnatural}, utilized task pools with hand-crafted seed examples. In contrast, LaMini-LM~\citep{wu2023lamini} built on this foundation by leveraging extensive data from Wikipedia to generate a wider range of instructions. Auto Evol-Instruct~\citep{zeng2024automatic}, originally designed to evolve instructions, automates the optimization of evolution rules through an Optimizer LLM that iteratively refines these rules based on evolving feedback data. Furthermore, Instruction Backtranslation~\citep{li2023self} enhances instruction-following capabilities by creating instruction-response pairs from unannotated data, thus minimizing the need for manual annotation. This ongoing refinement of data reformation is essential for enhancing performance across various tasks.

\textbf{Generation from LLMs.}  
Model generation utilizes powerful models—such as ChatGPT, StableVicuna, and GPT-4—to create datasets that enhance the performance of weaker models. Techniques include generating concise narratives through templates~\citep{eldan2023tinystories} and assessing dataset quality with LLMs. Research by Phi-1 and its subsequent studies~\citep{gunasekar2023textbooks,li2023textbooks} indicates that even a small volume of high-quality data can effectively train models via generated textbooks and exercises using GPT-3.5. Additionally, performance has been improved by developing instructional datasets and fine-tuning models to enhance dataset quality~\citep{honovich2022unnatural,Taori2023alpaca,chen2023alpagasus}. Domain model generation concentrates on the use of specialized models to produce domain-specific data. For example, domain generation can provide instructional materials for specific programming tasks in coding~\citep{wei2024magicoder,luo2024wizardcoder}. In mathematics, initiatives like Minerva~\citep{lewkowycz2022solving} and DeepSeekMath~\citep{xin2024deepseekprover} focus on generating accurate solutions.

\textbf{Synthetic Multi-step Reasoning.}  
To enhance reasoning in LLMs, additional reasoning steps are incorporated into data synthesis. The MMIQC framework~\citep{liu2024augmenting} iteratively creates synthetic question-response pairs by expanding problems and integrating reasoning steps while preserving logical structure. A complementary strategy involves generating chain-of-thought (CoT) answers based on questions~\citep{li2024common}. Building on question-CoT pairs through Self-Instruct, MathInstruct~\citep{yue2023mammoth} introduces the Program-of-Thought (PoT) rationale to streamline mathematical problem-solving.

In this work, we utilize the stronger LLM to generate CoT based domain training data to enhance the reasoning performance of the downstream tasks. As far as we know, \textit{this work is the first to explore whether the model merging influences the CoT based reasoning performance.}

\section{Theoretical Understanding}\label{appx:theoretical-analysis}
In this section, we provide the theoretical interpretation from the perspective from the In-context learning (ICL) to further understand why routing \textit{finetuned models} with task-level router instead of token-level ones might be better. Note that here the \textit{different finetuned models have been trained on individual tasks and never see other tasks}. We re-write the preliminary in Section~\ref{sec:Preliminary} here for convenience of reading. 

\textbf{Task Data Distribution.}
Given a set of different downstream tasks $\Tcal $, based on the sampling task $\tau \in \Tcal$, the pretraining document (data sample) is a sequence $o_{1:T}$ of tokens with the maximum length $T$ generated from a distribution $p_{\tau} = p(x_{1:T} \vert \tau)  = p(\obs_1, \dots, \obs_T \vert \tau)$~\citep{xie2022an,wies2023learnability,hahn2023theory,li2024language}. 

\textbf{Pretraining Data Distribution.} And we define the pretraining data is sampled from $p(o|\Tcal^\star) = \int_{\tau^\star \in \Tcal^\star} p(\obs_1, \dots, \obs_T \vert \tau)p(\tau^\star)d \tau^\star$. Each token $\obs$ is sampled from a vocabulary $\obsset$. $p(\tau^\star)$ is a prior distribution about $\tau^\star$. And both ($\Tcal$ and $\Tcal^\star$ belong to a large task family $\Omega$, i.e. $\Tcal,\Tcal^\star \subset \Omega$.

\textbf{Language Modeling.}
Current LLMs~\citep{gpt3_2020,touvron2023llama2,xie2022an} usually utilize the next word prediction as the language modelling, which predicts the next token $o_t$ given the previous tokens $o_{1:t-1}$ for all $t=1,\dots, T$. Formally, a LLM parameterized by $\theta$ is a distribution $f_{\theta}(o_t \vert o_{1:t-1})$. And it is pretrained on a huge corpus sampled from the pretraining distribution $p(o_{1:T}|\Tcal^\star)$~\citep{xie2022an}. 

\textbf{Finetuning LLM.} Normally, for each downstream task $\tau \in \Tcal$, finetuning LLM is to minimize the cross-entropy loss function as below:
\begin{align}
    \begin{small}
    L_\text{CE}(\theta, \tau)= -\sum_{t=1}^T \mathbb{E} [p_{\tau}(x_t|x_{1:t-1})\cdot\log f_{\theta}(x_t|x_{1:t-1})]. \notag
    \end{small}
\end{align}
After finetuning, the model parameters $\theta$ are updated to $\theta_{\tau}$ for each task $\tau$. 


\textbf{Prompt distribution in Pretraining \& Finetuing.} Following~\citep{xie2022an}, a prompt is composed of an input token sequence $o_{1:T}$ followed by an output token $y$. Then, the $i$-th training example \footnote{Here, training example in prompts means happens during the prompt learning, instead of the pretraining or the finetuning.} that can appear in any place in the whole prompt $o_{1:T}$ is defined as $O_i$ consisting of an input $s_i=O_i \left[1:k-1 \right]$ (the first $k-1$ tokens) followed by the output $y_i = O_i \left[k\right]$ at the end, where the length $k$ is fixed for simplicity. 

The $i$-th training example is independently generated as follows: 1) Generate a start hidden state $\hiddensegstart_i$ from a \emph{prompt start distribution} $\ppromptstart$;
2) Given $\hiddensegstart_i$, generate the example sequence $\obsseg_i=[s_i,\y_i]$ from $p(\obsseg_i \vert \hiddensegstart_i, \tau^\perp)$.
The test input $\Xtest = s_{n+1}$ is sampled similarly.
Between each example, a special delimiter token $\obsdelim$ ``reset'' the transition between examples~\citep{xie2022an}. Then, the prompt consists of a sequence of training examples ($\promptseq$) followed by the example $\Xtest$:

\begin{align}
    [\promptseq, \Xtest] = [s_1, \y_1, \obsdelim, s_2, \y_2, \obsdelim, \dots, s_n, \y_n, \obsdelim, \Xtest] \sim \pprompt.
\end{align}
Different from ~\citep{xie2022an}, here we distinguish the pretraining tasks (concepts) $\Tcal^\star$ and the finetuning tasks (concepts) $\Tcal = \left\{\tau_1, \tau_2, ..., \tau_{n_\tau} \right\}$, from which the prompts might be sampled. We mainly consider $\tau^\perp\in \Tcal$.


\textbf{In-context learning setups and Assumptions.} We follow other settings and assumptions in ~\citep{xie2022an}. With the greedy decoding~\citep{fubreak}, sampling the next token from the language modeling $f_{\theta}(o_t \vert o_{1:t-1})$ becomes the predictor as $y =\argmax_{o_t} f_{\theta}(o_t|o_{1:t-1})$. For simplicity, following~\citep{xie2022an}, we consider that the finetuned LLMs have been aligned with its pretraining and finetuning data distribution, i.e. $p_{\Tcal^\star \cup \tau} = p(o_{1:T}|\Tcal^\star \cup \tau)$ for any task $\tau\in\Tcal$. For convenience, we write $ p_{A\tau}= p_{\Tcal^\star \cup \tau}$ which means that the $\Tcal^\star$ is augmented with $\tau$.

Thus, for $[\promptseq, \Xtest]$, the in-context learning predictor can be written as $f_{\theta_\tau}^{n}(\Xtest) := \argmax_y p_{A\tau}(y|\promptseq, \Xtest)$, which outputs the most likely prediction over the \emph{pretraining distribution} conditioned on the \emph{prompt distribution}. Its expected 0-1 error with $n$ examples is $\Lzeroone(f_{\theta_\tau}^n) = \E_{\Xtest,\ytest \sim \pprompt}[\indicator[f_{\theta_\tau}^{n}(\Xtest) \neq \ytest]]$.

We define $p_\tau^i(o):=p(O[i]=o|O[1:i-1],\tau)$ of the $i$-th token with previous tokens and the analogous distribution $p^{i}_{prompt}:=p_{prompt}(O[i]=o|O[1:i-1])$ under the prompt distribution. Following~\citep{xie2022an}, there is a distinguishability condition formalizes when in-context learning occurs giving the downstream task $\tau$. 

The distinguishability condition is dependent on a KL divergence between the previous two distributions and the error terms $\epsilon_\tau$ resulting from the distribution mismatch between the prompt and the pertaining distributions for each example. Letting $p_{\tau}^i(o)$ and $p^{i}_{prompt}$ correspond to the task $\tau$ and
and $\tau^\perp$.



\begin{condition}[distinguishability~\citep{xie2022an}] The $\tau^\perp$ is distinguishable if for all $\tau\in\Omega$, $\tau \neq\tau^\perp$,
\begin{align}\label{eq:distinguish1}
    \sum_{i=1}^k \text{KL}_{i}(\tau^\perp||\tau)>\epsilon_\tau,
\end{align}
where the $\text{KL}_{i}(\tau^\perp||\tau) :=\mathbb{E}_{O[1:i-1]\sim p_{prompt}}[\text{KL}(p^{i}_{prompt}||p_\tau^{i})].$
\label{cond:distinguish}
\end{condition}

\begin{lemma}\label{lemma:1}~\citep{xie2022an}
let $\mathcal{B}$ denotes the set of $\tau$ which does not satisfy Condition~\ref{cond:distinguish}. We assume that $\text{KL}(p_{prompt}(y_\text{test}|x_\text{test}))||p(y_\text{test}|x_\text{test},\tau)$ is bounded for all $\tau$ and that $\tau^\perp$ minimizes the multi-class logistic risk as,
\begin{align}\label{eq:lemma:1:LCE}
\begin{split}
L_\text{CE}(\tau)=-\mathbb{E}_{x_\text{test}\sim p_{prompt}}[p_{prompt}(y_\text{test}|x_\text{test})\cdot\log p(y_\text{test}|x_\text{test},\tau)].
\end{split}
\end{align}
If
\begin{align}\label{eq:prompt_tau_upperbound_concrete_with_epsilon}
\mathbb{E}_{x_\text{test}\sim p_{prompt}}[\text{KL}(p_{prompt}(y_\text{test}|x_\text{test})
|| p(y_\text{test}|x_\text{test},\tau))]\leq \epsilon_{\tau},\quad \forall \quad \tau\in\mathcal{B},
\end{align}
then
\begin{align}
\lim_{n\rightarrow\infty} L_{0-1}(f_{\theta_\tau}^n) \leq \inf_{f} L_{0-1}(f) + g^{-1}\bigg(\sup_{\tau\in\mathcal{B}}(\epsilon_\tau)\bigg),
\end{align}
where $g(\nu) = \frac{1}{2}\big((1-\nu)\log(1-\nu)+(1+\nu)\log(1+\nu)\big)$ is the calibration function~\citep{Steinwart2007HowTC,pires2016multiclass} for the multiclass logistic loss for $\nu\in[0,1]$.
\end{lemma}

Following~\citep{Kleijn2012TheBT,xie2022an}, the task parameter $\tau$ is assumed to have the continuity, where the KL divergence is assumed to haver the 2nd-order Taylor expansion. Then, we have the following theorem and proof.
\begin{theorem}
\label{thm:continuity}~\citep{xie2022an}
Let the set of $\tau$ which does not satisfy Equation~\ref{eq:distinguish1} in Condition~\ref{cond:distinguish} to be $\mathcal{B}$.
Assume that KL divergences have a 2nd-order Taylor expansion around $\tau^\perp$:
\begin{align}
    \forall j>1,~~\text{KL}_{i}(\tau^\perp||\tau) = \frac{1}{2}(\tau - \tau^\perp)^\top \fisherinfj (\tau - \tau^\perp) + O(\|\tau - \tau^\perp\|^3)
\end{align}
where $\fisherinfj$ is the Fisher information matrix of the $j$-th token distribution with respect to $\tau^\perp$.
Let $\conditionnum = \frac{\max_{j}\lambdamax(\fisherinfj)}{\min{j}\lambdamin(\fisherinfj)}$ where $\lambdamax,\lambdamin$ return the largest and smallest eigenvalues.
Then for $k \geq 2$ and as $n\rightarrow \infty$, the 0-1 risk of the in-context learning predictor $f_{\theta_\tau}^n$ is bounded as
\begin{align}
    \lim_{n\rightarrow \infty} \Lzeroone(f_{\theta_\tau}^n) \leq \inf_{f} \Lzeroone(f) + g^\minv\left(O\left(\frac{\conditionnum\sup_{\tau \in \badset}(\errstart + \errdelim)}{k-1}\right)\right)
\end{align}
\end{theorem}

\begin{proof}~\citep{xie2022an}
By the continuity assumption, we have for any $\tau$ in $\mathcal{B}$ that
\begin{align}
    \sum_{j=2}^k \text{KL}_{i}(\tau^\perp||\tau)
    &\geq \frac{1}{2}\sum_{j=2}^k (\tau - \tau^\perp)^\top \fisherinfj (\tau - \tau^\perp) + (k-1)O(\|\tau - \tau^\perp\|^3)\\
                      &\geq \frac{1}{2}(k-1)\lambdamin(\fisherinfj) \|\tau - \tau^\perp\|^2\\
    \implies \|\tau - \tau^\perp\|^2 &\leq \frac{\errstart + \errdelim}{\frac{1}{2}(k-1)(\min_j~\lambdamin(\fisherinfj))}.
\end{align}
Using the above term to bound the last KL term ($k$-th token), we have:
\begin{align}
    \text{KL}_{k}(\tau^\perp||\tau) &= \frac{1}{2}(\tau - \tau^\perp)^\top  \fisherinfk(\tau - \tau^\perp) + O(\|\tau - \tau^\perp\|^3)\\
         &\leq \frac{1}{2}(\max_j~\lambdamax(\fisherinfj))\|\tau - \tau^\perp\|^2 + O(\|\tau - \tau^\perp\|^2)\\
                 &\leq \frac{(\errstart + \errdelim)(\max_j~\lambdamax(\fisherinfj) + O(1))}{(k-1)\min_j~\lambdamin(\fisherinfj)}.
\end{align}
Rearranging above equation, and with the defintion that $\text{KL}_{k}(\tau^\perp||\tau) = \E_{\Xtest \sim \pprompt} [KL(\pprompt(\ytest \vert \Xtest) \| p(\ytest \vert \Xtest, \tau)) ]$, we have
\begin{align}\label{eq:prompt_tau_upperbound_concrete}
\E_{\Xtest \sim \pprompt} [KL(\pprompt(\ytest \vert \Xtest) \| p(\ytest \vert \Xtest, \tau)) ] \leq \frac{(\errstart + \errdelim)(\max_j~\lambdamax(\fisherinfj) + O(1))}{(k-1)\min_j~\lambdamin(\fisherinfj)}
\end{align}
Combining Equation~\ref{eq:prompt_tau_upperbound_concrete} with Equation~\ref{eq:prompt_tau_upperbound_concrete_with_epsilon} into Lemma~\ref{lemma:1} completes the proof.
\end{proof}

\textbf{Task-level Routing.} Observing the Equation~\ref{eq:lemma:1:LCE} in Lemma~\ref{lemma:1}, the $L_\text{CE}(\tau^\perp)$ is the optimal risk over $\tau\in\Omega$. The $\tau \in \mathcal{B}$ which does not satisfy Condition~\ref{cond:distinguish} means that the $\tau \in \mathcal{B}$ should be close to $\tau^\perp$ enough. Thus, we can have $L_{0-1}(f_{\theta_\tau}^n)$ converges with $n\rightarrow\infty$ as in Lemma~\ref{lemma:1}. The task-level routing means to route $\tau^\perp$ to the finetuned LLM that has been trained on $p(o_{1:T}|\tau^\perp)$. Thus, the task-level routing can satisfy the requirement of $\tau \in \mathcal{B}$.

\textbf{Token-level Routing.} The core motivation of using token-level routing is that different tokens prefer different routers. Here, inspired by the distinguishability condition~\ref{cond:distinguish}, we can interpret the token-level router which dynamically finds the expert model $i^\star$ for $i$-th token that satistifies:
\begin{equation}\label{eq:optim_token_route}
    \sum_{i}^k \argmin_{i^\star} \text{KL}_{i}(\tau^\perp||\tau_{i^\star}).
\end{equation}
However, there is distribution shift between the $\tau^\perp$ and different $\tau_{i^\star}$. Revisiting the prompt sequence sampled as $[\promptseq, \Xtest] = [s_1, \y_1, \obsdelim, s_2, \y_2, \obsdelim, \dots, s_n, \y_n, \obsdelim, \Xtest] \sim \pprompt$, each pair $\obsseg_i=[s_i, \y_i]$ is sampled from $p(\obsseg_i \vert \hiddensegstart_i, \tau^\perp)$. If the $\tau_{i^\star}$ is choosed as different from $\tau^\perp$, the distribution shift implies that the $\text{KL}_{i}$ cannot be minimized. 

\textbf{Out-of-distribution Cases.} While the above intuition illustrates that the task-level routing might be more suitable for the in-distribution test data $x_\text{test}\sim \pprompt$, we illustrate that two cases of new prompt sampling might need need combination of different LLM experts.
\begin{itemize}[leftmargin=*]
    \item \textbf{OOD task.} Considering that the $\tau^\perp$ is different from all $\tau\in \Tcal$, there might be needs to process different tokens with different experts following equation~\ref{eq:optim_token_route}.
    \item \textbf{Compositional task.} Considering that $\obsseg_i=[s_i, \y_i]$ might be sampled from $p(\obsseg_i \vert \hiddensegstart_i, \tau_i)$, and each $\tau_i$ is different from others, the Equation~\ref{eq:optim_token_route} may helps to find the suitable experts.
\end{itemize}
However, the theoretical analysis of how Equation~\ref{eq:optim_token_route} benefits ICL is difficult and we left it as the future work, which might also be beneficial to analyse the MoE models~\citep{Dai:2024aa}. Currently, we utilize the uncertainty-based model task-level routing and merging to address the OOD problem.






\section{Detailed Experiment Settigns}\label{appx:detailed-hyper-parameters}

\subsection{Detailed Experimental Setup}\label{appx:ExperimentalSetup}

\textbf{Hardware.}  All experiments were conducted on an A800 GPU with 80GB VRAM, Intel Xeon 6348 CPU, and 100GB RAM.

\textbf{Models and Datasets.} We conduct comprehensive experiments on two cutting-edge large language model families: Qwen and LLaMA. Table~\ref{tab:backbone_models} shows the number of parameters, memory occupation and release data of these models. These models represent the latest advancements in language model development. \textit{To the best of our knowledge, this is the first model merging study focusing primarily on generative tasks, finetuning with CoT based data and cutting-edge LLM tasks.} 

\begin{table}[h]
    \centering
    \caption{Backbone Models Overview}
    \begin{tabular}{@{}lllc@{}}
        \toprule
        Model            & Number of Parameters & Release Date      &  Memory Occupation (GB) \\ \midrule
        Qwen-1.5-4B     & 4 Billion           & February 2024     & 15.26            \\
        Qwen-2.5-7B     & 7 Billion           & September 2024    & 26.00            \\
        LLaMA-3.1-3B    & 3 Billion           & April 2024        & 11.31            \\
        LLaMA-3.2-8B    & 8 Billion           & September 2024    & 30.52            \\ 
        \bottomrule
    \end{tabular}
    \label{tab:backbone_models}
\end{table}


\textbf{Generative and Reasoning Tasks in Evaluation.}
In designing our evaluation tasks, we strategically selected orthogonal benchmarks to effectively demonstrate our method's capability in resolving parameter conflicts during model merging. Our task selection follows these principles: 
\begin{itemize}[leftmargin=*]
\item \noindent \textit{(1) The mathematical reasoning and code generation tasks represent fundamentally different parameter spaces}. Specifically, mathematical computation requires numerical reasoning parameters, while code generation relies on syntax and programming logic parameters, allowing us to evaluate how well our merging approach handles potentially conflicting parameter updates. 
\item \noindent (2) \textit{Knowledge-based QA (TriviaQA) and concept understanding tasks (MMLU) evaluate distinct knowledge representations}. TriviaQA focusing on factual retrieval parameters and MMLU covering broader conceptual understanding parameters across domains. This helps assess our method's ability to preserve different types of knowledge without interference.
\item \noindent (3) The logical reasoning task (WinoGrande) may prefer to yet another independent parameter space focused on \textit{abstract reasoning}, providing insights into how well our merging technique maintains reasoning capabilities while optimizing for other tasks.
\end{itemize}

Based on above principle, we utilize the following cutting-edge LLm evaluation tasks about math reasoning, code generation, common sense QA, common sense logical reasoning, multi-domain knowledge.
\begin{itemize}[leftmargin=*]
\item \textbf{Mathematical Reasoning}: We evaluate mathematical question-answering capabilities using the GSM8K dataset~\citep{cobbe2021gsm8k}, which contains 8,500 high-quality elementary school math word problems (about 7,500 training, about 1,000 test) designed to evaluate mathematical reasoning capabilities. The problems feature diverse language styles and formats while avoiding templated designs. They use basic arithmetic operations with natural language solutions.

\item \textbf{Knowledge-based QA}: We utilize TriviaQA~\citep{joshi2017triviaqa}, a large-scale Wikipedia-based question answering dataset, where models are required to generate direct answers without multiple-choice options. It contains complex questions requiring cross-sentence inference, with significant syntactic and lexical variations between questions and answer sentences. The dataset provides challenging evaluation scenarios that better approximate human-like question answering.

\item \textbf{Code Generation}: The HumanEval~\citep{chen2021humaneval} consists of human-written programming tasks where models must complete missing Python code snippets based on provided inputs. 
The problems simulate real-world programming challenges requiring context understanding, reasoning, and multi-step operations across varying difficulty levels and abstraction layers. 

\item \textbf{Logical Reasoning}: WinoGrande~\citep{sakaguchi2019winogrande} is a large-scale commonsense reasoning dataset of approximately 2800 questions developed by University of Washington researchers. Questions are presented as fill-in-the-blank tasks with two options and correct answers, with dataset bias reduced through the AfLite algorithm. The benchmark evaluates models' commonsense reasoning abilities in understanding and generating relevant text.

\item \textbf{Multi-domain Knowledge}: We employ MMLU~\citep{hendrycks2021mmlu} to assess knowledge retention across diverse 57 subjects ranging from basic mathematics to US history, computer science, law, and ethics. Using multiple-choice questions of varying difficulty levels. Notably, we exploit the generation-based approach for multiple-choice evaluation, analyzing knowledge preservation across base models, fine-tuned variants, and merged models. The generation-based evaluation is better to measure the generative abilities of LLMs than choice-based evaluation.
\end{itemize}

In the experiments of evaluating the scalability of Mediator, we also finetune another 4 LLMs according to the following 4 extra evaluation tasks.
\paragraph{IFEval.}~\citep{zhou2023ifeval} A comprehensive benchmark dataset designed to evaluate instruction-following capabilities of language models. It contains carefully curated instruction-response pairs across diverse task categories including text generation, analysis, and reasoning. The dataset aims to assess models' ability to accurately interpret and execute natural language instructions while maintaining coherence and relevance in responses. The evaluation spans multiple dimensions including instruction comprehension, output quality, and adherence to specified constraints.

\paragraph{CEval.}~\citep{huang2023ceval} A comprehensive Chinese evaluation suite designed to assess language models' knowledge and capabilities across various academic and professional domains. It consists of multiple-choice questions drawn from professional qualification exams and academic tests in China. For our evaluation, we specifically focus on three key subjects:
(1) \textit{Medicine:} testing clinical knowledge, diagnosis, and treatment principles from medical licensing exams;
(2) \textit{College Economics:} evaluating understanding of micro/macroeconomics concepts, market principles, and economic theories;
(3) \textit{Law:} assessing comprehension of Chinese legal principles, regulations, and judicial procedures.
These subjects were chosen to evaluate models' domain-specific expertise in technically demanding professional fields.

\textbf{Finetuning Settings.} We adopt the ms-swift~\citep{zhao2024swiftascalablelightweightinfrastructure} to finetune the given pretrained LLM. The finetuning datasets are constructed by  augmenting some publicly  datasets (task related but without overlap) with GPT-4o~\citep{gilardi2023chatgpt} and Chain-of-Thoughts~\citep{CoT}. For each finetuning process, we use at least 180K training samples to ensure sufficient performance improvement on the corresponding task, which helps validate the effectiveness of our experiments. We provide the details of how we construct the finetuning datasets in Section~\ref{appx:finetuning-data-generation}. 

\textbf{Baselines.} Following the summary of the related works in Section~\ref{appx:more-related-works}, we compare methods in following four categories:


\begin{itemize}
    \item \textbf{Pretrained model.} The pretrained models are directly downloaded from its open-source repository. These models are pretrained on the large corpus and have included enormous knowledge about the evaluation tasks.
    \item \textbf{Finetuned Models}: We finetune the pretrained models on datasets that we construct for each domain. Then, each finetuned model is evaluated on all tasks. The results help to show wheter finetuning on task A enhance or decrease model performance on task B. 
    \item \textbf{Static merging methods.} These methods use fixed weights to merge multiple finetuned models. The advanced static merging methods like Fisher merging ~\citep{FisherMerging_NeurIPS2022} and RegMean~\citep{Jin:2022aa} require extra dataset and forward process to estimate some information like gradients, hessian, features to estimate parameter importance, which causes significant computational costs. Furthermore, considering that LLMs need to be deployed on various tasks, the utilized dataset actually cannot reflect the real-world data distribution. Therefore, these methods are shown empirically to perform worse than some calibration-less methods~\citep{Du:2024aa}.
    Recently, TIES~\citep{TiesMerging_NeurIPS2023} and PCB-merging~\citep{Du:2024aa} achieve the best performance in weighted average method and do not require calibration data. Thus, we choose it for comparison. 
    \item \textbf{Dynamic Advanced Methods}: We compare with state-of-the-art dynamic merging techniques that adapt model fusion parameters based on the input data. For example, Branch-train-mix dynamically routes different tokens to corresponding experts for generation through token-level routing. Similarly, the twin-merge~\citep{Lu:2024aa} computes merging weights through task-level routing mechanisms and dynamically fuses SVD-decomposed task vectors into the pretrained model in real-time.
\end{itemize}

\subsection{Hyperparameters of Finetuning and Implementing Baselines}\label{appx:single-task-finetuning}

\textbf{Hyperparameters for Single-task Finetuning.} For single-task finetuning, we utilize a set of hyperparameters that remain consistent across all models and tasks. The learning rate is set at 1.2e-5, applying a cosine decay schedule. The batch size varies, with one sequence per batch for both the 7B and 8B models, while the 3B and 4B models use two sequences per batch considering the GPU memory limitation. The maximum sequence length is confined to 4096 tokens for both math and QA tasks and extends to 7000 tokens for coding tasks. The training consists of two epochs, and we employ the AdamW optimizer with parameters $\beta_1=0.9$, $\beta_2=0.999$, and $\epsilon=1e-8$. Additionally, warmup steps constitute 5\% of the total steps. 

For all model merging baselines, the finetuned LLMs are the same. And all finetuned LLMs have shown that they can successfully improve the performance of the pretrained model on various tasks. The following is the details of how we tune and implement baseline methods.

\textbf{Hyperparameters for PCB-merging.} We follow the original paper of PCB merging and have searched its hyperparameters. The weight clipping ratio is established at 0.1, which means weights with magnitudes in the bottom 10\% are clipped to zero, following recommendations from the original paper concerning LLM generalization tasks. For model merging exploration, we perform 200 random exploration steps. The initial weights for random exploration are set to (0.4, 0.4, 0.4, 0.4) for the 3B, 4B, 7B, and 8B models with four experts, while for all models with eight experts, they are set to a repeated value of 0.2 across eight instances. The validation batch size is configured to handle 8 samples per task, and we implement early stopping with a patience of 10 steps without improvement. The weight clipping ratio and exploration parameters are uniform across all model sizes and tasks to facilitate fair comparison. It's worth noting that for the 7B and 8B models, the validation batch size is reduced to 4 due to memory limitations.

\textbf{Optimizing PCB-merging.} To enhance the computational speed of PCB-merging, several optimizations were introduced based on the original framework, which do not influence its task performance. Instead of merging entire models simultaneously, we adopt a layer-wise model merging strategy. This layer-by-layer merging approach has multiple benefits: it decreases memory overhead during the merging process, facilitates parallel processing of different layers, and allows for the assignment of layer-specific merging weights. Moreover, we implemented asynchronous model input/output operations which enable overlapping of I/O with computational processes. This adjustment is instrumental in reducing the total merging time by as much as 40\%, enabling the seamless streaming of large models. These optimizations have significantly boosted both the efficiency and effectiveness of PCB-merging, particularly the layer-wise method, which has lowered peak memory usage by approximately 60\% while maintaining or enhancing final model performance. And other hyper-parameters and settings are completely followed as the original paper.

\textbf{Hyperparameters for Twin-merging finetuning.} For Twin-merging, we leverage LoRA finetuning in lieu of SVD to attain greater precision following the original paper. The rank is set as 32, and both the alpha and dropout parameters are also set at 32 and 0.1, respectively. The target modules involved in this finetuning process include the query and value matrices within the attention layers. And we also have conducted grid search for the hyper-parameters. Each task involves training over two epochs, with a batch size set at 16; this batch size is reduced to 8 for the 7B and 8B models. The learning rate is specified at 1.5e-4, utilizing a cosine decay schedule, and the optimizer employed is AdamW. 


\textbf{Hyperparameters for Branch-train Mix (BTX).} For the training of the BTX router, we follow the original implementation of it within ms-swift and its original paper to implement it. The relevant hyperparameters for this setup include a training duration of 2 epochs, with a batch size of 2; this is adjusted to 1 for the 7B and 8B models. The learning rate is established at 1.5e-6, utilizing a linear decay schedule, alongside the AdamW optimizer, which is configured with a weight decay of 0.001. The router's architecture consists of an input dimension derived from 2 layers of an FFN, with a hidden dimension of 256, an output dimension corresponding to the number of experts, and a dropout rate of 0.1. Warmup steps account for 5\% of the total steps, and evaluations are conducted at every 1000 steps. To ensure balanced representation, the router is trained on a dataset that equally samples from all tasks, employing early stopping with a patience of 2 epochs based on validation accuracy.



\textbf{Hyperparameters for Mediator.} For Mediator training, we utilize the same single-task finetuned experts as delineated in Appendix~\ref{appx:single-task-finetuning}. The task-level router is constructed from the first 9 layers of the pretrained LLM (with gradients stopped) and includes 2 additional FFN layers. The router is trained on a balanced dataset with equal samples from each task domain to ensure unbiased task routing. We evaluate the router performance every 1000 steps and use early stopping with patience of 2 epochs based on validation accuracy. The training process for the router involves sampling 2000 examples from each task domain, specifically in mathematics, coding, question answering, law, economics, instruction following, and medicine. The specific hyperparameters applicable to router training encompass a duration of 2 epochs, and a batch size of 256, which is decreased to 128 for the 7B and 8B models. The learning rate is set to 3e-4, accompanied by a cosine decay schedule, and the optimizer remains as AdamW. The warmup ratio is defined as 10\% of the total steps. The router's architecture features a frozen backbone comprising the initial 9 layers from the pretrained LLM, along with 2 trainable FFN layers. These layers have a hidden dimension of 1280, with the output dimension reflecting the number of experts, and a dropout rate fixed at 0.05.

For router based selection, we use the temperature parameter $\beta$ in Equation~\ref{eq:temperature-scale} as 1.5 to convert the prediction rates into concrete merging parameters for each expert, which achieves the best experimental results. This temperature scaling helps balance between being decisive in expert selection while maintaining some degree of smoothness in the merging weights. A temperature of 1.5 empirically provides the optimal trade-off, where lower temperatures lead to more concentrated weights but potentially miss useful signals from secondary experts, while higher temperatures result in overly diffuse weights that don't sufficiently leverage expert specialization.



\section{The variations and details of different parts of Mediator}\label{appx:Mediator-details}
In this section, we provide the detailed variations and implementation details of different parts of Mediator. Some definitions and operations that appear in the main text may be re-defined in this section for better clarity of reading.

\subsection{Measuring Parameter conflicts}\label{appx:measuring-parameter-conflicts}

\paragraph{Task Arithmetics.}
We define the task arithmetics as the parameter difference between the finetuned LLM $\theta_{\tau}$ based on task $\tau$ and the pre-trained LLM $\theta$, i.e., $\Delta_{\tau} = \theta_{\tau} - \theta$. Such a task arithmetics can represent the update on the finetuned LLM $\theta_{\tau}$ based on task $\tau$. Given a pretrained LLM $\theta$, one can recover the finetuned LLM $\theta_\tau = \theta + \Delta_\tau$.

\paragraph{Denoising Parameters.}
Because the finetuing directions on different tasks are various and stochastic, there exist some elements in $\Delta_\tau$ that do not influence the performance on task $\tau$. Before measuring the parameter conflicts~\citep{Yadav:2023aa,he2024localize}, we firstly denoise the parameters by removing the elements in $\Delta_\tau$ that do not influence the performance on task $\tau$. We also model the update directions of different elements as the Gaussian distribution $\Ncal_{\text{UPD}}(\mu_{\text{UPD}}, \sigma_{\text{UPD}}^2)$, where $\mu_{\text{UPD}}$ is the mean of the update direction and $\sigma_{\text{UPD}}^2$ is the variance. 

Based on the estimated $\mu_{\text{UPD}}$ and $\sigma_{\text{UPD}}$, we can regard the elements within range ($\mu_{\text{UPD}} - \sigma_{\text{UPD}}$, $\mu_{\text{UPD}} + \sigma_{\text{UPD}}$)
as the elements that do not influence the performance on task $\tau$. Thus, we can denoise the parameters by removing the elements within range ($\mu_{\text{UPD}} - \sigma_{\text{UPD}}$, $\mu_{\text{UPD}} + \sigma_{\text{UPD}}$) (set as 0) and obtain the new parameter arithmetic $ \hat{\theta}_{\tau} = \theta + \hat{\Delta}_{\tau}$. In the deployment, these elements are saved with their indexes and values for realistic sparsification thus saving memory.

\subsection{Adaptive Merging}\label{appx:adaptive-merging}
Inspired by the empirical observation in Figure~\ref{fig:parameter-conflict} in Section~\ref{sec:Understanding-Conflict}, we propose to leverage the parameter conflict distribution across different finetuned LLMs to adaptively merge the finetuned models. 

Practically, before merging, Mediator automatically calculates the conflicts $d_l$ across different finetuned LLMs. Then, Mediator models the conflicts as a Gaussian distribution $\Ncal(\mu, \sigma)$. Then, for each layer index $l$, Mediator average layer parameters if the conflict $d_l$ is less than the $\mu + \sigma$, otherwise, Mediator Routing this layer. We denote the averaged layer parameters as $\phi_{\text{AVG}}^l$ and the Routing layer parameters as $\phi_{\text{UP}}^l$. Algorithm~\ref{algo:adaptive-merging} shows this detailed process.

\subsection{Averaging Operations}\label{appx:averaging-operations}

\paragraph{Naive Average Operation.}
The naive average operation $\Mcal_{\text{AVG}}$ is defined as:
\begin{equation}
    \Mcal_{\text{AVG}}(\theta_1, \theta_2, \dots, \theta_{|\Tcal|}) = \frac{1}{|\Tcal|} \sum_{\tau=1}^{|\Tcal|} \theta_{\tau},
\end{equation}
which regards all finetuned LLMs equally and utilizes the same weight for each finetuned LLM. Such a simple average operation is easy to implement, without fabricated procedures, thus having low computational overhead. However, different parameters may have different sensitivities to the final merged model, which may lead to suboptimal performance. 

\paragraph{Taylor Expansion.}
The Taylor expansion is a powerful tool for approximating a function around a specific point, and it is widely used in various fields, including model compression~\citep{lee2018snip} and previous works on model merging~\citep{jhunjhunwala2023towards,qu2022generalized}. We can utilize the Taylor expansion to measure the sensitivity of each parameter that influences the model performance on the downstream task $\tau$ as follows:
\begin{align}
    L_{\text{CE}} (\theta + \delta_\theta, \tau) = L_{\tau}(\theta) + \frac{\partial L_{\tau}(\theta)}{\partial \theta} \theta \delta_\theta + O(\delta_\theta^2).
\end{align}

The first-order derivative $\frac{\partial L_{\text{CE}} (\theta, \tau)}{\partial \theta}$ measures the sensitivity of the loss function $L_{\tau}$ to the parameter $\theta$. Thus, we can see that utilizing the same averaging operation for all parameters may not be the optimal choice for merging different finetuned LLMs, as it does not take into account the different contributions of each finetuned LLM to the final merged model.

\paragraph{Parameter-level Importance based Model Merging.}
To this end, one can utilize the first-order derivative or higher-order derivative to measure the sensitivity of the loss function $L_{\tau}$ to the parameter $\theta$, based on which, the parameter-level importance can be measured as $w_\tau = \frac{\partial L_{\tau}(\theta)}{\partial \theta}\theta$. Then, the parameter-level importance can be used as the averaging weight for each finetuned LLM like the following:
\begin{equation}\label{eq:FO-Taylor}
    \Mcal_{\text{FO-Taylor}}(\theta_1, \theta_2, \dots, \theta_{|\Tcal|}) = \sum_{\tau=1}^{|\Tcal|} w_\tau \theta_\tau.
\end{equation}

\paragraph{Preprocessing Parameters.}
Considering that the finetuing directions on different tasks are various and stochastic, some elements in $\theta_\tau$ that are optimized stochastically and may not influence the performance on task $\tau$. Thus, before averaging, we can denoise the parameters by removing the elements in $\theta_\tau$ that do not influence the performance on task $\tau$. Like the preprocessing the task arithmetics and the denoising, we recover the finetuned LLM $\hat{\theta}_\tau = \theta + \hat{\Delta}_\tau$ by removing the elements in $\Delta_\tau$ that do not influence the performance on task $\tau$. Then, the averaged models can be obtained by the following equation:
\begin{equation}
    \Mcal_{\text{de-noise}}(\theta_1, \theta_2, \dots, \theta_{|\Tcal|}) = \sum_{\tau=1}^{|\Tcal|} w_\tau \hat{\theta}_\tau.
\end{equation}


\subsection{Details of Expert Routing}\label{appx:expert-routing}
For an input $x_{1:t}$ sampled from the training dataset $p_\tau$, the intuitive routing mechanism is to directly use the finetuned LLM $\theta_\tau$ that is trained on the training dataset $p_\tau$ to generate the output $x_{t+1:T}$. However, the real-world deployment is usually different from the training distribution, which may lead to suboptimal performance. Especially for an LLM deployment scenario, the input distribution is various.


\textbf{Modeling the likelihood $\pi_\kappa (\tau|x)$.} We build a task-level deep neural network as the router. In designing the router structure, we carefully balance model accuracy with additional memory requirements. While LLMs inherently demonstrate excellent classification capabilities, we need an efficient solution that wouldn't significantly impact performance. After extensive experimentation, we opt to utilize the embeddings from the first 9 layers of the base LLM combined with 2 FFN layers as our router architecture. This design choice eliminates the need for a separate complex router structure while maintaining high classification accuracy with minimal memory overhead and fast execution speed. The router leverages the rich semantic understanding already present in the base model's lower layers, making it both resource-efficient and effective for expert selection.

\textbf{Constructing training datasets for learning $\pi_\kappa (\tau|x)$.} For training the router, we randomly sample 2000 examples from each domain rather than using the entire finetuned dataset for efficiency. For each task, we only extract the question part (other than the question and answer pairs) to better simulate real-world deployment scenarios. We do not explicitly construct a training dataset for the ``others'' category to consider the OOD category and the sample will be regarded as prefering the original pretrained model. Instead, during inference, if the predicted probabilities for math, coding, and QA tasks are all below 0.5, the input is classified as ``others'' and processed this question by the base model. 
The training dataset can be represented as $\left\{(x, \tau)| x \sim p_{\tau}\right\}_{\tau \in \{\text{math}, \text{coding}, \text{QA}\}}$.



\section{System Optimization}\label{appx:system-optimization}
The inference latency and memory consumption are critically important for the real-world LLM applications. Thus, we consider to optimize the inference latency and memory consumption of Mediator. The overall latency of Mediator is mainly affected by the routing, loading experts between CPUs and GPUs (if required offloading), inference of the models itself.

\paragraph{Routing Latency.}
We run expert routing only once per sample because we use task-level routing. The classifier \( \kappa \) consists of two FFN layers, and its input is the hidden state of the first through ninth layers of the LLM. The total execution time of the classifier \( \kappa \) is between 0.2s and 0.4s.

\paragraph{Loading Experts.}
After obtaining \( \pi_{\kappa}(\tau|x) \), we compute \( h(\tau|x) \) according to Equation~\ref{eq:temperature-scale}. We also load expert parameters only once. To optimize this process, we explore two methods. For sparse expert parameters, we store all of them in the CPU and prefetch the parameters for the next layer while performing computations in the current layer. For non-sparse expert parameters, we store them on disk and use ZipNN~\citep{hershcovitch2024zipnn} to accelerate loading from disk to CPU.

\paragraph{Inference Timeline.} We present the optimized inference timeline of Mediator, as shown in Figure~\ref{fig:inference_timeline}. The additional time incurred by Mediator is fixed and relatively small(approximately 0.2s to 0.4s). This portion of the time overhead will decrease as the model size increases or the decoding length becomes longer.

\begin{figure}[ht]
    \centering
    \includegraphics[width=0.8\linewidth]{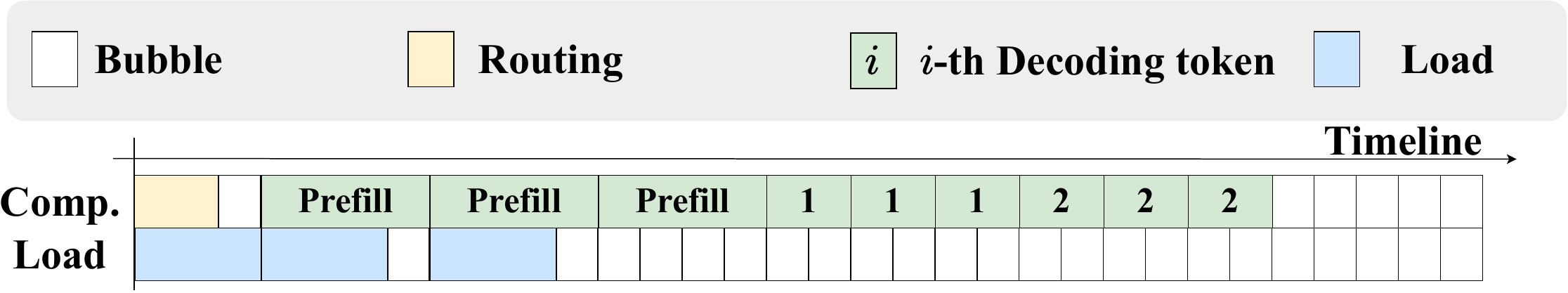}
    \caption{The inference timeline of Mediator, assuming that the number of layers is three.}
    \label{fig:inference_timeline}
\end{figure}

\subsection{Cuda Kernal Merging}\label{appx:Cuda-kernal}

We accelerate the integration or disintegration of sparse experts into the dense backbone by using CUDA's atomicAdd, which enables parallel merging of multiple experts while maintaining accuracy. Through this approach, we can split the weights into individual elements, allowing each element to be processed in parallel. However, we have observed that parallel merging alone is sufficient to mask the associated costs.


\subsection{Serving with Batched requests}\label{appx:batching-serving}
The traditional LLM serving usually accepts different requests asynchronously. Then, different requests are allocated to different batches with a predifined batch size and feed into the model. A batch-style inference usually is faster than the single-request inference, because the computation matrix is more dense and become GPU friendly.

However, the Mediator and many routing based merging works~\citep{sukhbaatar2024branch,Lu:2024aa} require to select different experts for different requests. Thus, a batch of various requests may lead to various experts being selected, which would disturb the regularity of the computation matrix. To implement the batch-style serving, we implement following two new system optimization schemes to improve Mediator.

\begin{itemize}
    \item Clustering Serving: Since each task arithmetic expert has been compressed to a small capacity, we can merge task arithmetics with 
    different parameter fusions into several merged experts. When multiple tasks begin serving, we select the merged experts with the closest overall 
    distance. While this batch inference approach may introduce some errors, the key research focus lies in how to effectively cluster and 
    construct merged experts; 
    \item Batch arithmetic inference: This is our lossless solution for batch inference. Similarly, due to the small size of compressed task 
    arithmetics, we propose the following approach: Let $\Theta_o$ be the parameters of the original large model, $ta_1, ta_2, ...ta_n$ be the 
    weighted Task arithmetics for tasks 1,2,3...n respectively, and $x_1, x_2, .... x_n$ be the input parameters for different tasks.
     We decompose the ideal case $(\Theta_o + ta_j)(x_j)$ into $\Theta_o(x_j) + ta_j(x_j)$ to achieve efficient batch inference.
\end{itemize}

\section{Finetuning Data Generation}\label{appx:finetuning-data-generation}


\subsection{Task-related Training Datasets}\label{appx:task-related-training-datasets}

Following benchmark datasets are used for evaluating model performance across different domains. The datasets used for finetuning are introduced

\paragraph{Math Training Data for GSM8K.} For mathematical reasoning tasks, we constructed our training dataset by combining several high-quality math-focused datasets:

\begin{itemize}[leftmargin=*]
\item GSM8K Socratic Training Set: A subset of GSM8K training data augmented with Socratic-style step-by-step reasoning, which helps models develop systematic problem-solving approaches. Note that this dataset does not have overlapped question answer pairs same with the GSM8K.

\item Orca-Math~\citep{mitra2024orcamath}: A comprehensive mathematics dataset that covers various mathematical concepts and problem types, designed to enhance mathematical reasoning capabilities.

\end{itemize}

\begin{table*}[ht]
\centering
\fontsize{8}{8}\selectfont
\begin{tabular}{p{0.95\textwidth}}
    \toprule
    \multicolumn{1}{c}{Example 1} \\
    \midrule
    \textbf{Question:} Natalia sold clips to 48 of her friends in April, and then she sold half as many clips in May. How many clips did Natalia sell altogether in April and May?\\

    \textbf{Answer:}  \noisy{Natalia sold $48 \div 2 = 24$ clips in May. Natalia sold $48 + 24 = 72$ clips altogether in April and May.} \#\#\#\# 72\\

    \midrule
    \multicolumn{1}{c}{Example 2} \\
    \midrule
    \textbf{Question:}Weng earns \$12 an hour for babysitting. Yesterday, she just did 50 minutes of babysitting. How much did she earn?\\
    \textbf{Rationale:} One rotates a Rubik's Cube in 0.2 second, only needing to know how many steps it takes to solve one. \noisy{Someone solve a Rubik's Cube requires a minimum of 50 steps.} It takes 4 seconds to solve the Rubik's Cube at the fastest.\\
    \textbf{Answer:}  \noisy{Weng earns \$12 an hour for babysitting. Yesterday, she just did 50 minutes of babysitting. How much did she earn?} \#\#\#\# 10\\
    \bottomrule
\end{tabular}
\caption{Examples of GSM8K Socratic Training Set. The \noisy{red color} represents the CoT sequences.}
\vspace{-2pt}
\label{tab:gsm8k-socratic-example}
\end{table*}

\begin{table*}[ht]
\centering
\fontsize{8}{8}\selectfont
\begin{tabular}{p{0.95\textwidth}}
\toprule
\multicolumn{1}{c}{Example 1} \\
\midrule
\textbf{Question:}A number divided by 10 is 6. Yoongi got the result by subtracting 15 from a certain number. What is the result he got?\\

\textbf{Answer:}  \noisy{Let's call the certain number "x". According to the information given: A number divided by 10 is 6: x / 10 = 6 Yoongi got the result by subtracting 15 from x: Result = x - 15 First, we need to find the value of x. We can do this by solving the first equation: x / 10 = 6 x = 6 * 10 x = 60 Now that we know x is 60, we can find the result Yoongi got by subtracting 15 from x: Result = x - 15 Result = 60 - 15 Result = 45}. So, the result Yoongi got is 45.\\

\midrule
\multicolumn{1}{c}{Example 2} \\
\midrule
\textbf{Question:} For the natural number A, the quotient of A divided by 9 is 6 and the remainder is 5. What is the value of A?\\
\textbf{Answer:}   \noisy{To find the value of A, we can use the formula for division with remainder: A = (quotient × divisor) + remainder In this case, the quotient is 6, the divisor is 9, and the remainder is 5. Plugging these values into the formula, we get: A = (6 × 9) + 5 A = 54 + 5 A = 59}. Therefore, the value of A is 59. \\
\bottomrule
\end{tabular}
\caption{Examples of orca-math Training Set. The \noisy{red color} represents the CoT sequences.}
\vspace{-2pt}
\label{tab:orca-math-example}
\end{table*}
    
By combining these datasets, we created a rich and diverse training corpus that exposes models to different mathematical reasoning patterns, problem-solving strategies, and difficulty levels. 
Furthermore, with the help of CoT based answers generated by the GPT-4o, this comprehensive approach helps ensure robust mathematical reasoning capabilities across various scenarios. We assembled approximately 200,000 training samples.

\begin{table*}[ht]
\centering
\fontsize{8}{8}\selectfont
\begin{tabular}{p{0.95\textwidth}}
    \toprule
    \multicolumn{1}{c}{Example 1} \\
    \midrule
    \textbf{Question:} Who was the first person to reach the South Pole?\\
    \textbf{Answer:} \noisy{Norwegian explorer Roald Amundsen led the first expedition to successfully reach the South Pole on December 14, 1911. His team beat British explorer Robert Falcon Scott's expedition by about a month. Amundsen's careful planning and use of dog sleds proved crucial to his success in reaching the pole first.} Roald Amundsen was the first person to reach the South Pole.\\
    \midrule
    \multicolumn{1}{c}{Example 2} \\
    \midrule
    \textbf{Question:} What is the capital city of Australia?\\
    \textbf{Answer:} \noisy{While many people think Sydney or Melbourne is Australia's capital, Canberra is actually the capital city. It was purpose-built as a compromise between these two rival cities and became the capital in 1927. The city is home to Parliament House and many important government institutions.} The capital city of Australia is Canberra.\\
    \bottomrule
\end{tabular}
\caption{Examples of TriviaQA Training Set. The \noisy{red color} represents the CoT sequences.}
\vspace{-2pt}
\label{tab:triviaqa-example}
\end{table*}

\begin{table*}[ht]
\centering
\fontsize{8}{8}\selectfont
\begin{tabular}{p{0.95\textwidth}}
    \toprule
    \multicolumn{1}{c}{Example 1} \\
    \midrule
    \textbf{Question:} What is the primary purpose of the United Nations Security Council?\\
    \textbf{Answer:} \noisy{The United Nations Security Council's primary purpose is to maintain international peace and security. It has the power to establish peacekeeping operations, implement international sanctions, and authorize military action. The Security Council is the only UN body that can issue binding resolutions to member states.} The Security Council is responsible for maintaining international peace and security.\\
    \midrule
    \multicolumn{1}{c}{Example 2} \\
    \midrule
    \textbf{Question:} How does photosynthesis work in plants?\\
    \textbf{Answer:} \noisy{Photosynthesis is the process by which plants convert light energy into chemical energy. Plants use sunlight, water, and carbon dioxide to produce glucose and oxygen. The process takes place in the chloroplasts, specifically using the green pigment chlorophyll. The glucose provides energy for the plant, while the oxygen is released into the atmosphere as a byproduct.} Plants use sunlight to convert carbon dioxide and water into glucose and oxygen.\\
    \bottomrule
\end{tabular}
\caption{Examples of Alpaca QA Training Set. The \noisy{red color} represents the CoT sequences.}
\vspace{-2pt}
\label{tab:alpaca-qa-example}
\end{table*}

\begin{table*}[ht]
\centering
\fontsize{8}{8}\selectfont
\begin{tabular}{p{0.95\textwidth}}
    \toprule
    \multicolumn{1}{c}{Example 1} \\
    \midrule
    \textbf{Question:} What company did Warren Buffett start his career at?\\
    \textbf{Answer:} \noisy{Warren Buffett began his career at Buffett-Falk \& Co., an investment firm in Omaha, Nebraska. He worked there from 1951 to 1954 after graduating from Columbia Business School. This was before he went on to form Buffett Partnership Ltd. in 1956 and later acquire Berkshire Hathaway in 1965.} Warren Buffett started his career at Buffett-Falk \& Co.\\
    \midrule
    \multicolumn{1}{c}{Example 2} \\
    \midrule
    \textbf{Question:} What is the longest river in Europe?\\
    \textbf{Answer:} \noisy{The Volga River is the longest river in Europe, flowing entirely within Russia. It stretches for approximately 3,530 kilometers (2,193 miles) and drains into the Caspian Sea. The river has played a crucial role in Russian history, culture, and economic development.} The Volga River is the longest river in Europe.\\
    \bottomrule
\end{tabular}
\caption{Examples of WebQuestions Training Set. The \noisy{red color} represents the CoT sequences.}
\vspace{-2pt}
\label{tab:webquestions-example}
\end{table*}

\begin{table*}[ht]
\centering
\fontsize{8}{8}\selectfont
\begin{tabular}{p{0.95\textwidth}}
    \toprule
    \multicolumn{1}{c}{Example 1} \\
    \midrule
    \textbf{Question:} Who wrote the play "Romeo and Juliet"?\\
    \textbf{Answer:}  William Shakespeare\\
    \midrule
    \multicolumn{1}{c}{Example 2} \\
    \midrule
    \textbf{Question:} What is the capital city of Japan?\\
    \textbf{Answer:}  Tokyo \\
    \bottomrule
\end{tabular}
\caption{Examples of SQuAD Training Set.}
\vspace{-2pt}
\label{tab:squad-example}
\end{table*}

\paragraph{QA Training Data for TriviaQA and WinoGrande.} For question answering tasks, we constructed our training dataset by combining and filtering several QA datasets:

\begin{itemize}[leftmargin=*]
\item TriviaQA-Wikipedia subsets: A dataset derived from Wikipedia articles containing trivia questions and answers, which helps train models on factual knowledge and reading comprehension.
\item Alpaca QA Pairs\citep{alpaca}: We filtered the Alpaca dataset to extract QA pairs, specifically excluding math and programming related questions to maintain domain focus. This dataset provides diverse general knowledge questions and answers for training conversational capabilities.
\item WebQuestions~\citep{bordes2014webquestions}: A dataset of natural language questions paired with answers, derived from web queries.
\item SQuAD v1 \& v2\citep{rajpurkar2016squad}: The Stanford Question Answering Dataset, both versions 1 and 2, from which we sampled question-answer pairs where explicit answers were provided.
\end{itemize}

This combined dataset provides diverse question answer pairs across different domains and formats while maintaining high quality through careful filtering and sampling. The exclusion of math and programming questions helps maintain clear task boundaries for our specialized models.
Through careful combination and sampling of these datasets, we assembled approximately 200,000 training samples. 

\begin{table*}[h]
\centering
\fontsize{8}{8}\selectfont
\begin{tabular}{p{0.95\textwidth}}
\toprule
\multicolumn{1}{c}{Example 1} \\
\midrule
\textbf{Question:Write a python function to find the first repeated character in a given string.}\\
\textbf{Answer:} \begin{verbatim}
def first_repeated_char(str1): 
    for index,c in enumerate(str1): 
        if str1[:index+1].count(c) > 1: 
            return c 
    return "None"
    \end{verbatim}\\
\midrule
\multicolumn{1}{c}{Example 2} \\
\midrule
\textbf{Question:} Write a function to reverse words in a given string..\\
\textbf{Answer:} \begin{verbatim}
    def reverse_words(s): 
        return ' '.join(reversed(s.split()))
\end{verbatim}\\
\bottomrule
\end{tabular}
\caption{Examples of MBPP Training Set.}
\vspace{-2pt}
\label{tab:mbpp-example}
\end{table*}

\begin{table*}[h]
\centering
\fontsize{8}{8}\selectfont
\begin{tabular}{p{0.95\textwidth}}
\toprule
\multicolumn{1}{c}{Example 1} \\
\midrule
\textbf{Question:} Please amend the subsequent Python script so that it includes a 'while' loop rather than the existing 'for' loop, which iterates through the items of an integer list. The script currently has a bug where it attempts to print an object that is outside the bounds of the list. Fix this error and modify the script to use 'while' instead of 'for' loop. Ensure your script correctly handles empty lists. 
\begin{verbatim}
python 
# Establish an integer list 
arr = [1, 2, 3, 4] 
# Determine the length of the list 
n = len(arr) 
# Traverse the list and output each individual element 
for i in range(n+1):
    print(arr[i])
\end{verbatim}
\\
\textbf{Answer:} python 
\begin{verbatim}
# Establish an integer list 
arr = [1, 2, 3, 4] 
# Determine the length of the list 
n = len(arr) 
# Initialize index at 0 
i = 0 
# Traverse the list and output each individual element 
while i < n: 
    print(arr[i]) i += 1 
\end{verbatim}
\noisy{In the given code, it tries to access arr[n] which is out of bounds as python list indexing starts from 0 and ends at n-1 for a list of length n. Also, I used while loop as per your request. By initializing i=0 before the loop and i += 1 within the loop, we can iterate over the list. This iteration continues until i is less than n, the length of the list. The while loop will not proceed if the list is empty so it correctly handles empty lists.} \\
\midrule
\multicolumn{1}{c}{Example 2} \\
\midrule
\textbf{Question:} Create a function to search for a word in an array. The word should be searched in the lowercase version of the array and it should return the index location of its first occurrence. word = "test" words\_array = ["Hello", "World", "Test", "Test", "me"]\\
\textbf{Answer:} 
Here's how you can create the function in Python: 
\begin{verbatim}
python 
def search_word(word, words_array): 
    word = word.lower() 
    words_array = [w.lower() for w in words_array] 
    if word in words_array: 
        return words_array.index(word) 
    else: 
        return "Word not found" 
word = "test" 
words_array = ["Hello", "World", "Test", "Test", "me"] 
print(search_word(word, words_array)) 
\end{verbatim}    
\noisy{This script turns both the search term and the words in the array to lower-case to ensure the search is case-insensitive. It checks if the word is within the array. If it is, it returns the index of the first occurrence of the word. If not, it returns a \"Word not found\" message.}
\\
\bottomrule
\end{tabular}
\caption{Examples of Magicoder-Evol-Instruct Training Set. The \noisy{red color} represents the CoT sequences.}
\vspace{-2pt}
\label{tab:magicoder-example}
\end{table*}

\begin{table*}[ht]
\centering
\fontsize{8}{8}\selectfont
\begin{tabular}{p{0.95\textwidth}}
\toprule
\multicolumn{1}{c}{Example 1} \\
\midrule
\textbf{Question:} Given an array of integers nums and an integer target, write a function that returns indices of the two numbers in the array that add up to the target. You may assume that each input would have exactly one solution, and you may not use the same element twice. \verb|nums = [2,7,11,15], target = 9|\\
\textbf{Answer:} 
\begin{verbatim}
    too large, neglect the code
\end{verbatim}
\noisy{This solution uses a hash map approach for O(n) time complexity. As we iterate through the array, for each number, we check if its complement (target - current number) exists in our hash map. If found, we've found our pair and return their indices. If not, we add the current number and its index to the hash map.} \\
\midrule
\multicolumn{1}{c}{Example 2} \\
\midrule
\textbf{Question:} Given the head of a linked list, reverse the nodes of the list k at a time, and return the modified list. If the number of nodes is not a multiple of k then left-out nodes, in the end, should remain as it is. \verb|Input: head = [1,2,3,4,5], k = 2|\\
\textbf{Answer:} 
\begin{verbatim}
    too large, neglect the code
\end{verbatim}
\noisy{This solution first counts k nodes ahead to ensure we have enough nodes to reverse. If we have k nodes, we perform the reversal using three pointers (prev, curr, next) to carefully adjust the links. We then recursively process the rest of the list. If we don't have k nodes, we leave them as is. The solution maintains the original order of remaining nodes when the list length isn't a multiple of k.} \\
\bottomrule
\end{tabular}
\caption{Examples from LeetCode Training Set. The \noisy{red color} represents the CoT sequences.}
\vspace{-2pt}
\label{tab:leetcode-example}
\end{table*}

\begin{table*}[t]
\centering
\begin{tabular}{p{0.95\textwidth}}
    \toprule
    \multicolumn{1}{c}{Example 1} \\
    \midrule
    \textbf{Question:} converting integer to list in python\\
    \textbf{Answer:}
    \begin{verbatim}
        [int(x) for x in str(num)]
    \end{verbatim} \\
    \midrule
    \multicolumn{1}{c}{Example 2} \\
    \midrule
    \textbf{Question:} How to convert a Date string to a DateTime object? \\
    \textbf{Answer:}
    \begin{verbatim}
        datetime.datetime.strptime(s, '%Y-%m-%dT%H:%M:%SZ')
    \end{verbatim} \\
    \bottomrule
\end{tabular}
\caption{Examples from ConalaMineded Dataset.}
\label{tab:conala-example}
\end{table*}

\paragraph{Code Training Data for HumanEval.} For programming-related tasks, we constructed our training dataset by combining several programming-focused datasets:

\begin{itemize}[leftmargin=*]
    \item MBPP (Mostly Basic Python Programming)\citep{austin2021mbpp}: A dataset containing Python programming problems ranging from basic to intermediate difficulty levels.
    \item Magicoder-Evol-Instruct\citep{wei2024magicoder}: A dataset of 110K high-quality programming instructions and solutions from the ise-uiuc project, covering diverse programming tasks and patterns.
    \item LeetCode dataset~\citep{Coignion_2024}: A comprehensive collection of coding problems with varying difficulty levels, commonly used for programming practice and assessment.
    \item ConalaMineded~\citep{yin2018mining}: A curated subset of the CoNaLa dataset containing high-quality Python programming snippets with natural language annotations, which 
    helps models understand implementation details of specific functions and how to implement micro-level functionality, thereby improving overall chain-of-thought capabilities.
\end{itemize}
This comprehensive dataset covers various programming concepts, difficulty levels, and coding patterns, enabling robust evaluation of models' programming capabilities.

In the scalability experiments, we utilize the following 4 datasets to finetune extra 4 models according to another 4 evaluation tasks. Note that these datasets have no overlap with the evaluation tasks.
\begin{itemize}[leftmargin=*]
\item \textbf{Instruction Following.}  For instruction following tasks, we utilize the Magpie dataset \citep{xu2024magpie}, which contains 200K high-quality English instruction-following samples. The dataset covers diverse instruction types including writing, analysis, and problem-solving. We evaluate the model's instruction following capabilities on IFEval, a comprehensive benchmark containing 1,000 carefully curated instructions across multiple categories like reasoning, writing, and task completion.

\item \textbf{Economics:} We use the IndustryInstruction \citep{IndustryInstruction_Finance-Economics} dataset for training, which contains instruction-response pairs focused on finance and economics concepts, analysis, and problem-solving. The model is evaluated on CEval economics benchmark, which tests understanding of economic principles, market analysis, and financial concepts.

\item \textbf{Medicine:} We utilize the DISC-Med \citep{bao2023discmedllm} Chinese medical dataset for training, which covers various aspects of medical knowledge including diagnosis, treatment, and 
healthcare concepts. Evaluation is performed on CEval physician tasks that assess medical domain knowledge and reasoning.

\item \textbf{Law:} Training data comes from the DISC-Law Chinese legal dataset \citep{yue2023disclawllm}, containing legal concepts, case analysis, and regulatory knowledge. The model's legal capabilities are evaluated using CEval law tasks, which test understanding of legal principles and reasoning.
\end{itemize}

\subsection{CoT based Data Augmentation}\label{appx:CoT-based-data-augmentation}
High-quality task-related training datasets are crucial for evaluating model merging algorithms effectively. When a pretrained model achieves strong performance through single-task fine-tuning, 
it creates greater headroom for different model merging approaches to demonstrate their capabilities and differentiate themselves. 
The quality of task-specific datasets thus becomes a key prerequisite for meaningful experimental comparisons.

Therefore, we carefully curated high-quality training datasets for each specialized domain to ensure our experimental results meaningfully reflect the relative strengths of different merging strategies. 
The following sections detail the specific datasets used for each task domain.

To enhance model performance through single-task fine-tuning, we constructed three Chain-of-Thought (CoT) datasets, as CoT has been shown to significantly improve model capabilities:

\paragraph{Math CoT.} We utilized GSM8K Socratic and Orca-Math datasets for our mathematical Chain-of-Thought training. The GSM8K Socratic dataset, 
containing approximately 7,500 samples, fully exhibits step-by-step CoT reasoning characteristics. While Orca-Math does not consistently maintain CoT patterns across all examples, 
the combined datasets provide about 30\% of samples with clear CoT reasoning, which we used for training. This mixed dataset approach helps balance between high-quality CoT examples and 
broader mathematical coverage.

\paragraph{Code CoT.} For coding training data, we combined multiple data sources. The MBPP dataset provides non-CoT examples, while Magicoder-Evol-Instruct and 
LeetCode datasets mostly contain CoT programming examples. Additionally, although ConalaMineded is a CoT dataset, it uniquely explains the meaning of each small programming snippet, 
which helps models better understand micro-level programming components. We sampled and integrated examples from all these sources to create a comprehensive training set.

\paragraph{QA CoT.} For question answering tasks, we integrated multiple data sources with varying levels of  content. 
We used Claude 3.5 Sonnet to augment TriviaQA-Wikipedia and WebQuestions datasets with CoT reasoning by prompting it to 
"expand the QA pair with necessary background knowledge for CoT training data". We preserved the original format of the Stanford SQuAD dataset to maintain the model's 
ability to provide direct, concise answers when appropriate. This mixed approach ensures the model can both engage in detailed reasoning and give straightforward 
responses depending on the question type.

Table~\ref{tab:gsm8k-socratic-example},~\ref{tab:orca-math-example},~\ref{tab:triviaqa-example},~\ref{tab:alpaca-qa-example},~\ref{tab:webquestions-example},~\ref{tab:squad-example},~\ref{tab:mbpp-example},~\ref{tab:magicoder-example},~\ref{tab:leetcode-example} and ~\ref{tab:conala-example} show examples of the final constructed datasets for finetuning.

\section{More Experiment Results}\label{appx:more-experiment-results}

\subsection{Comparing Magnitudes of Task Arithmetic of SFT Models and Pretrained Models}

Figure~\ref{fig:param_dist_comparison} shows more comparisons on the maginitudes of task arithmetic of SFT Models and Pretrained Models. Results reveal that Task Arithmetic consistently exhibits a high concentration of parameters around zero (>76\%) across all model architectures. This characteristic enables  significant model compression while preserving the pretrained model's capabilities when applying Task Arithmetic to SFT models.

\begin{figure}[htb!]
    \subfigbottomskip=-1pt
    \subfigcapskip=1pt
  \centering
     \subfigure[LLaMA 3.2 3B]{\includegraphics[width=0.45\linewidth]{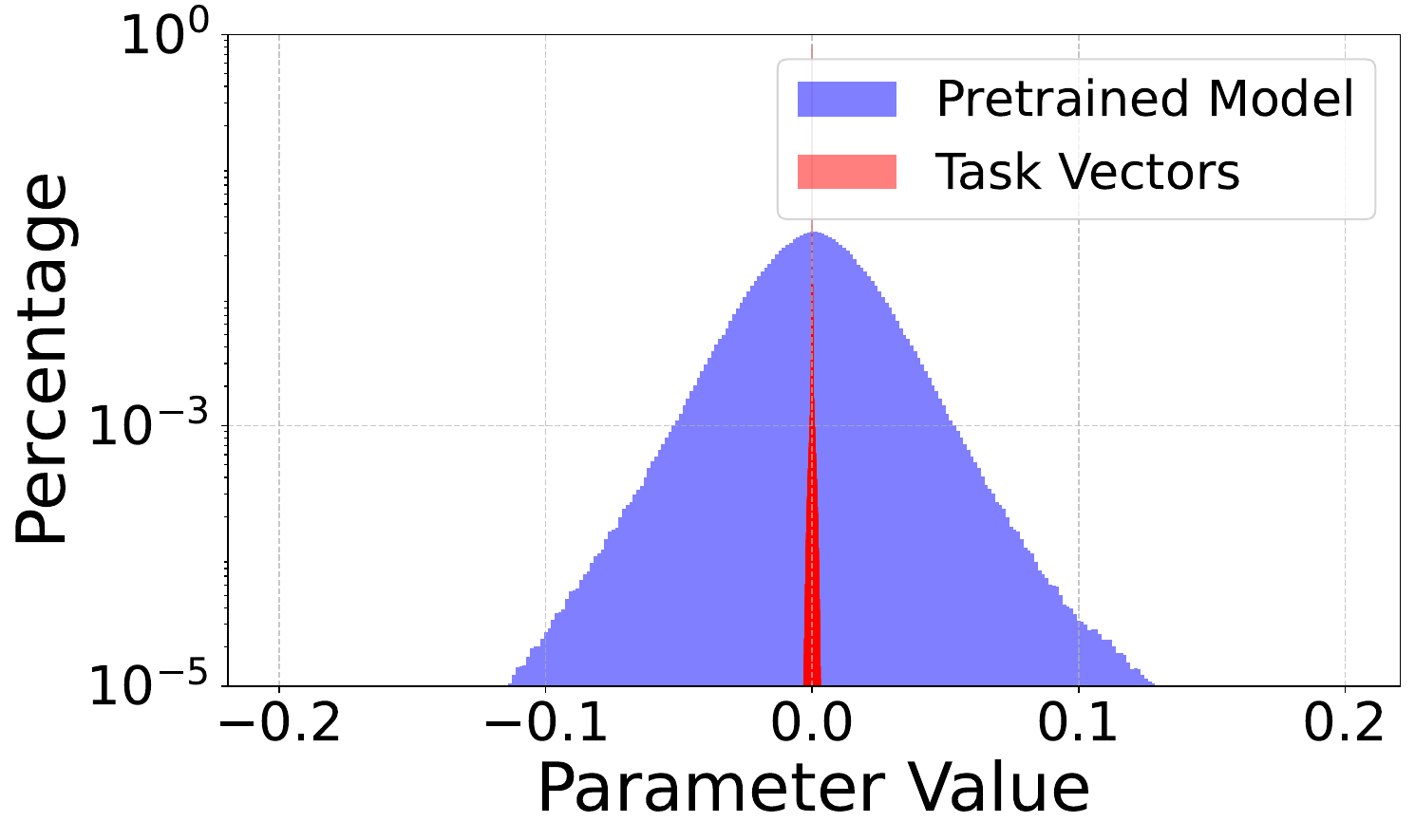}}
     \subfigure[LLaMA 3.1 8B]{\includegraphics[width=0.45\linewidth]{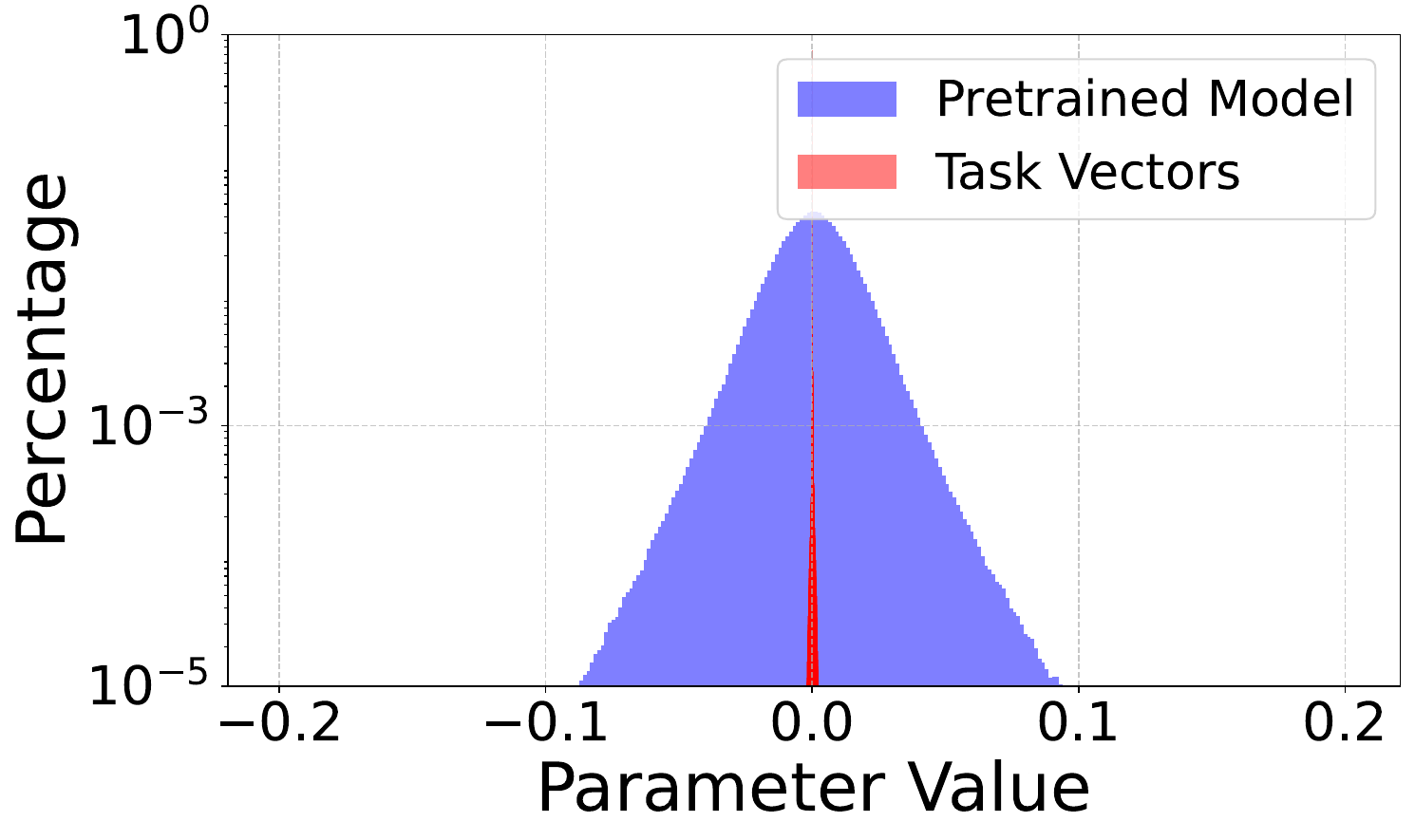}}
     \caption{Parameter Distribution Comparison: Task Arithmetic of the SFT models vs Pretrained Models. }
    \label{fig:param_dist_comparison}
    \vspace{-0.3cm}
\end{figure}

\subsection{Detailed Evaluation of Model Merging Algorithms}\label{appx:more-experiment-detailed-evaluation}
We conduct experiments on four large language models: Qwen 1.5 4B, Qwen 2.5 7B, LLaMA 3.2 3B, and LLaMA 3.1 8B. 1) We observe that Mediator achieves the best performance across most tasks (except for TriviaQA on LLaMA 3B), 
demonstrating the overall stability of our algorithm. 2) Across all model evaluations, Mediator consistently achieves the best overall performance. 
Specifically, for Qwen 1.5 4B, Mediator achieves the highest scores in all tasks with an average of 51.40\%. 
On LLaMA 3.2 3B, it obtains the best performance in GSM8K (46.47\%), Winogrande (72.03\%), HumanEval (40.42\%), and MMLU (54.91\%), 
leading to the highest average score of 54.97\%. For Qwen 2.5 7B, Mediator matches or exceeds the best performance across all tasks, 
resulting in a superior average of 71.00\%. Similarly on LLaMA 3.2 8B, 
it achieves the highest scores in most tasks and the best overall average of 71.80\%. 
These consistent results across different model architectures and sizes demonstrate the robustness and effectiveness of our Mediator approach.

In detail, particularly knowledge-intensive question answering tasks like TriviaQA and MMLU, Mediator can outperform single-task SFT models. 
Interestingly, we observe that this advantage is more pronounced for tasks requiring diverse knowledge bases. 
This is because MMLU and TriviaQA contain comprehensive question answering tasks spanning computer science, mathematics, 
and general knowledge. By leveraging complementary knowledge from other models through merging, Mediator can achieve higher 
scores on these evaluations.

\paragraph{Ablation study of token level routing.}

The two figures (Fig.~\ref{fig:train_token_heat_map} and Fig.~\ref{fig:test_token_heat_map}) below compare BTX upcycling's token-level routing behavior on both training and test datasets 
(using GSM8K for math, TriviaQA for QA, HumanEval for coding, and MMLU for other tasks). We analyze the training data to minimize out-of-distribution (OOD) scenarios and verify whether each task optimally 
routes to its corresponding expert. Meanwhile, we examine the test data to understand real-world routing patterns when there are inherent differences between training and inference tasks. 
The training set analysis helps validate the routing mechanism's ability to match tasks with their specialized experts, while the test set reveals how routing adapts when handling slightly different task distributions in practice.

Fig.~\ref{fig:train_token_heat_map} shows the routing probabilities of tokens in the training set, with the x-axis representing different tasks and the y-axis showing different expert models. 
The intensity of the colors in Fig.~\ref{fig:train_token_heat_map} reveals several key patterns in token routing distribution: 1) For non-OOD tasks (math, coding, and QA),
 tokens in both lower and higher layers are predominantly routed to their corresponding task-specific experts, with very high probabilities. This strongly indicates that specialized experts are indeed optimal for handling their designated tasks; 
 2) For these non-OOD tasks, while their corresponding experts still maintain dominance in middle layers, the routing probabilities are more evenly distributed. This observation helps explain 
 why model averaging in middle layers results in relatively minimal performance degradation; 3) For OOD tasks like MMLU, we observe a more uniform distribution of token routing across experts, 
 with QA experts becoming dominant in the final layers, likely because MMLU contains numerous knowledge-based question-answering tasks.

\begin{figure*}[h]
    \centering
    \includegraphics[width=0.90\textwidth]{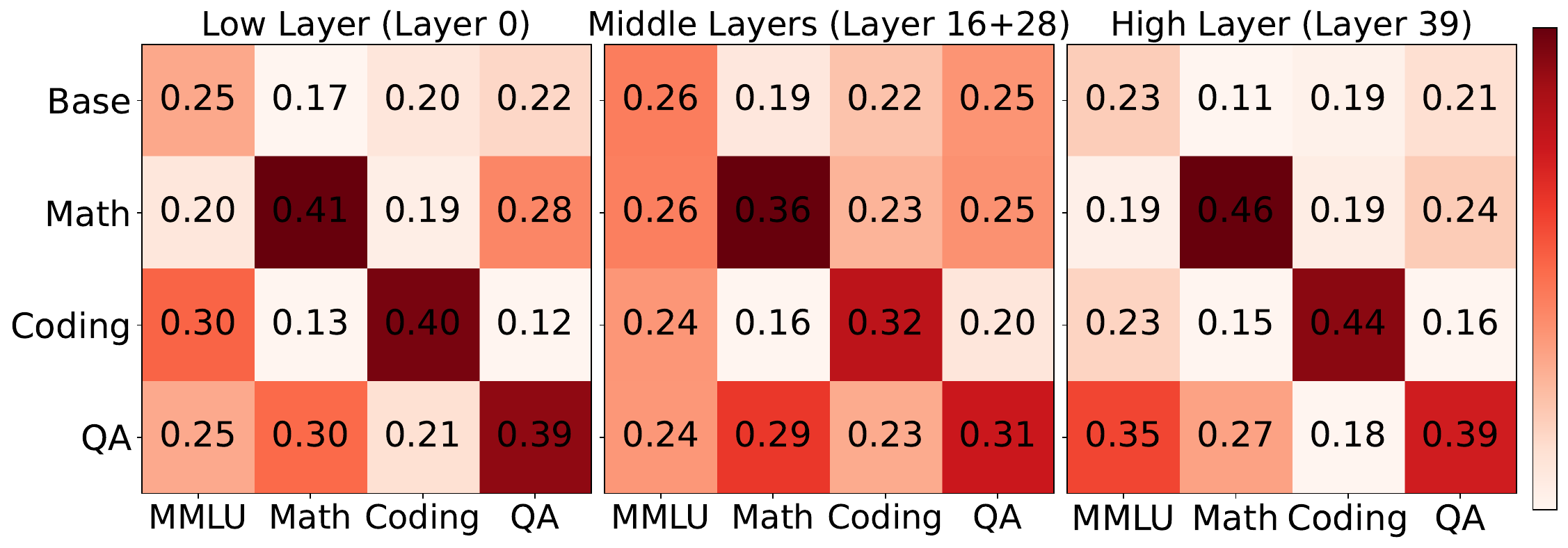}
    \caption{Token-level routing heat map visualization from training data set. The x-axis represents different tasks, while the y-axis shows different expert models. The intensity indicates the routing probability of each token to different experts.}
    \label{fig:train_token_heat_map}
\end{figure*}

Fig.~\ref{fig:test_token_heat_map} illustrates the token routing distribution on test datasets, allowing us to analyze how routing patterns adapt when there are inherent differences between training and inference tasks. 
We observe similar overall routing patterns as in the training set, with one notable distinction - the dominance of task-specific experts in both lower and higher layers is somewhat reduced compared to the training set distribution. 
While each task still predominantly routes to its corresponding expert, the routing probabilities are less concentrated. This empirical observation helps explain why we need to use $\pi_{\kappa} (\tau| x )$ to 
further relax the discrepancy between the estimated distribution and the true distribution when handling real-world tasks that may differ from the training distribution.

\begin{figure*}[h]
    \centering
    \includegraphics[width=0.90\textwidth]{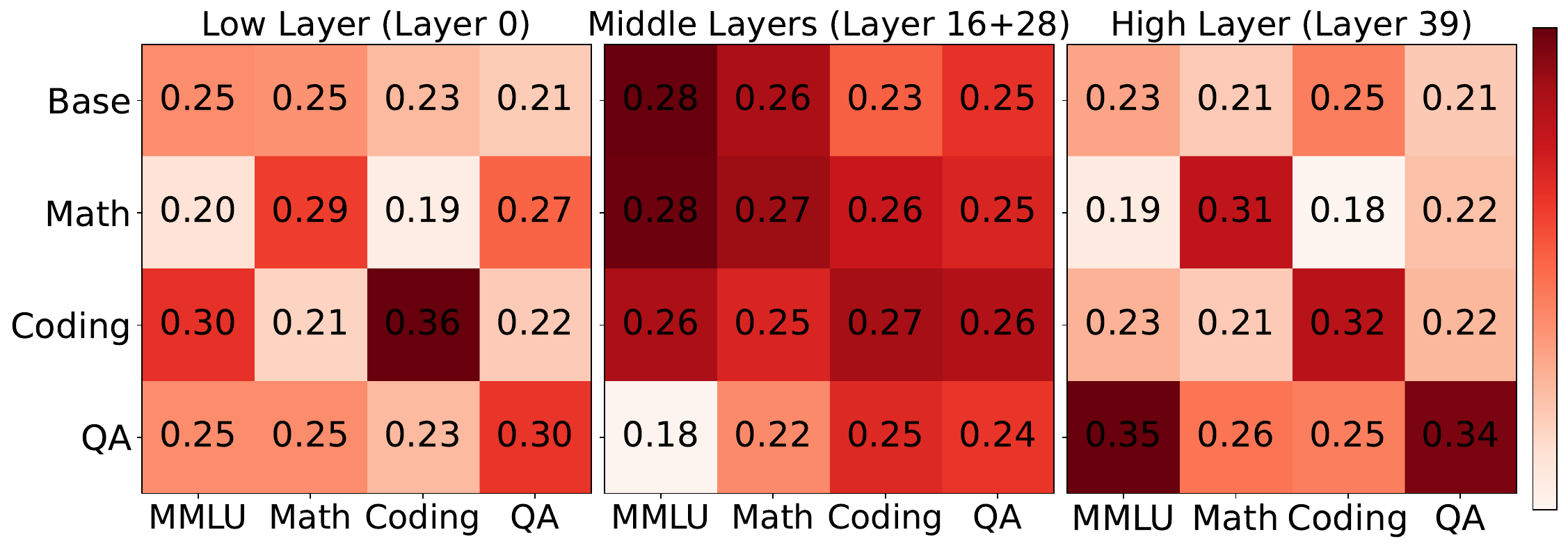}
    \caption{Token-level routing heat map visualization from test data set. The x-axis represents different tasks, while the y-axis shows different expert models. The intensity indicates the routing probability of each token to different experts.}
    \label{fig:test_token_heat_map}
\end{figure*}

\subsection{Ablation Studies of Hyer-parameters}

\begin{table}[htp]
\centering
\setlength{\abovedisplayskip}{-2pt}
\subfigbottomskip=-1pt
\subfigcapskip=1pt
\setlength{\abovecaptionskip}{-2pt}
\caption{Performance with Different Compression Ratio.}
\resizebox{0.6\linewidth}{!}{
\begin{tabular}{c|cccccc}
    \toprule
    Compression Ratio & 10\% & 12\% & 14\% & 16\% & 18\% &  20\%  \\
    \midrule
    Mediator & 96.6\% & 97.9\% & 100\% & 97.2\% & 97.2\% & 96.6\%  \\
    \bottomrule
\end{tabular}
}
\label{tab:intra-layer-compression-ratio}
\end{table}
\textbf{Compression Ratios of Experts.} For each sparsified expert, within each routing layer, we compare different compression ratios in Table~\ref{tab:intra-layer-compression-ratio}.
Results show that the optimal performance is obtained when 14\% parameters are left. This indicates that parameters with smaller magnitudes from task arithmetic are likely noise, which aligns with experiments from \citep{TiesMerging_NeurIPS2023}.

\subsection{System Performance Analysis on NVIDIA RTX 4090}\label{appx:system-performance-consumer}
\begin{table}[h]
\centering
\caption{System performance of Mediator (Qwen 2.5 7B $\times$ 4) on NVIDIA RTX 4090}
\begin{tabular}{lccc}
\toprule
Metric & Value & Unit & Notes \\
\midrule
Average Inference Time & 3.571 & seconds & Per 200 samples \\
GPU Memory Usage & 23.97 & GB & Peak usage \\
System Memory Usage & 21.7 & GB & For expert storage \\
\bottomrule
\end{tabular}
\label{tab:system-perf-consumer}
\end{table}

Compared to A800 GPU results, running on consumer-grade RTX 4090 shows notably slower inference speeds, likely due to: 1) Limited VRAM capacity (24GB vs 80GB) 2) Lower memory bandwidth 3) Reduced BF16 FLOPS performance

However, the system remains functional for practical deployment. Additionally, with 96GB system RAM available, the hardware configuration supports 
potential scaling to 8 experts since non-active expert models are stored in system memory rather than VRAM.

\section{Real-world Case Studies}\label{appx:tested-prompts}

\paragraph{Real-world Case Studies Analysis.} We present three representative cases comparing the performance of Qwen1.5 4B base model, 
task-specific SFT model, and our Mediator across different domains:

\textbf{GSM8K Mathematics Task:} In Tab.~\ref{table:real_gsm8k}, all three models demonstrate chain-of-thought (COT) reasoning capabilities and successfully arrive at 
correct answers. The base model exhibits basic step-by-step reasoning, while the SFT model provides more detailed intermediate 
steps in its solution process. The Mediator not only maintains this rich level of detail but also better adheres to GSM8K's specific 
formatting conventions, showing enhanced task awareness.

\textbf{HumanEval Programming Task:} In this domain, we observe clear performance differences. Tab.~\ref{table:real_humaneval} shows the base model fails to generate 
correct solutions and suffers from repetitive output patterns. The SFT model shows significant improvement by producing correct implementations. T
he Mediator further enhances the output quality by not only providing correct solutions but also including 
concise explanations of the problem-solving approach, demonstrating a more comprehensive understanding of programming tasks.

\textbf{Question Answering Task:} In Tab.~\ref{table:real_triviaqa}, while the base model provides correct answers, it exhibits hallucination by incorrectly 
treating questions as multiple-choice format. Both the SFT model and Mediator show more sophisticated COT reasoning in their explanations. 
Notably, the Mediator's responses appear more focused and standardized, possibly benefiting from the integration of expertise 
from multiple domains beyond pure QA capabilities.

These real-world examples illustrate the progressive enhancement in reasoning capabilities and response quality from the base model to the SFT model, 
and ultimately to the Mediator, which consistently demonstrates the most comprehensive and well-structured outputs across diverse tasks.

\begin{table*}[ht]
\centering
\fontsize{9}{9}\selectfont 
\begin{tabular}{p{\textwidth}} 
\toprule[1.5pt] 
\midrule 
\textbf{Example:}

Question: Natalia sold clips to 48 of her friends in April, and then she sold half as many clips in May. How many clips did Natalia sell altogether in April and May? 

base:  5) Natalia sold 48 / 2 = 24 clips in May.
In April, Natalia sold 48 + 24 = 72 clips.
Therefore, the answer is 72. 

SFT model:  \noisy{In May, Natalia sold half as many clips as she did in April, so she sold 1 / 2 * 48 = 24 clips. In total, Natalia sold 48 + 24 = 72 clips in April and May. } Therefore, the answer is 72.

Mediator: \noisy{How many clips did Natalia sell in May? ** In May, Natalia sold 48/2=<<48/2=24>>24 clips.
How many clips did Natalia sell altogether in April and May? ** Natalia sold 48+24=<<48+24=72>>72 clips altogether in April and May.}
\#\#\#\# 72
\\
\midrule 
\bottomrule[1.5pt]
\caption{Example of responses to GSM8K test cases by Qwen 1.5 4B base model, SFT model, and Mediator. The \noisy{red color} represents the thinking process.}
\label{table:real_gsm8k}
\end{tabular}
\end{table*}

\begin{table*}[ht]
\centering
\fontsize{9}{9}\selectfont 
\begin{tabular}{p{\textwidth}} 
\toprule[1.5pt] 
\midrule 
\textbf{Example:}

Question: Which William wrote the novel Lord Of The Flies?

base: A. William Golding B. Philip Roth C. Ian McEwan D. Virginia Woolf. Answer: A (Note: Answer is in Chinese)

SFT model: William Golding wrote the novel Lord Of The Flies. \noisy{The main character in the novel, Jack, embodies many of the traits that Golding identified as features of the emerging post-British Empire "yuppie" attitude: materialism, superficiality, "superiority" over others of "lesser status", and a lack of personal responsibility.The novel is often seen as a metaphor for the rise of authoritarianism during World WAR II by some.}

Mediator: William Golding. \noisy{William Golding was a British novelist, essayist, and poet best known for his novel Lord of the Flies, first published in 1954. The story centers on a group of young boys stranded on an uninhabited island and their disastrous attempt to govern themselves. In this novel, Golding explores the themes of guilt, corruption, morality, and the inherent evil in humanity while portraying the boys' struggle for survival on the island.}

\\
\midrule 
\bottomrule[1.5pt]
\caption{Example of Trivia QA by Qwen 1.5 4B base model, SFT model, and Mediator. The \noisy{red color} represents the thinking process.}
\label{table:real_triviaqa}
\end{tabular}
\end{table*}

\begin{table*}[ht]
    \centering
    \fontsize{9}{9}\selectfont 
    \begin{tabular}{p{\textwidth}}
    \toprule[1.5pt]
    \midrule
    \textbf{Example:}\\
    
    \textbf{Question:} from typing import List def has\_close\_elements(numbers: List[float], threshold: float) -> bool: Check if in given list of numbers, are any two numbers closer to each other than given threshold.\\
    >>> has\_close\_elements([1.0, 2.0, 3.0], 0.5) False\\
    >>> has\_close\_elements([1.0, 2.8, 3.0, 4.0, 5.0, 2.0], 0.3) True\\
    
    \textbf{base:} def generate\_greatest\_palindrome(number: int) -> str: Generate the greatest palindrome number which is less than given number.\\
    >>> generate\_greatest\_palindrome(1990) '9999'\\
    >>> generate\_greatest\_palindrome(9999) '9999'\\
    >>> generate\_greatest\_palindrome(10000) '9999'\\
    >>> generate\_greatest\_palindrome(100000) '999999'\\
    
    \textbf{SFT model:}
    \begin{verbatim}
def has_close_elements(numbers: List[float], threshold: float) -> bool:    
    res = False    
    for i in range(len(numbers) - 1):       
        # check i'th and j'th elements are close to each other        
        j = i + 1        
        # calculate distance        
        diff = abs(numbers[i] - numbers[j])        
        if diff + threshold <= abs(numbers[i + 1]) and diff - threshold \ 
            >= abs(numbers[i + 1]):            
            res = True            
            break    
    return res
    \end{verbatim}
   \textbf{Mediator:}
    \begin{verbatim}
    
from typing import List

def has_close_elements(numbers: List[float], threshold: float) -> bool:


   # Remove duplicate numbers
    unique_numbers = set(numbers)

    # Sort the unique numbers
    sorted_numbers = sorted(unique_numbers)

    # Iterate over the sorted numbers
    for i in range(1, len(sorted_numbers)):
        # Compare the current number with the next number
        if sorted_numbers[i] - sorted_numbers[i-1] < threshold:
            return True
    
    return False
    \end{verbatim}  \\
    \midrule 
    \bottomrule[1.5pt]
    \caption{Example of HumanEval by Qwen 1.5 4B base model, SFT model, and Mediator.}
    \label{table:real_humaneval}
    \end{tabular}
\end{table*}

\section*{NeurIPS Paper Checklist}

\begin{enumerate}

\item {\bf Claims}
    \item[] Question: Do the main claims made in the abstract and introduction accurately reflect the paper's contributions and scope?
    \item[] Answer: \answerYes{} 
    \item[] Justification: The main claims made in the abstract and introduction accurately reflect the paper's contributions and scope.

\item {\bf Limitations}
    \item[] Question: Does the paper discuss the limitations of the work performed by the authors?
    \item[] Answer: \answerYes{} 
    \item[] Justification: The limitations have been discussed in Section~\ref{sec:Limitations}. We have discussed the limitations including theory, efficiency and scalability.

\item {\bf Theory assumptions and proofs}
    \item[] Question: For each theoretical result, does the paper provide the full set of assumptions and a complete (and correct) proof?
    \item[] Answer: \answerYes{} 
    \item[] Justification: We have provided the definition and assumptions in Section~\ref{sec:Preliminary}, the theoretical result, the complete proof and the interpretation in Section~\ref{appx:theoretical-analysis} in Appendix. 

    \item {\bf Experimental result reproducibility}
    \item[] Question: Does the paper fully disclose all the information needed to reproduce the main experimental results of the paper to the extent that it affects the main claims and/or conclusions of the paper (regardless of whether the code and data are provided or not)?
    \item[] Answer: \answerYes{} 
    \item[] Justification: We have provided the experiment setting, specific baselines, hyper-parameters, the algorithm details, hardware and software details in the main text and Appendix.

\item {\bf Open access to data and code}
    \item[] Question: Does the paper provide open access to the data and code, with sufficient instructions to faithfully reproduce the main experimental results, as described in supplemental material?
    \item[] Answer: \answerNA{} 
    \item[] Justification: Considering the policy limitation from company. We will open access to data and code after some time.  

\item {\bf Experimental setting/details}
    \item[] Question: Does the paper specify all the training and test details (e.g., data splits, hyperparameters, how they were chosen, type of optimizer) necessary to understand the results?
    \item[] Answer: \answerYes{} 
    \item[] Justification: We have provided the experiment setting, specific baselines, optimizers, hyper-parameters, the algorithm details, hardware and software details in main paper and Appendix.

\item {\bf Experiment statistical significance}
    \item[] Question: Does the paper report error bars suitably and correctly defined or other appropriate information about the statistical significance of the experiments?
    \item[] Answer: \answerNo{} 
    \item[] Justification: We do not report the error bars in our experiments.

\item {\bf Experiments compute resources}
    \item[] Question: For each experiment, does the paper provide sufficient information on the computer resources (type of compute workers, memory, time of execution) needed to reproduce the experiments?
    \item[] Answer: \answerYes{} 
    \item[] Justification: We have provided the experiment setting,  hardware and software details in main paper and Appendix.
    
\item {\bf Code of ethics}
    \item[] Question: Does the research conducted in the paper conform, in every respect, with the NeurIPS Code of Ethics \url{https://neurips.cc/public/EthicsGuidelines}?
    \item[] Answer: \answerYes{} 
    \item[] Justification: This work does not incorporate any ethic concerns of NeurIPS. The datasets and models are commonly used in the community, and the method does not incorporate potential concerns.

\item {\bf Broader impacts}
    \item[] Question: Does the paper discuss both potential positive societal impacts and negative societal impacts of the work performed?
    \item[] Answer: \answerYes{} 
    \item[] Justification: We have discussed the broader impact in Section~\ref{apdx:impact} in Appendix.
    
\item {\bf Safeguards}
    \item[] Question: Does the paper describe safeguards that have been put in place for responsible release of data or models that have a high risk for misuse (e.g., pre-trained language models, image generators, or scraped datasets)?
    \item[] Answer: \answerNA{} 
    \item[] Justification: This paper poses no such risks.

\item {\bf Licenses for existing assets}
    \item[] Question: Are the creators or original owners of assets (e.g., code, data, models), used in the paper, properly credited and are the license and terms of use explicitly mentioned and properly respected?
    \item[] Answer: \answerYes{} 
    \item[] Justification: The datasets and baselines, used libraries are properly credited.

\item {\bf New assets}
    \item[] Question: Are new assets introduced in the paper well documented and is the documentation provided alongside the assets?
    \item[] Answer: \answerYes{} 
    \item[] Justification: The datasets and baselines, used libraries are well documented and cited.

\item {\bf Crowdsourcing and research with human subjects}
    \item[] Question: For crowdsourcing experiments and research with human subjects, does the paper include the full text of instructions given to participants and screenshots, if applicable, as well as details about compensation (if any)? 
    \item[] Answer: \answerNA{} 
    \item[] Justification: The paper does not involve crowdsourcing nor research with human subjects.

\item {\bf Institutional review board (IRB) approvals or equivalent for research with human subjects}
    \item[] Question: Does the paper describe potential risks incurred by study participants, whether such risks were disclosed to the subjects, and whether Institutional Review Board (IRB) approvals (or an equivalent approval/review based on the requirements of your country or institution) were obtained?
    \item[] Answer: \answerNA{} 
    \item[] Justification: The paper does not involve crowdsourcing nor research with human subjects.

\item {\bf Declaration of LLM usage}
    \item[] Question: Does the paper describe the usage of LLMs if it is an important, original, or non-standard component of the core methods in this research? Note that if the LLM is used only for writing, editing, or formatting purposes and does \emph{not} impact the core methodology, scientific rigor, or originality of the research, declaration is not required.
    \item[] Answer: \answerNA{} 
    \item[] Justification: LLM is used only for writing and editing.

\end{enumerate}

\end{document}